\documentclass{article}

\usepackage{arxiv}

\usepackage{braket,amsfonts}
\usepackage{cancel} 
\usepackage{array}

\usepackage{subcaption}
\usepackage{wrapfig}

 \usepackage{pgfplots}

\usepackage{soul}

\newtheorem{theorem}{Theorem}

\newtheorem{lemma}{Lemma}

\newtheorem{remark}{Remark}

\newtheorem{proof}{Proof}

\usepackage{algorithm}
 \usepackage{algpseudocode}
 \usepackage{algorithmicx}
 \algdef{SE}[DOWHILE]{Do}{doWhile}{\algorithmicdo}[1]{\algorithmicwhile\ #1}%
\algrenewcommand\algorithmicensure{\textbf{Result:}}
\usepackage{multirow}
\usepackage{colortbl}

\usepackage{graphicx,epstopdf}

\usepackage{amsmath,amsfonts,bm}

\def\eqref#1{equation~\ref{#1}}

\def\1{\bm{1}}

\DeclareMathAlphabet{\mathsfit}{\encodingdefault}{\sfdefault}{m}{sl}
\SetMathAlphabet{\mathsfit}{bold}{\encodingdefault}{\sfdefault}{bx}{n}

\DeclareMathOperator*{\argmax}{arg\,max}

\usepackage{amsmath,amssymb}
\usepackage{graphicx}
\usepackage[colorlinks=true, allcolors=blue]{hyperref}
\usepackage{subcaption}
\usepackage{cleveref}
\usepackage{ifpdf}

\Crefname{ALC@unique}{Line}{Lines}

\usepackage{amsopn}

\usepackage{xspace}
\usepackage{bold-extra}
\usepackage[most]{tcolorbox}
\colorlet{texcscolor}{blue!50!black}
\colorlet{texemcolor}{red!70!black}
\colorlet{texpreamble}{red!70!black}
\colorlet{codebackground}{black!25!white!25}

\renewcommand{\appendix}{\par
  \setcounter{section}{0}
  \setcounter{subsection}{0}
  \gdef\thesection{SM\arabic{section}}
}

\lstdefinestyle{siamlatex}{%
  style=tcblatex,
  texcsstyle=*\color{texcscolor},
  texcsstyle=[2]\color{texemcolor},
  keywordstyle=[2]\color{texemcolor},
  moretexcs={cref,Cref,maketitle,mathcal,text,headers,email,url},
}

\tcbset{%
  colframe=black!75!white!75,
  coltitle=white,
  colback=codebackground, 
  colbacklower=white, 
  fonttitle=\bfseries,
  arc=0pt,outer arc=0pt,
  top=1pt,bottom=1pt,left=1mm,right=1mm,middle=1mm,boxsep=1mm,
  leftrule=0.3mm,rightrule=0.3mm,toprule=0.3mm,bottomrule=0.3mm,
  listing options={style=siamlatex}
}

\DeclareTotalTCBox{\code}{ v O{} }
{ 
  fontupper=\ttfamily\color{black},
  nobeforeafter,
  tcbox raise base,
  colback=codebackground,colframe=white,
  top=0pt,bottom=0pt,left=0mm,right=0mm,
  leftrule=0pt,rightrule=0pt,toprule=0mm,bottomrule=0mm,
  boxsep=0.5mm,
  #2}{#1}

\patchcmd\newpage{\vfil}{}{}{}
\title{Lipschitz-regularized gradient flows and generative
particle algorithms for high-dimensional scarce data\thanks{
This research  was partially supported by the Air Force Office of Scientific Research (AFOSR) under the grant FA9550-21-1-0354. The research of H. G., M. K. and L.R.-B. was partially supported by the National Science Foundation (NSF) under the grants DMS-2008970, DMS-2307115 and TRIPODS CISE-1934846. The research of P. B.  was partially supported by the National Science Foundation (NSF) under the grant DMS-2008970. 
The work of Y.P. was partially supported by the Hellenic Foundation for Research and Innovation (HFRI) 
under Project 4753.
}}

\author{Hyemin Gu\thanks{Department of Mathematics \& Statistics, University of Massachusetts Amherst (\texttt{hgu, luc, markos@umass.edu}).}
\and Panagiota Birmpa\thanks{ Department of Actuarial Mathematics \& Statistics, Heriot-Watt University (\texttt{P.Birmpa@hw.ac.uk}).}
\and Yannis Pantazis\thanks{ Institute of Applied \& Computational Mathematics, FORTH, Greece (\texttt{pantazis@iacm.forth.gr}).}
\and Luc Rey-Bellet\footnotemark[2]
\and Markos A. Katsoulakis\footnotemark[2]
}

\ifpdf
\hypersetup{ pdftitle={Lipschitz-Regularized Gradient Flows and Generative Particle Algorithms for High-Dimensional Scarce Data} }
\fi

\begin{document}
\maketitle

\begin{abstract}
We have developed a new class of generative algorithms capable of efficiently learning arbitrary target distributions from possibly scarce, high-dimensional data and subsequently generating new samples. These particle-based generative algorithms are constructed as gradient flows of Lipschitz-regularized Kullback-Leibler or other $f$-divergences. In this framework, data from a source distribution can be stably transported as particles towards the vicinity of the target distribution. 
As a notable result in data integration, we demonstrate that the proposed algorithms accurately transport gene expression data points with dimensions exceeding 54K, even though the sample size is typically only in the hundreds.

\end{abstract}

\section{Introduction and main results}
\label{introduction}

We construct new algorithms that are capable of efficiently transporting samples from a source   distribution to a target data set. The transportation mechanism  is built as the gradient flow (in probability space) for Lipschitz-regularized divergences, \cite{dupuis2022formulation,birrell2022f,birrell2022structure}.  Samples are viewed as particles and are transported along the gradient of the discriminator of the divergence towards the target data set. Lipschitz regularized $f$-divergences interpolate between the Wasserstein metric and $f$-divergences and   provide a flexible family of loss functions to compare non-absolutely continuous probability measures. In machine learning one needs to build algorithms to handle target 
distributions $Q$ which are singular, either by their intrinsic nature such as  probability densities concentrated on low dimensional structures and/or because $Q$ is usually only  known through     $N$ samples.
The Lipschitz regularization also 
provides numerically stable, mesh free, particle algorithms that can act as a generative model for high-dimensional target distributions.
The proposed generative approach is validated on a wide variety of datasets and applications ranging from heavy-tailed distributions and  image generation to gene expression data integration, including problems in very high dimensions and with scarce target data.
In this introduction we provide an outline of our main results, background material and related  prior work.

\paragraph{Generative modeling} In generative modeling, which is a form of unsupervised learning, a  data set  $(X^{(i)})_{i=1}^N$  from an unknown ``target" distribution $Q$ is given and the goal is to construct an approximating  model in the form of a distribution $P\approx Q$ which is easy to simulate, with the goal to generate  additional, inexpensive,  approximate samples from the distribution $Q$. Succinctly, the goal of generative modeling is to learn the target distribution $Q$ from input data  $(X^{(i)})_{i=1}^N$.
This is partly  in contrast to sampling, where typically  $Q$ is known up to normalization.
In the last 10 years, generative modeling has been revolutionized by new innovative algorithms taking advantage of neural networks (NNs) and more generally  deep learning.  On one hand NNs provide enormous flexibility to parametrize functions and probabilities and on the other, lead to efficient optimization algorithms in function spaces. Generative adversarial networks (GANs) \cite{goodfellow2014generative,Wasserstein:GAN}, for example, are able to generate complex 
distributions and are quickly becoming a standard tool  in image analysis, medical data, cosmology, computational chemistry, materials science and so on. 
Many other  algorithms have been proposed since, such as normalizing flows  \cite{noe-kohler20a, NEURALODE}, diffusion models
\cite{sohldickstein2015deep,ho2020denoising}, score-based generative flows \cite{song2020generative, Song2021ScoreBasedGM}, 
variational autoencoders \cite{VAE} and energy-based methods \cite{lecun-06}.  

\paragraph{Information theory, divergences  and  optimal transport} 
Divergences such as Kullback-Leibler (KL) and $f$-divergences, and probability metrics 
such as Wasserstein,  provide a notion of `distance' between  probability distributions, thus allowing for comparison of models with one another and with data. Divergences and metrics are used in many theoretical and practical problems in mathematics, engineering, and the natural sciences, ranging from statistical physics, large deviations theory, uncertainty quantification, partial differential equations (PDE)  and statistics to   information theory, communication theory, and machine learning. 
In particular, in the context of GANs, the choice of  objective functional (in the form of a probability divergence plus a suitable regularization) plays a central role.

A very flexible family of divergences, the $(f, \Gamma)$-divergences, 
 were introduced  in \cite{birrell2022f}. These new divergences  interpolate between $f$-divergences (e.g KL, $\alpha$-divergence, Shannon-Jensen) and $\Gamma$-Integral Probability Metrics (IPM) like  1-Wasserstein and MMD distances (where $\Gamma$ is the 1-Lipschitz functions or an   RKHS 1-ball respectively).  Another way  to think of $\Gamma$ is as a regularization to avoid over-fitting, built directly in the divergence, see for instance structure-preserving GANs \cite{birrell2022structure}. 
 In this paper, we focus  on one specific family which we  view as a Lipschitz regularization of the KL-divergence (or  $f$-divergences) or as an entropic  regularization of the 1-Wasserstein metric. In this context, 
 the interpolation is mathematically described by the Infimal Convolution formula
 \begin{equation}\label{eq:inf_conv}
     D_{f}^{\Gamma_L}(P\|Q)=\inf_{\gamma\in\mathcal{P}(\mathbb{R}^d)}\left\{L \cdot W^{\Gamma_1}(P,\gamma) + D_f(\gamma\|Q) \right\}\, ,
 \end{equation}
where  $\mathcal{P}(\mathbb{R}^d)$ is the space of all Borel probability measures on $\mathbb{R}^d$ and 
   $\Gamma_L=\{\phi:\mathbb{R}^d\to \mathbb{R}: |\phi(x)-\phi(y)|\le L |x-y|\, \textrm{ for all } x,y\}$  is the space of Lipschitz continuous functions with  Lipschitz constant bounded by $L$ (note that $L \Gamma_1 = \Gamma_{L}$). 
  Furthermore, $W^{\Gamma_1}(P,Q)$ denotes the  1-Wasserstein metric  with transport cost $|x-y|$ which is an integral probability metric, and has the dual representation
\begin{equation}\label{eq:IPM}
    W^{\Gamma_1}(P, Q)=\sup_{\phi\in\Gamma_1}\left\{E_{P}[\phi]-E_Q[\phi]  \right\}.
\end{equation}
Finally, if $f: [0,\infty)\to\mathbb{R}$ is strictly convex and lower-semicontinuous with $f(1)=0$ the $f$-divergence of $P$ with respect to $Q$ is defined by  $D_f(P\|Q)= E_Q[f(\frac{dP}{dQ})]$ if $P\ll Q$ and set to be $+\infty$ otherwise.
The new divergences  inherit  desirable properties from both objects, e.g. 
 \begin{equation}\label{eq:fgdivergence:bounds}
     0\leq D_{f}^{\Gamma_L}(P\|Q) \leq\min\left\{ D_f(P\|Q),L \cdot W^{\Gamma_1}(P,Q)\right\}\, .
 \end{equation} 
The Lipschitz-regularized $f$-divergences \cref{eq:inf_conv} admit a dual variational representation,  
\begin{equation}\label{eq:fgdivergence:dual}
D_{f}^{\Gamma_L}(P\|Q):=\sup_{\phi\in\Gamma_L}\left\{E_{P}[\phi]-\inf_{\nu \in \mathbb{R}}\left\{ \nu + E_{Q}[f^*(\phi -\nu)]\right\}\right\}\, ,
\end{equation}
where $f^*$ is the Legendre transform of $f$.
Some of the important properties of Lipschitz regularized $f$-divergences, which summarizes results from 
\cite{dupuis2022formulation,birrell2022f} are given in \ref{sec:appendix:fdivergences}. 
Typical examples of $f$-divergences  include the KL-divergence with  $f_{\mathrm{KL}}(x)=x\log x$, and  the $\alpha$-divergences with $f_{\alpha}(x)=\frac{x^{\alpha}-1}{\alpha(\alpha-1)}$. The corresponding    Legendre transforms are $f^*_{\mathrm{KL}}(y)=e^{y-1}$ and $f^*_{\alpha} \propto y^{\frac{\alpha}{(\alpha-1)}}$.
In the KL case  the infimum over $\nu$ can be solved analytically and yields the Lipschitz-regularized Donsker-Varadhan formula with a $\log E_Q[e^\phi]$ term, see \cite{Birrell_IEEE_2022} for more on variational representations.

 \paragraph{Gradient flows in probability space} The groundbreaking work of \cite{jordan1998variational,otto2001geometry} recasted the Fokker-Planck (FP)  and the porous media equations as gradient flows in the 2-Wasserstein space of probability measures.  
More specifically,  the Fokker-Planck equation can be thought as the gradient flow of the KL divergence 
\begin{equation}
\label{eq:gradflow:FPE:KL}
    \partial_t p_t = \nabla \cdot \left(p_t\nabla \frac{\delta D_{KL}(p_t\| q)}{\delta p_t} \right)\, = \nabla \cdot \left(p_t\nabla  \log\left(\frac{p_t}{q}\right)\right)\,
\end{equation}
where $p_t$ and $q$ are the densities at time $t$ and the stationary density respectively. A similar result relates weighted porous media equation and gradient flows for $f$ divergences \cite{otto2001geometry}. This probabilistic formulation allowed the use of such 
gradient flows and related perspectives  to build new  Machine Learning concepts and tools. For instance, the 
Fokker-Planck equation plays a key role in both generative modeling and in sampling.

In the remaining part of this Introduction we provide an outline of our main results, as well as a discussion of related prior work.

\paragraph{Lipschitz-regularized gradient flows in probability space}
From a generative modeling perspective, where $Q$ is known only through samples—and may not have a density, especially if $Q$ is concentrated on a low-dimensional structure—one cannot use gradient flows such as  \cref{eq:gradflow:FPE:KL} without further regularization.
 For instance, related generative methods such as score matching and diffusion models  regularize data by adding  noise, \cite{song2020generative,Song2021ScoreBasedGM}. 
%
Here we propose  a different and complementary approach by regularizing the divergence directly and without adding noise to the data. 
We propose gradient flows for the Lipschitz-regularized divergences \cref{eq:fgdivergence:dual} of the form 
\begin{equation}\label{eq:fgdivergence:gradflow}
\partial_{t}P_t={\rm div}\left(P_t\nabla \frac{\delta  D_{f}^{\Gamma_L}(P_t\|Q)}{\delta P_t}\right)\, , 
\end{equation}
for an initial (source) probability measure $P_0$ and an equilibrium (target) measure $Q$,
for  $P_0,Q$ in the Wasserstein space $\mathcal{P}_1(\mathbb{R}^d)= \left\{ P \in \mathcal{P}(\mathbb{R}^d) : \int |x| dP(x) < \infty\right\}$. We want to emphasize that $\mathcal{P}_1(\mathbb{R}^d)$  includes singular measures such as empirical distributions constructed from data.
In \Cref{sec:L-reg:gradflow} we prove the first variation formula 
\begin{equation}\label{eq:gradflow:variational:discriminator}
\frac{\delta  D_{f}^{\Gamma_L}(P\|Q)}{\delta P}= \phi^{L,*} =  \underset{\phi\in \Gamma_L}{\rm argmax} \left\{E_P[\phi]- \inf_{\nu \in \mathbb{R}}(\nu + E_{Q}[f^*(\phi-\nu)])\right\}\, .
\end{equation}
The optimal $\phi^{L,*}$ in \cref{eq:gradflow:variational:discriminator} (called the discriminator in the GAN literature) in the variational representation of the divergence \cref{eq:fgdivergence:dual}
serves as a potential to transport probability measures,  leading to 
the \textit{transport/variational} PDE reformulation of 
\cref{eq:fgdivergence:gradflow}:
\begin{equation}
\begin{aligned}
     & \partial_{t}P_t +{\rm div}(P_t v_t^L) = 0 \, , 
      \quad 
 P_{0}=P\in \mathcal{P}_1(\mathbb{R}^d)\, ,
     \label{eq:transport:variational:pde}\\ 
     & v_t^L= -\nabla \phi_t^{L,*} \, , \quad \phi_t^{L,*} = \underset{\phi\in \Gamma_L}{\rm argmax}  \left\{E_{P_t}[\phi]- \inf_{\nu \in \mathbb{R}}(\nu + E_{Q}[f^*(\phi-\nu)])\right\}\, ,
     \end{aligned}
\end{equation} 
where we remind that $\Gamma_L=\{\phi:\mathbb{R}^d\to \mathbb{R}: |\phi(x)-\phi(y)|\le L |x-y|\, \textrm{ for all } x,y\}$.
This transport/variational PDE should be understood in a weak sense since $P_t$ and $Q$ are not necessarily assumed to have densities. However, the purpose of this paper is not to develop the PDE theory for this new gradient flow but rather to first establish its computational feasibility through associated particle algorithms, explore its usefulness in generative modeling for problems with high-dimensional scarce data,  and overall  computational efficiency and scalability.
Given sufficient regularity, along a trajectory of a smooth solution $P_t$ of (\ref{eq:transport:variational:pde}) we have the following  dissipation identity: 
\begin{equation}\label{eq:dissipation:intro}
    \frac{d}{dt}D_{f}^{\Gamma_L}(P_t\|Q)=-I_f^{\Gamma_L}(P_t\|Q) \le 0 \quad  \textrm{ where }  \quad  I_f^{\Gamma_L}(P_t\|Q)= E_{P_t}\left[|\nabla \phi_t^{L,*}|^2\right]
  \end{equation}
and $I_f^{\Gamma_L}(P\|Q)$ is a Lipschitz-regularized version of the Fisher Information.  Due to the transport/variational  PDE (\ref{eq:transport:variational:pde}) $I_f^{\Gamma_L}(P\|Q)$ can be interpreted as a total kinetic energy, see \Cref{sec:L-reg:gradflow},  and   \Cref{sec:GPA} for its practical importance in the particle algorithms introduced next.

\paragraph{Lipschitz-regularized Generative Particle Algorithms (GPA)} 
In the context of generative models, the target $Q$ and the generative model $P_t$ in \cref{eq:fgdivergence:gradflow} are available only through their  samples and associated empirical distributions. However, as it can be seen  from \cref{eq:fgdivergence:bounds} the divergence $D_{f}^{\Gamma_L}(P\|Q)$ can compare directly singular distributions (e.g. empirical measures) without need for extra regularization such as  adding noise to our models. For precisely this reason the proposed gradient flow \cref{eq:fgdivergence:gradflow} is a natural mathematical object to consider as a generative model.

From a computational perspective, it becomes  feasible to solve high-dimensional transport PDE such as \cref{eq:fgdivergence:gradflow} when considering the 
Lagrangian  formulation of the transport PDE in (\ref{eq:transport:variational:pde}), i.e.  the 
ODE/variational problem
\begin{equation}
    \begin{aligned}
        \label{eq:gradient_flow:particles:lagrangian}
    \frac{d}{dt} Y_t &= v_t^L(Y_t)=-\nabla \phi_t^{L,*} (Y_t)\, ,  
      \quad 
 Y_0 \sim P\, ,
    \\ \phi_t^{L,*}&=\underset{\phi\in \Gamma_L}{\rm argmax}\left\{E_{P_t}[\phi]-\inf_{\nu \in \mathbb{R}}\left\{ \nu + E_{Q}[f^*(\phi -\nu)]\right\}  \right\}\, . 
    \end{aligned}
\end{equation}
%
In order to turn \cref{eq:gradient_flow:particles:lagrangian} into a particle algorithm we need the following ingredients:
\begin{itemize}
    \item Consider samples  $(X^{(i)})_{i=1}^N$  from the target $Q$ and $(Y^{(i)})_{i=1}^M$ samples from an initial (source) distribution $P=P_0$. In this case for the corresponding empirical measures $\widehat{Q}^N$ and $\widehat{P}^M$ we will consider the gradient flow 
    \cref{eq:fgdivergence:gradflow} for $D_{f}^{\Gamma_L}(\widehat{P}^M\|\widehat{Q}^N)$.
   %
    A key observation in our algorithms is that 
    the divergence $D_{f}^{\Gamma_L}(\widehat{P}^M\|\widehat{Q}^N)$ is always well-defined and finite
    due to Lipschitz regularization and \cref{eq:fgdivergence:bounds}.
    
    \item Corresponding estimators for the objective functional in the variational representation of the  divergence $D_{f}^{\Gamma_L}(\widehat{P}^M\|\widehat{Q}^N)$, see 
    \cref{eq:fgdivergence:dual}  and also
    \cref{eq:gradient_flow:particles:lagrangian}:
    \[
E_{\widehat{P}^M}[\phi]- \inf_{\nu}(\nu + E_{\widehat{Q}^N}[f^*(\phi-\nu)])= \frac{\sum_{i=1}^{M}\phi(Y_n^{(i)})}{M}- \inf_{\nu \in \mathbb{R}}\left\{ \nu + \frac{\sum_{i=1}^{N}f^*(\phi(X^{(i)})-\nu)}{N}\right\} \, .
    \]
 \item    The function space $\Gamma_L$ in \cref{eq:gradient_flow:particles:lagrangian} is approximated by a space of neural network approximations $\Gamma_L^{NN}$. The Lipschitz condition can be implemented via neural network spectral normalization  as discussed in \Cref{sec:GPA}.
 \item The transport ODE  in \cref{eq:gradient_flow:particles:lagrangian} is  discretized in time using an Euler or a higher order scheme, see  \Cref{sec:GPA}. Furthermore the gradient  $\nabla \phi_t^{L,*}$ is evaluated by automatic differentiation of neural networks at the positions of the particles.
\end{itemize}
 By incorporating  these approximations we 
 derive from \cref{eq:gradient_flow:particles:lagrangian}, upon Euler time discretization the   \textit{Lipschitz-regularized generative particle algorithm} (GPA):
\begin{equation}
\begin{aligned}
\label{eq:GPA:algorithm}
    Y_{n+1}^{(i)}&=Y_n^{(i)}-\Delta t\nabla {\phi}_n^{L, *}(Y_n^{(i)})\,, \quad Y^{(i)}_0=Y^{(i)}\, , \, Y^{(i)} \sim P\, , \quad i=1,..., M
    \\ 
   {\phi}_n^{L,*}&=\underset{\phi\in \Gamma_L^{NN}}{\rm argmax} \left\{\frac{\sum_{i=1}^{M}\phi(Y_n^{(i)})}{M}- \inf_{\nu \in \mathbb{R}}\left\{ \nu + \frac{\sum_{i=1}^{N}f^*(\phi(X^{(i)})-\nu)}{N}\right\}\right\} \, ,  
  \end{aligned}
\end{equation}
Besides the transport aspect of \cref{eq:GPA:algorithm}, it  can be also viewed as a new generative algorithm,   where the input is samples 
$(X^{(i)})_{i=1}^N$
from the ``target" $Q$.  Initial data, usually referred to as ``source" data ,
$(Y_0^{(i)})_{i=1}^M$
from $P$ are transported  via  \cref{eq:GPA:algorithm}, after time  $T=n_T\Delta t$, where $n_T$ is the total number of steps,   to a new   set of generated data $(Y_{n_T}^{(i)})^{M}_{i=1}$ 
that approximate samples from  $Q$. 
See for instance the demonstration  in \Cref{fig:3dsierpinskicarpet}. 

In analogy to \cref{eq:gradient_flow:particles:lagrangian}, this Lagrangian point of view has been recently introduced  to write  the solution of  the Fokker-Planck equation \cref{eq:gradflow:FPE:KL} as the density of particles evolving according to its Lagrangian formulation, \cite{MaoutsaReichOpper},
\begin{align}\label{eq:probflow:KL:intro}
    \frac{d}{dt} Y_t = v_t(Y_t)=\nabla \log q(Y_t) -\nabla \log p_t(Y_t)\, , \quad \mbox{where $Y_t \sim P_t$}\, .
\end{align} 
In fact, in  \cite{Song2021ScoreBasedGM}, the authors proposed the deterministic probability flow \cref{eq:probflow:KL:intro}
 as an alternative to generative stochastic samplers for score generative models
 due to advantages related to obtaining better statistical estimators.
We note here that the score term $\nabla \log p_t(Y_t)$ 
 in \cref{eq:probflow:KL:intro} is not a priori known and can be estimated by score-based methods \cite{ScoreMatching:2005}. 
 In practice, these Lagrangian tools are used both for generation, \cite{Song2021ScoreBasedGM} as well as sampling \cite{Reich_Weissmann_Bayesian, Boffi_EVE_2022}.

\paragraph{Main contributions} 
As discussed earlier,  the purpose of this paper is to introduce the  new Lipschitz-regularized gradient flow \cref{eq:fgdivergence:gradflow}, in \Cref{sec:L-reg:gradflow},  and  subsequently establish its computational feasibility through associated particle algorithms, its computational efficiency and scalability, and explore its usefulness in generative modeling for problems with high-dimensional scarce data.
Towards these goals our main findings can be summarized as follows.

\begin{enumerate}

 \item \textit{
 GPA for generative modeling with  scarce data. }
    We demonstrate that our proposed  GPA, introduced  in \Cref{sec:GPA}, can learn distributions from very small data sets, including   MNIST  and other  benchmarks, often supported on  low-dimensional structures, see \Cref{fig:3dsierpinskicarpet}.
    In \Cref{sec:generalization-overfit} we discuss generalization properties of GPA and strategies for mitigating memorization of target data, which has proved to be  a significant and ongoing challenge in generative modeling. %
In \Cref{sec:example:scarceMNIST}  we compare GPA to  GANs and score-based generative models (SGM) in a series of examples and show GPA to be an effective data-augmentation tool. 


\item \textit{Lipschitz-regularization}.  We demonstrate that Lipschitz-regularized divergences provide a well-behaved pseudo-metric between models and data or data and data. 
They  remain  finite under very broad conditions, making the training of generative particle algorithms \cref{eq:GPA:algorithm} on data always well-defined and numerically stable. 
In fact, 
Lipschitz regularization corresponds to effectively imposing an advection-type   Courant – Friedrichs – Lewy (CFL) numerical stability condition on the Fokker-Planck PDE \cref{eq:gradflow:FPE:KL} through the Lipschitz-regularization parameter $L$ in \cref{eq:fgdivergence:gradflow}.
The example in \Cref{sec:example:cfllipschitzregularization}  demonstrates empirically that the selection of $L$ is important.

\item \textit{Choice of $f$-divergence in \cref{eq:fgdivergence:gradflow}}. Although KL is often a natural choice, a careful selection of $f$-divergences, for example the family of $\alpha$-divergences  where $f_{\alpha}=\frac{x^{\alpha}-1}{\alpha(\alpha-1)}$, will allow for  training  that is numerically stable, including examples with  heavy-tailed data, see \Cref{sec:example:heavytails}.

\item \textit{Latent-space GPA for very high-dimensional problems.} GPA can be effective even for scarce data sets in high dimensions. We provide a demonstration  where we integrate (real) gene expression data sets 
exceeding 50,000 dimensions. The goal of  data transportation  in this context is to mitigate batch effects between studies of different  groups of patients, see \Cref{sec:example:batch:effects}. 
From a practical  perspective, to be able to operate in such high-dimensions we need a latent-space representation of the data and subsequently we use GPA to transport particles in the latent space.  In \Cref{sec:dpi:latentGPA} we provide related \textit{performance guarantees} 
using  a new  Data Processing Inequality (DPI) for Lipschitz-regularized divergences.


\end{enumerate}

\begin{figure}
    \centering
    \includegraphics[width=.24\textwidth]{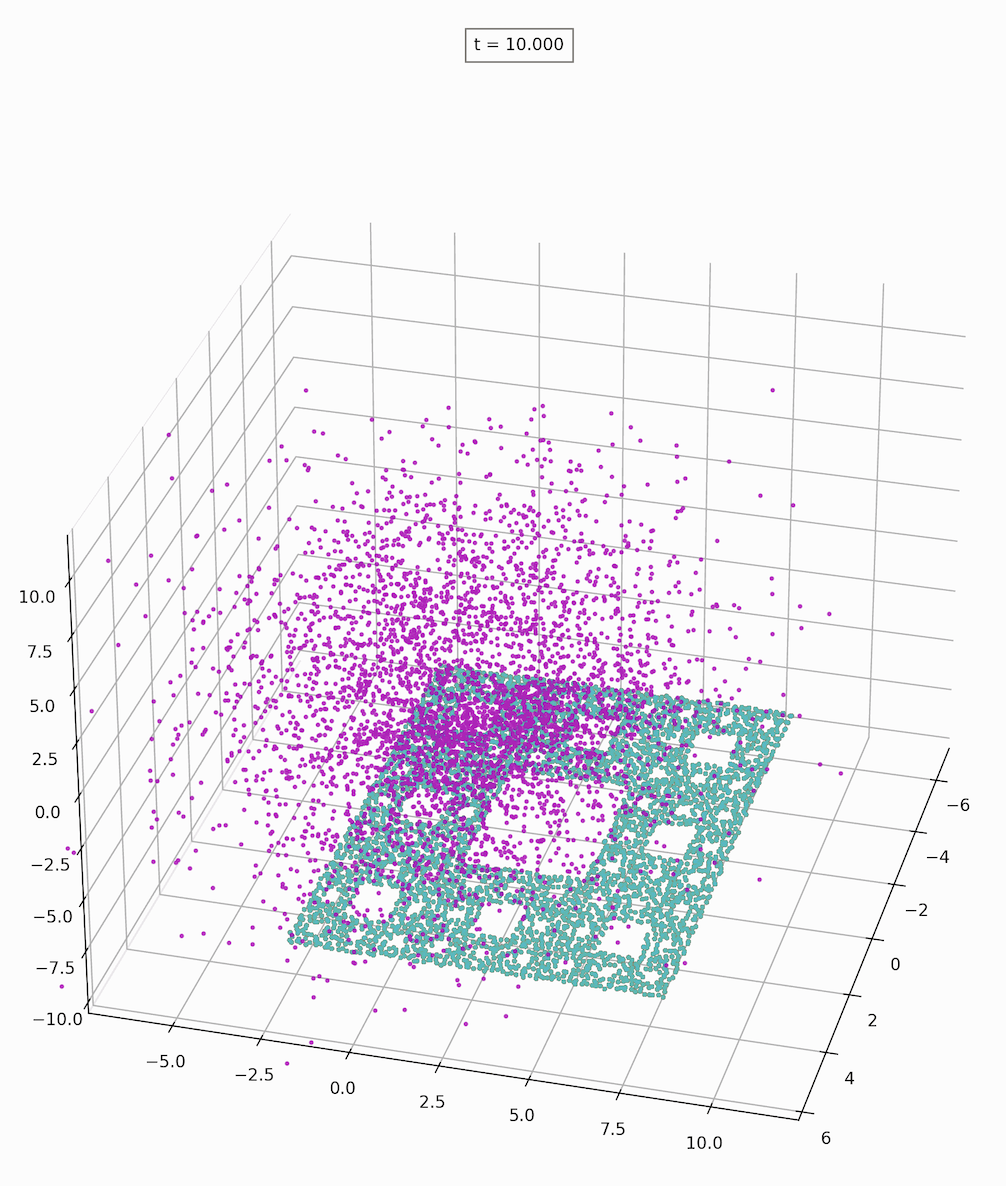}
    \includegraphics[width=.24\textwidth]{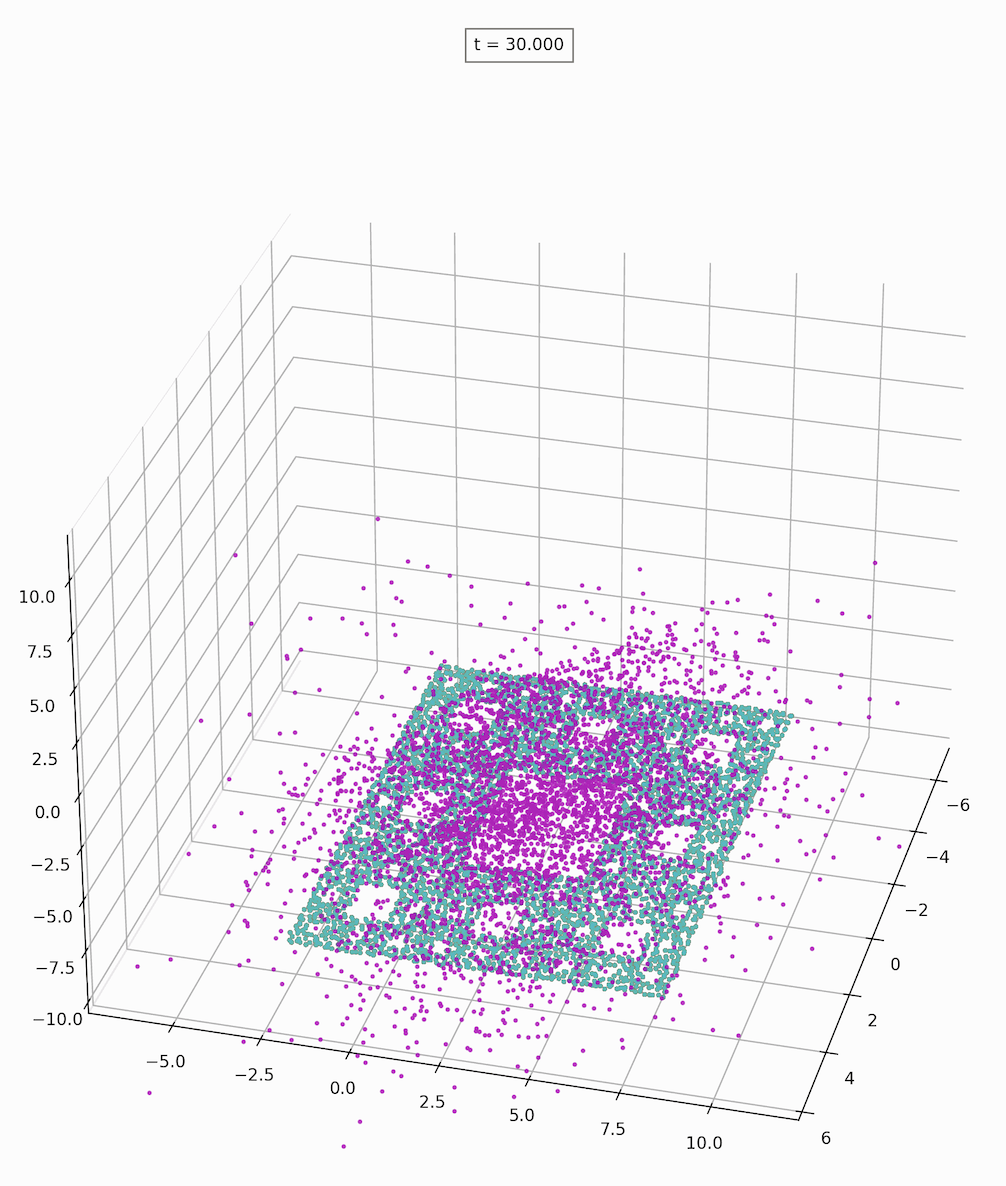}
    \includegraphics[width=.24\textwidth]{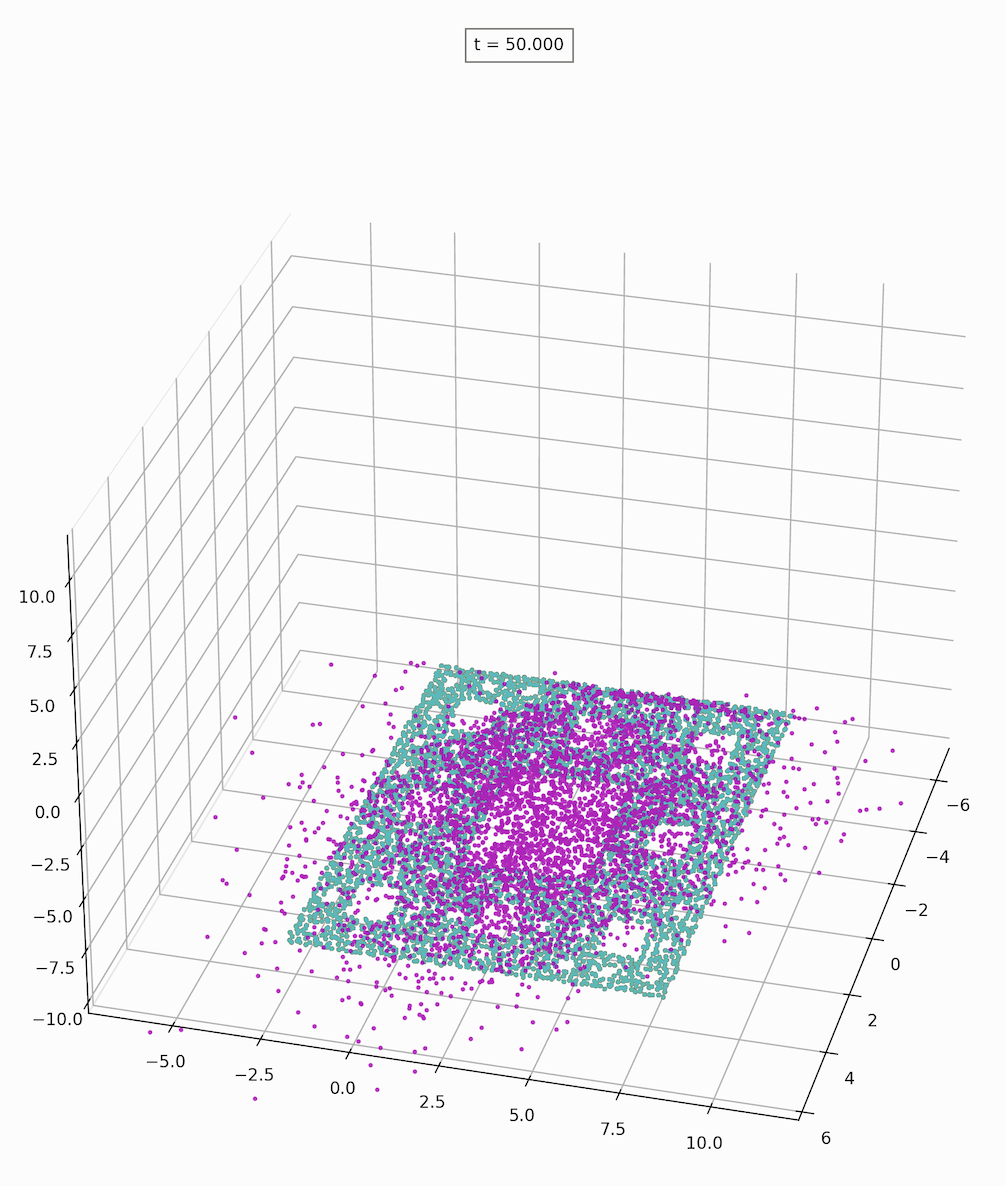}
    \includegraphics[width=.24\textwidth]{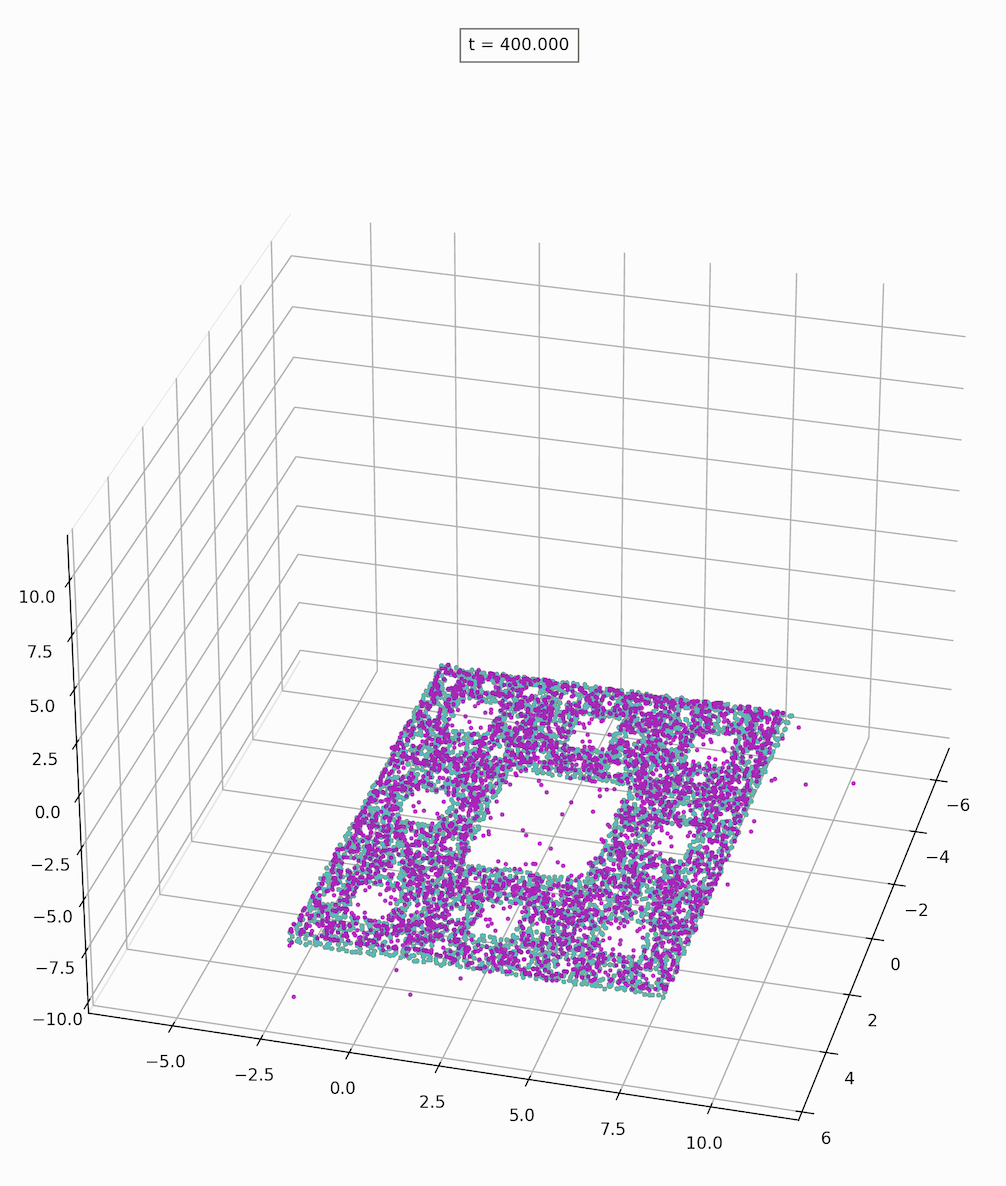}
    \caption{Sierpinski carpet embedded in 3D. Source data (purple particles) are transported via GPA close to the target data (cyan particles). The target particles were  sampled from a Sierpinski carpet of level $4$ by omitting all finer scales. See \cref{fig: Sierpinski carpet} for a related 2D demonstration and a comparison to GANs.}
    \label{fig:3dsierpinskicarpet}
\end{figure}

\paragraph{Related work} 
 Our approach is inspired by the MMD and KALE gradient flows from \cite{arbel2019maximum,glaser2021kale} based on an entropic regularization of the MMD metrics, and related work using the Kernelized Sobolev Discrepancy \cite{mroueh2019sobolev}. Furthermore, the recent work of \cite{dupuis2022formulation,birrell2022f} built the mathematical foundations for a large class of new divergences which contains the Lipschitz regularized $f$-divergences and used them to construct GANs, and in particular symmetry preserving GANs \cite{birrell2022structure}. 
Also related is the Sinkhorn divergence \cite{Cuturi:stoch_optim_OT:2016} which is a different entropic regularization of the 2-Wasserstein metrics. 
Lipschitz regularizations  and the  related spectral normalization have been shown 
to improve the stability of GANs \cite{Miyato_spectral_norm,Wasserstein:GAN, GP-WGAN}. 
Our particle algorithms share similarities with GANs \cite{goodfellow2014generative, Wasserstein:GAN}, sharing the same  discriminator but having a different generator step. They are also   broadly related to continuous-time generative algorithms, such as continuous-time normalizing flows (NF) \cite{Chen_CNF_2017, noe-kohler20a, NEURALODE},     diffusion models \cite{sohldickstein2015deep,ho2020denoising} and score-based generative flows \cite{song2020generative, Song2021ScoreBasedGM}.
However, the aforementioned  continuous-time models, along with variational autoencoders \cite{VAE} and energy based methods \cite{lecun-06},  are mostly KL/likelihood-based.  

On the other hand, particle gradient flows such as the ones proposed here, can be classified as a separate  class  within  implicit generative models. 
Within such generative models that  include GANs, there is more flexibility in selecting  the loss function in terms of a suitable divergence or probability metric, enabling the direct comparison of even mutually singular distributions, e.g. \cite{Wasserstein:GAN, GP-WGAN}. 
Gradient flows in probability spaces related to the Kullback-Leibler (KL) divergence, such as the Fokker-Planck equations and Langevin dynamics \cite{MALA,ULA} or Stein variational gradient descent \cite{LiuSVGD2016,LiuSVGD2017,LuLuSVGD}, form the basis of a variety of sampling algorithms when the target distribution $Q$ has a known density (up to normalization). The weighted porous media equations form another family of gradient flows based on $\alpha$-divergences, e.g.  \cite{
otto2001geometry, ambrosio2005gradient,dolbeault2007lq, vazquez2014barenblatt} 
which are very useful in the presence of heavy tails. 
Our gradient flows are  Lipschitz-regularizations of such classical PDE's (Fokker-Planck and porous medium equations). 
Finally, deterministic particle methods and associated probabilistic flows of ODEs such as the ones derived here for Lipschitz-regularized gradient flows,  were considered in recent works for classical KL-divergences and associated  Fokker-Planck equations as sampling tools \cite{MaoutsaReichOpper, Boffi_EVE_2022}, for Bayesian inference \cite{Reich_Weissmann_Bayesian} and as generative models \cite{Song2021ScoreBasedGM}.

\section{Lipschitz-regularized gradient flows} 
\label{sec:L-reg:gradflow}

In this section we introduce the concept of  
Lipschitz-regularized gradient flows in probability space,
including the key computation of 
the first variation of Lipschitz-regularized divergences. This will allow us to build effective particle-based algorithms in \Cref{sec:GPA}. Indeed, given a target probability measure $Q$,  
we build an evolution equation
for probability measures based on the Lipschitz regularized $f$-divergences $D_f^{\Gamma_L}(P\|Q)$ in \cref{eq:fgdivergence:dual}, by considering the PDE
\begin{equation}\label{eq:fgdivergence:gradflow:sec2}
\partial_{t}P_t={\rm div}\left(P_t\nabla \frac{\delta  D_{f}^{\Gamma_L}(P_t\|Q)}{\delta P_t}\right)\,,\quad \text{with initial condition} \quad P_{0}\in \mathcal{P}_1(\mathbb{R}^d)
\end{equation}
where $\frac{\delta D_{f}^{\Gamma_L}(P\|Q)}{\delta P}$ is the first variation of $D_{f}^{\Gamma_L}(P\|Q)$, to be discussed below in Theorem \ref{thm:first variation}. 
An advantage of the Lipschitz regularized $f$-divergences is its ability to compare singular measures and thus   \cref{eq:fgdivergence:gradflow:sec2} needs  to be understood in a weak sense.  For this reason we use the probability measure $P_t$ notation in \cref{eq:fgdivergence:gradflow:sec2}, instead of density notation $p_t$ as in the Fokker-Planck (FP) equation \cref{eq:gradflow:FPE:KL}. 
In the formal asymptotic limit $L\to \infty$ and if $P\ll Q$,   \cref{eq:fgdivergence:gradflow:sec2}
yields the FP equation  \cref{eq:gradflow:FPE:KL} (for KL divergence)
and the weighted porous medium equation (for $\alpha$-divergences) \cite{otto2001geometry,dolbeault2007lq}, see \Cref{rem:limit.pde}. Note that the purpose of this paper is not to develop the PDE theory for 
\cref{eq:fgdivergence:gradflow:sec2} but rather to first establish its computational feasibility through associated particle algorithms and  demonstrate its usefulness in generative modeling.

\begin{theorem}[first variation of Lipschitz regularized $f$-divergences]\label{thm:first variation}  
Assume $f$ is superlinear, strictly convex 
and  $P,Q \in \mathcal{P}_1(\mathbb{R}^d)$. We define  
\begin{equation}\label{eq: discriminators:sec2}
     \phi^{L,*}:=\underset{\phi\in \Gamma_L}{\rm argmax} \left\{E_P[\phi]- \inf_{\nu \in \mathbb{R}}\{\nu + E_{Q}[f^*(\phi-\nu)] \} \right\}
    \,.
\end{equation}
where the  optimizer $\phi^{L,*} \in \Gamma_L$ exists, is defined on  $\textrm{supp}(P)\cup \textrm{supp}(Q)$, and is unique up to a constant. 
Subsequently, we extend $\phi^{L,*}$ in all of $\mathbb{R}^d$ using \cref{eq:phi_extensions:sec2}.
Let $\rho$ be a signed measure of total mass $0$ and let $\rho=\rho_+-\rho_-$
    where $\rho_{\pm} \in \mathcal{P}_1(\mathbb{R}^d)$ are mutually singular, 
    i.e., there exist two disjoint sets $X_\pm$ such that $\rho_\pm(A)=\rho_\pm(A \cap X_\pm)$
for all measurable sets $A$.

    \noindent
    If $P+ \epsilon \rho \in \mathcal{P}_1(\mathbb{R}^d)$ for sufficiently small $\epsilon>0$, then 
    \begin{equation}\label{eq: Gateaux derivative}
        \lim_{\epsilon \to 0} \frac{1}{\epsilon}
        \left(D^{\Gamma_L}_f(P+\epsilon \rho \|Q)-D^{\Gamma_L}_f(P\|Q)\right) = \int \phi^{L,*} d\rho \,.
    \end{equation} 
    Then we write 
    \begin{equation}\label{eq: first variation}
    \frac{\delta  D_{f}^{\Gamma_L}(P\|Q)}{\delta P}(P)= \phi^{L,*}
    \,.
\end{equation}
\end{theorem}

\begin{remark}
\label{rem:first variation}
{\rm 
The first variation of the Lipschitz-regularized KL divergence given in  \Cref{thm:first variation}, is defined on $\mathcal{P}_1(\mathbb{R}^d)$ which includes singular measures such as empirical distributions. On the other hand,  the classical Fokker-Planck  \cref{eq:gradflow:FPE:KL} (where $L=\infty$) can be  re-written in a gradient flow formulation 
\begin{equation}\label{eq:FPE:transporter}
\begin{aligned}   
   \partial_t p_t &= \nabla \cdot \left(\nabla \phi^* (x, t) p_t\right)\, , \quad \mbox{where} \\ \quad \phi_t^*&=\log\frac{p_t(x)}{q(x)}= \underset{\phi\in C_b(\mathbb{R}^d)}{\rm argmax}\left\{E_{P_t}[\phi]-\inf_{\nu \in \mathbb{R}}\left\{ \nu + E_{Q}[e^{\phi -\nu-1}]\right\}  \right\}\, 
   \end{aligned}
\end{equation}
is built on the first variation of the (un-regularized) KL divergence  given by 
\[
\frac{\delta D_{KL}(P\| Q)}{\delta P}=\log\frac{dP}{dQ}=\phi^*=
\underset{\phi\in C_b(\mathbb{R}^d)}{\rm argmax}
\left\{E_{P}[\phi]-\inf_{\nu \in \mathbb{R}}\left\{ \nu + E_{Q}[e^{\phi -\nu-1}]\right\}  \right\}
\]
where $C_b(\mathbb{R}^d)$ is the space of all bounded continuous functions on $\mathbb{R}^d$. In this case, the first variation is defined on the space of  probability measures which are absolutely continuous with respect to $Q$.
}
\end{remark}
The proof of Theorem \ref{thm:first variation} is partly based on the next lemma (proof in \ref{subsec:appendix:lemma}).
\begin{lemma}\label{lemma}
Let  $f$ be superlinear 
and strictly convex and  $P,Q \in \mathcal{P}_1(\mathbb{R}^d)$. For $y\notin {\rm{supp}}(P) \cup {\rm{supp}}(Q)$, we define 
\begin{equation}
\label{eq:phi_extensions:sec2}
    \phi^{L,*}(y)=\sup_{x\in {\rm{supp}}(Q)}\left\{\phi^{L,*}(x)+ L|x-y|\right\}\, . 
\end{equation}
Then $\phi^{L,*}$ is Lipschitz continuous on $\mathbb{R}^d$  with Lipschitz constant $L$ 
and $\phi^{L,*}=\sup\{h(x): h\in\Gamma_L,\, h(y)=\phi^{L,*}(y), {\textrm{for every }} y\in{\rm{supp}}(Q)\}$.
\end{lemma}
 See \Cref{rem:first.variation}, part (b) for the algorithmic intepretation of this Lemma.

\begin{proof}[Proof of Theorem \ref{thm:first variation}] 

If $\rho=\rho_+-\rho_-$, we may assume (Jordan decomposition) that  $\rho_{\pm}\in \mathcal{P}(X)$ are mutually singular so there exist two disjoint sets $X_\pm$ such that $\rho_\pm(A)=\rho_\pm(A \cap X_\pm)$
for all measurable sets $A$. The measure $P+ \epsilon(\rho_+-\rho_-)$ has total mass $1$ but to be a probability measure we need that
$\epsilon \rho_-(A) \le (P+\epsilon \rho_+)(A)$ holds for all $A$. This implies that $\rho_-$ is absolutely continuous with respect to $P$.  Indeed if $P(A)=0$ then 
\begin{equation}\label{eq:perturb:abs_contin}
\epsilon \rho_-(A)=\epsilon \rho_-(A \cap X_-) \le P(A\cap X_-) + \epsilon \rho_+(A \cap X_-) \le P(A) =0.   
\end{equation}
If $P+\epsilon \rho\in\mathcal{P}_1(\mathbb{R}^d)$  the divergence is finite and thus by \cref{eq:fgdivergence:dual} 
\begin{eqnarray}
D_{f}^{\Gamma_L}(P+\epsilon \rho\|Q)&=&\sup_{\phi\in\Gamma_L}\left\{E_{P+\epsilon \rho}[\phi]-\inf_{\nu \in \mathbb{R}}\left\{ \nu + E_{Q}[f^*(\phi -\nu)]\right\}\right\}\nonumber\\
&\geq&\int \phi^{L,*}\, d(P+\epsilon \rho)-\inf_{\nu \in \mathbb{R}}\left\{ \nu +\int f^*(\phi^{L,*} -\nu)dQ\right\}\nonumber\\
&=&\epsilon \int \phi^{L,*}d\rho+D_{f}^{\Gamma_L}(P\|Q)
\end{eqnarray}
Thus 
\begin{equation} \label{eq:leftderivative}
\liminf_{\epsilon\to 0^{+}}\frac{1}{\epsilon}\left(D_{f}^{\Gamma_L}(P+\epsilon \rho\|Q)-D_{f}^{\Gamma_L}(P\|Q)\right)\geq \int \phi^{L,*}d\rho
\end{equation}
For the other direction let us define $F(\epsilon)=D_{f}^{\Gamma_L}(P+\epsilon \rho\|Q)$.  By \cref{thm:fgdivergence_properties:appendix}
 $F(\epsilon)$ is convex, lower semicontinuous and finite on $[0,\epsilon_0]$. Due to the convexity of  $F$, it  is differentiable on $(0,\epsilon_0)$ except for a countable number of points. If $\phi^{L,*}_{\epsilon}$ is the optimizer for $D_{f}^{\Gamma_L}(P+\epsilon \rho\|Q)$ we have, using the same argument as before,  
\begin{equation}
    D_{f}^{\Gamma_L}(P+(\epsilon+\delta) \rho\|Q)-D_{f}^{\Gamma_L}(P+\epsilon\rho\|Q)\geq \delta\int \phi_{\epsilon}^{L,*}d\rho
\end{equation}
\begin{equation}
    D_{f}^{\Gamma_L}(P+(\epsilon-\delta) \rho\|Q)-D_{f}^{\Gamma_L}(P+\epsilon\rho\|Q)\geq -\delta\int \phi_{\epsilon}^{L,*}d\rho
\end{equation}
If $F$ is differentiable at $\epsilon$ this implies that   
\begin{eqnarray}
\int \phi_{\epsilon}^{L,*}d\rho&\leq&\lim_{\delta\to0}\frac{1}{\delta}\left(D_{f}^{\Gamma_L}(P+(\epsilon+\delta) \rho\|Q)-D_{f}^{\Gamma_L}(P+\epsilon\rho\|Q)\right)\nonumber=F'(\epsilon)\\
&=&\lim_{\delta\to0}\frac{1}{\delta}\left(D_{f}^{\Gamma_L}(P+\epsilon\rho\|Q)-D_{f}^{\Gamma_L}(P+(\epsilon-\delta) \rho\|Q)\right)\leq \int  \phi_{\epsilon}^{L,*}d\rho\,.
\end{eqnarray}
Consequently,
\begin{equation}\label{eq:F'}
    F'(\epsilon)=\int  \phi_{\epsilon}^{L,*}d\rho\,.
\end{equation}
Let $F_{+}'(0)$ be the right derivative at $\epsilon=0$, i.e. $F_{+}'(0)=\lim_{\epsilon\to 0^+}\frac{1}{\epsilon}\left(F(\epsilon)-F(0)\right)$. By convexity, for any sequence $\epsilon_{n}$ such that $F$ is differentiable at $\epsilon_n$ and $\epsilon_n \searrow 0$, we have
\[
F_{+}'(0)=\lim_{n\to\infty}F'(\epsilon_{n})=\lim_{n\to\infty}\int  \phi_{\epsilon_n}^{L,*}d\rho\, .
\]
We write  $\mathbb{R}^d=\cup_{m\in\mathbb{N}} K_m$ with $K_m\subset\mathbb{R}^d$ being compact set and $K_m\subset K_{m+1}$.  The optimizer $\phi_{\epsilon_n}^{L,*}$  are unique up to constant which we choose now such that $\phi_{\epsilon_n}^{L,*}(0)=0$. 
The Lipschitz condition implies that the sequence $\phi_{\epsilon_n}^{L,*}$ is equibounded and equicontinuous on $K_m$.  By the Arzel\`a-Ascoli
theorem, there exists a subsequence of 
$\phi_{\epsilon_n}^{L,*}$ that converges uniformly in
$K_m$. Using diagonal argument, by taking subsequences sequentially along $\{K_m\}_{m\in\mathbb{N}}$ we conclude there exists a subsequence 
such that $\phi_{\epsilon_{n_k}}^{L,*}$ converges uniformly in any $K_m$ and thus $\phi_{\epsilon_{n_k}}^{L,*}$ converges pointwise in $\mathbb{R}^d$. Let $\phi_0^{L,*}\in {\rm{Lip}}^L(\mathbb{R}^d)$ be the limit and for simplicity we also denote by $\phi_{\epsilon_n}^{L,*}$ the convergent subsequence.  The choice  $\phi_{\epsilon_n}^{L,*}(0)=0$ and the Lipschitz condition implies that 
$|\phi^{L,*}_{\epsilon_n}(x)|\leq L|x|$
which is integrable with respect to $\rho$ since $\rho_{\pm}\in \mathcal{P}_{1}(X)$.
Thus by dominated convergence 
\[
F_{+}'(0)=\lim_{n\to\infty}\int  \phi_{\epsilon_n}^{L,*}d\rho= \int \phi_0^* d\rho\, .
\]
By the lower semicontinuity of $D_f^{\Gamma_L}(\cdot\|Q)$, see \cref{thm:fgdivergence_properties:appendix}, we have 
\begin{equation}\label{eq:dfeq}
\begin{aligned}   
D_f^{\Gamma_L}(P\|Q)&\leq\liminf_{n\to\infty}D_f^{\Gamma_L}(P+\epsilon_n\rho\|Q) \nonumber\\
  &= \liminf_{n\to\infty}\left\{E_{P+\epsilon_n \rho}[\phi_{\epsilon_n}^{L,*}]-\inf_{\nu \in \mathbb{R}}\left\{ \nu + E_{Q}[f^*(\phi_{\epsilon_n}^{L,*} -\nu)]\right\}\right\}\nonumber\\
  &=\liminf_{n\to\infty}E_{P+\epsilon_n \rho}[\phi_{\epsilon_n}^{L,*}]-\limsup_{n\to\infty}\inf_{\nu \in \mathbb{R}}\left\{ \nu + E_{Q}[f^*(\phi_{\epsilon_n}^{L,*} -\nu)]\right\}\nonumber\\
  &\leq E_{P}[\phi_{0}^{L,*}]-\inf_{\nu \in \mathbb{R}}\left\{ \nu + E_{Q}[f^*(\phi_{0}^{L,*} -\nu)]\right\} \leq D_f^{\Gamma_L}(P\|Q)
\end{aligned}
\end{equation}
where for the second inequality we use the dominated convergence theorem, \cref{eq:F'} and that by Fatou's lemma, (using that $f^*(x)\ge x$
and that $|\phi^{L,*}_{\epsilon_n}(x)|\leq L|x|$),
\[
\limsup_{n\to\infty}\int  f^*(\phi_{\epsilon_n}^{L,*})dQ\geq\liminf_{n\to\infty}\int  f^*(\phi_{\epsilon_n}^{L,*})dQ\geq\int  f^*(\phi_{0}^{L,*})dQ\,.
\]
From  \eqref{eq:dfeq} we conclude that $\phi_{0}^{L,*}$ must be an optimizer, and thus $\phi_{0}^{L,*}(x) = \phi^{L,*}(x)$, $P$ a.s., 
and $\phi_{0}^{L,*} (x) \le \phi^{L,*}(x)$ for all $x$ (see \Cref{lemma}).
Using that $\rho_{-}$ is absolutely continuous with respect to $P$ we have then 
\begin{equation}\label{eq:mainthm: Maximal_Lip_extension}
    F_{+}'(0)= \int \phi_0^{L, *} d\rho 
=\int \phi_0^{L, *} d\rho_+ -\int \phi_0^{L, *} d\rho_- =\int \phi_0^{L, *} d\rho_+ -\int \phi^{L, *} d\rho_- \leq \int \phi^{L, *} d\rho.
\end{equation}
Combining with \cref{eq:leftderivative} implies that $F_{+}'(0)=\int \phi^{L, *} d\rho$.  
\end{proof}

\begin{remark}[Algorithmic perspectives and related results]\label{rem:first.variation} {\rm
   The statement and the proof of \Cref{thm:first variation} contain certain key algorithmic elements that will become relevant in later sections: \textbf{(a)} A version of  \Cref{thm:first variation} was proved   in \cite{dupuis2022formulation} for the special case of KL divergence. In \Cref{thm:first variation} our results are proved for  general  $f$-divergences. This  generality is necessary in generative modeling based on both past experience in GANs \cite{f-GAN,LS_GAN,birrell2022f, birrell2022structure}, as well as the demonstration examples with heavy tails considered here.
    \textbf{(b)} In \Cref{thm:first variation}, the maximizer  $\phi^{L,*} \in \Gamma_L$ defined on  $\textrm{supp}(P)\cup \textrm{supp}(Q)$, is maximally extended as an $L$-Lipschitz function to all of $\mathbb{R}^d$, see \Cref{lemma}. Notice that in our algorithms in \Cref{sec:GPA},  we  also allow for 
    $L$-Lipschitz extensions which  are constructed algorithmically simply by optimization in the space of $L$-Lipschitz neural networks, see \cref{algo:gpa}.
    \textbf{(c)} The derived (not assumed!) absolute continuity of the perturbation $\rho$ in \cref{eq:perturb:abs_contin}, captures  some important intuition about the nature of $P+\epsilon \rho$ when $P$ is an empirical measure, e.g. when  it is built from particles as in \cref{algo:gpa}: in this perturbation, existing particles can be removed from $P$ according to $\rho_{-}$, corresponding to  the absolute continuity \cref{eq:perturb:abs_contin}, while new particles can be created anywhere according to $\rho_+$, the latter not requiring absolute continuity. These perturbations/variations  of empirical measures are precisely the ones arising in the particle algorithm \cref{eq:GPA:sec3}.}
\end{remark}

Using  \Cref{thm:first variation} we can now rewrite  \cref{eq:fgdivergence:gradflow:sec2} as a
 \textit{transport/variational} PDE,
\begin{equation}
\begin{aligned}
     & \partial_{t}P_t +{\rm div}(P_t v_t^L) = 0 \, , 
     \quad P_{0}=P\in \mathcal{P}_1(\mathbb{R}^d)\, ,\label{eq:transport:variational:pde:sec2}\\ 
     & v_t^L= -\nabla \phi_t^{L,*} \, , \quad \phi_t^{L,*} = \underset{\phi\in \Gamma_L}{\rm argmax}  \left\{E_{P_t}[\phi]- \inf_{\nu \in \mathbb{R}}(\nu + E_{Q}[f^*(\phi-\nu)])\right\}\, . 
     \end{aligned}
\end{equation}  
The transport/variational reformulation \cref{eq:transport:variational:pde:sec2} is the starting point for developing our generative particle algorithms in \Cref{sec:GPA} based on data, when $P$ and $Q$ are replaced by their empirical measures $\hat{P}^M$, $\hat{Q}^N$
based on $M$ and $N$ i.i.d. samples respectively.  Furthermore, \cref{eq:transport:variational:pde:sec2} provides a numerical stability  perspective on the  Lipschitz regularization \cref{eq:fgdivergence:gradflow:sec2} 
In particular, the Lipschitz condition on $\phi \in \Gamma_L$ enforces a finite speed of propagation of at most $L$ in the transport equation in \cref{eq:transport:variational:pde:sec2}. This is in sharp contrast with the FP equation \cref{eq:gradflow:FPE:KL},
which is a diffusion equation and has infinite speed of propagation. We refer to \Cref{sec:example:cfllipschitzregularization} for connections to the Courant, Friedrichs, and Lewy (CFL) stability condition. 

The gradient flow structure of \cref{eq:fgdivergence:gradflow:sec2} is reflected in dissipation estimates, namely an equation for the rate of change (dissipation)
of the divergence along smooth solutions $P_t$  of \cref{eq:fgdivergence:gradflow:sec2}.
\begin{theorem}[Lipschitz-regularized dissipation]\label{thm:dissipation}
Along a trajectory of a smooth solution $\{P_t\}_{t\geq 0}$ of \cref{eq:transport:variational:pde:sec2} with source probability  $P_0=P$ we have the  rate of decay identity 
\begin{equation}\label{dissipation}
    \frac{d}{dt}D_{f}^{\Gamma_L}(P_t\|Q)=-I_f^{\Gamma_L}(P_t\|Q) \le 0
  \end{equation}
  where we define the Lipschitz-regularized Fisher Information as 
  \begin{equation}\label{eq:fisher_info}
   I_f^{\Gamma_L}(P_t\|Q)= E_{P_t}\left[|\nabla \phi^{L,*}|^2 \right] \,.
    \end{equation}
Consequently, for any $T\geq 0$, we have
$D_{f}^{\Gamma_L}(P_T\|Q) = D_{f}^{\Gamma_L}(P\|Q)-\int_{0}^{T} I_f^{\Gamma_L}(P_s\|Q)ds\,$. 
\end{theorem}
The proof can be found in \ref{subsec:appendix:dissipation}.
For the generative particle algorithms of \Cref{sec:GPA} the Lipschitz-regularized Fisher Information will  be interpreted as the total kinetic energy of the particles \cref{eq:Fisher_on_particles}.

\begin{remark}[Formal asymptotics of Lipschitz-regularized gradient flows]\label{rem:limit.pde}{\rm 
 The rigorous ($L\to \infty$)-asymptotic results of the limit of the Lipschitz-regularized $f$-divergences  to (un-regularized) $f$-divergences presented in \cite{dupuis2022formulation, birrell2022f} (see also \cref{thm:fgdivergence_properties:appendix}),
motivates a discussion  on the formal asymptotics  of the corresponding gradient flows. In particular, the Lipschitz-regularization  $L \to \infty$ asymptotics towards  the (unregularized) gradient flows
can be formally obtained as  the limit of the transport/variational PDEs \cref{eq:transport:variational:pde:sec2},
i.e., 
 \begin{equation}\label{eq: f lip L}
    \underbrace{\partial_{t}P_t = \text{div}\left(P_t\nabla \phi_t^{L,*}\right)}_{\textrm{Lip. regularized $f$-divergence flow}} \quad \quad \underset{L \to \infty}{\longrightarrow} \quad \quad 
    \underbrace{ \partial_{t}P_t = \text{div}\left(P_t\nabla \phi_t^{*}\right)}_{\textrm{$f$-divergence flow}},\;  {\rm where}\; \phi_t^*=f'\left(\frac{dP_t}{dQ}\right)
 \end{equation}
\noindent When  $p_t$, $q$ are the probability densities of $P_t$ and $Q$ respectively, and  $f(x)=f_{\rm{KL}}(x)=x\log(x)$ and $f_\alpha(x)=\frac{x^\alpha-1}{\alpha(\alpha-1)}$, the Lipschitz regularized $f$-divergence flow in \cref{eq: f lip L}
converges to  the classical Fokker-Planck equation given by $\partial_{t}p_t = \text{div}\left(p_t\nabla \log\left(\frac{p_t}{q}\right)\right)$  and Weighted Porous Medium equation given by  $\partial_{t}p_t = \frac{1}{\alpha -1}\text{div}\left(p_t\nabla 
 \left(\frac{p_t}{q}\right)^{\alpha-1}\right)$ respectively.}
 Similarly, when $f=f_{\rm{KL}}$, as $L \to \infty$, we formally recover from \cref{eq:fisher_info} the  usual Fisher information 
$
I_f^{\Gamma}(P\|Q)=E_{P}\left[|\nabla \log\left(\frac{p}{q}\right)|^2\right]
$.
%


\paragraph{Some PDE questions for Lipschitz-regularization}
A rigorous analysis encompassing aspects such as well-posedness, stability, regularity, and convergence to equilibrium $Q$, remains to be explored. {For example,  the DiPerna-Lions theory \cite{ambrosio2017lecture,diperna1989ordinary}} for transport equations with rough velocity fields and its more recent variants could be useful for proving well-posedness.
Additionally, functional inequalities tailored for porous medium and Fokker-Planck equations contribute to proving   convergence of a PDE to its equilibrium such as exponential or polynomial convergence.  Classical examples of such inequalities are Poincar\'e and Logarithmic Sobolev-type inequalities, and generalizations thereof for Fokker-Planck and porous medium equations \cite{ambrosio2005gradient,OTTO2000361,
 dolbeault2007lq}. However, convergence of the new class of PDE gradient flows  \cref{eq:fgdivergence:gradflow:sec2} to their equilibrium states, will require new functional inequalities entailing the Lipschitz-regularized Fisher Information and probability measures $Q$ which may not have densities.

\end{remark}

\section{Generative Particle Algorithms} 
\label{sec:GPA}

\begin{algorithm}[t]
\caption{[$(f, \Gamma_L)$-GPA] Lipschitz regularized generative particles algorithm}
\label{algo:gpa}
\begin{algorithmic}[1]
\Require{$f$ for the choice of $f$-divergence and its Legendre conjugate $f^*$, $L$: Lipschitz constant, 
$n_\text{max}$: number of updates for the particles, $\Delta t$: time step size, $M$: number of initial particles, $N$: number of target particles}
\Require{$W = \{W^{l}\}_{l=1}^D$: parameters for the NN $\phi: \mathbb{R}^d \rightarrow \mathbb{R}$, $D$: depth of the NN, $\delta$: learning rate of the NN, $m_{\text{max}}$: number of updates for the NN.}
\Ensure{\{$Y_{n_\text{max}}^{(i)}\}_{i=1}^M$}
\State{Sample $\{Y_0^{(i)}\}_{i=1}^M\sim P_0=P$, a batch of prior samples}
\State{Sample $\{X^{(j)}\}_{j=1}^N \sim Q$, a batch from the real data}
\State{Initialize $\nu \leftarrow 0$}
\State{Initialize $W$ randomly and $W^l \leftarrow L^{1/D} * W^l/\|W^l\|_2$, $l=1,\cdots,D$ \;} \Comment{$\phi_0^{L}(\cdot; W) \in \Gamma_L$}
\For{$n=0$ {\bf to} $(n_\text{max}-1)$}
 \For{$m=0$ {\bf to} $m_{\text{max}}-1$}
 \State{$grad_{W,\nu}  \leftarrow \nabla_{W, \nu} \left[M^{-1}\sum_{i=1}^M\phi_n^L(Y_n^{(i)}; W) - N^{-1} \sum_{j=1}^N f^*(\phi_n^L(X^{(j)}; W) -\nu) + \nu \right]$}
 \State{$(\nu, W) \leftarrow (\nu, W) + \delta \cdot optimizer(grad_\nu, grad_W)$}
 \State{$W^l \leftarrow L^{1/D} * W^l/\|W^l\|_2$ , $l=1,\cdots,D$}  
 \EndFor \Comment{$\phi_n^{L,*}(\cdot; W) \in \Gamma_L$}
 \State{$Y_{n+1}^{(i)} \leftarrow Y_{n}^{(i)} - \Delta t \nabla\phi_n^{L,*}(Y_n^{(i)}; W),\;\; i=1,\cdots, M\, $} \Comment{forward Euler}

  \EndFor
  \newline
{\rm $L$-Lipschitz continuity is imposed by $W^l \leftarrow L^{1/D} * W^l/\|W^l\|_2$, $l=1,\cdots, D$.}
\end{algorithmic}
\end{algorithm}

In this section we build a numerical algorithm to solve the transport/discriminator gradient flow \cref{eq:transport:variational:pde:sec2} when $N$ i.i.d. samples from the target distribution $Q$ are given. We first discretize the system  in time using a forward-Euler scheme,
\begin{equation}\label{eq:scheme:1:sec3}
\begin{aligned}
    P_{n+1}&=\left(I-\Delta t\nabla \phi_n^{L,*} \right)_{\#}P_n,\;\;\;\mathrm{where}\,  P_0=P \\
    \phi_n^{L,*}&=\underset{\phi\in\Gamma_L}{\argmax}\left\{E_{P_n}[\phi]- \inf_{\nu \in \mathbb{R}} \left\{ \nu + E_{Q}[f^*(\phi-\nu)]\right\}\right\}\, .
    \end{aligned}
\end{equation}
Here, the pushforward measure for a map $T:\mathbb{R}^d \to \mathbb{R}^d$  and $P \in \mathcal{P}(\mathbb{R}^d)$  is denoted by $T_{\#}P$ (i.e. $T_{\#}P(A)=P(T^{-1}(A)$).
Next, given $N$ i.i.d. samples $\{X^{(i)}\}_{i=1}^{N}$ from the target distribution $Q$, we consider the empirical measure ${\hat{Q}}^N=N^{-1}\sum_{i=1}^{N}\delta_{X^{(i)}}$. Likewise, given $M$ i.i.d. samples $\{Y_0^{(i)}\}_{i=1}^{M}$ from a known initial (source) probability measure $P$  and consider the empirical measure $ {\hat{P}}^M=M^{-1}\sum_{i=1}^{M}\delta_{Y_0^{(i)}}$. By replacing the measures $P$ and $Q$ in \cref{eq:scheme:1:sec3} by their empirical measures ${\hat{P}}^M$ and ${\hat{Q}}^N$ we obtain the following particle system. 
\begin{equation}
\begin{aligned}
\label{eq:GPA:sec3}
    Y_{n+1}^{(i)}&=Y_n^{(i)}-\Delta t\nabla {\phi}_n^{L, *}(Y_n^{(i)})\,, \quad Y^{(i)}_0=Y^{(i)}\, , \, Y^{(i)} \sim P\, , \quad i=1,..., M
    \\ 
   {\phi}_n^{L,*}&=\underset{\phi\in \Gamma_L^{NN}}{\rm argmax} \left\{\frac{\sum_{i=1}^{M}\phi(Y_n^{(i)})}{M}- \inf_{\nu \in \mathbb{R}}\left\{ \nu + \frac{\sum_{i=1}^{N}f^*(\phi(X^{(i)})-\nu)}{N}\right\}\right\} \, ,
  \end{aligned}
\end{equation}
where the function space $\Gamma_L$ in \cref{eq:scheme:1:sec3} is approximated by a space of neural network (NN) approximations $\Gamma_L^{NN}$. 
We will refer to  this particle algorithm  as  $(f, \Gamma_L)$-GPA or simply GPA.
The transport mechanism given by \cref{eq:GPA:sec3} corresponds to a linear transport PDE in \cref{eq:transport:variational:pde:sec2}.  However, between particles nonlinear interactions   are introduced via the discriminator ${\phi}_n^{L,*}$ which in turn depends on all particles in \cref{eq:GPA:sec3} at step $n$ of the algorithm, namely the generated particles $(Y_n^{(i)})_{i=1}^M$, as well as the  ``target" particles $(X^{(i)})_{i=1}^N$.
Notice that  
 $\phi_n^{L,*}$ discriminates the generated samples  at time $n$  from  the target data  using the second equation of \cref{eq:GPA:sec3}, and is not directly using  the generated data of the previous steps  up to step $n-1$. 
Moreover the gradient of the discriminator is computed only at the positions of the particles.

Overall, \cref{eq:GPA:sec3}
is  an approximation scheme of the Lagrangian formulation \cref{eq:gradient_flow:particles:lagrangian} of the Lipschitz-regularized gradient flow \cref{eq:fgdivergence:gradflow}, where we have (a) discretized time, (b) approximated the function space $\Gamma_L$ in terms of neural networks, and (c) used empirical distributions/particles to build approximations of  the target $Q$, (d) used gradient-based optimization methods to approximate the discriminator ${\phi}_n^{L,*}$ such as stochastic gradient descent or the Adam optimizer.
All these elements are combined in \cref{algo:gpa}.

\begin{remark}[Lipschitz regularization for GPA] 
{\rm Lipschitz regularized $f$-divergences are practically advantageous since they allow to calculate divergences between arbitrary empirical measures with non-overlapping supports. Indeed, given a Lipschitz constant $L$, the $L$-Lipschitz regularized $f$-divergence is bounded by $L$ times the 1-Wasserstein metric as stated in \cref{eq:fgdivergence:bounds}
 and  discussed in more detail in \cite{birrell2022f}. 
Therefore a suitable choice of $L$ depending on  data offers numerical tractability for the particle system in \cref{eq:GPA:sec3} and \cref{algo:gpa}.  Without proper Lipschitz regularization, GPA diverges or produces inaccurate solutions as illustrated in \Cref{fig: Mixture_of_gaussians_2D - KL Lip GPA}. 
In our implementation, the Lipschitz regularization is enforced 
via  Spectral Normalization (SN) for neural networks, \cite{Miyato_spectral_norm}.
Despite its clear numerical benefits, SN incurs a relatively modest computational cost. 
Applying SN in an experiment leads to a 10\% increase in computational time compared to a non-regularized counterpart. Another way to impose Lipschitz regularization for neural networks is to add a gradient penalty to the loss \cite{GP-WGAN, birrell2022f}.}

\end{remark}

\begin{remark}[Improved accuracy and higher-order 
schemes]
{\rm 
Replacing  the forward Euler in \cref{eq:GPA:sec3} or Line 10 in \cref{algo:gpa} with Heun's predictor/corrector method is observed to lead to a significant improvement in the accuracy of the GPA for several examples, see for instance \Cref{fig:12D - 2D gaussian mixtures}. In addition, adopting a smaller $\Delta t$ in \cref{eq:GPA:sec3} and \cref{algo:gpa} may contribute  to enhanced accuracy in GPA outcomes.    Employing a smaller $\Delta t$ often requires a smoother discriminator, achieved by substituting the ReLU activation function with a smoothed ReLU.
We refer to \ref{subsec:settings:improve:accuracy} for details. 
}
\end{remark}

\paragraph{GPA kinetic energy and  Lipschitz-regularized Fisher Information}
 \Cref{thm:dissipation} suggests the empirical  Lipschitz-regularized Fisher Information,  
\begin{equation}\label{eq:Fisher_on_particles}
    I_f^{\Gamma_L}(\hat{P}_n^M\|\hat{Q}^N)=\int|\nabla {\phi}_n^{L, *}|^2 \hat{P}_n^M(dx)=\frac{1}{M}\sum_{i=1}^M |\nabla {\phi}_n^{L,*}(Y^{(i)}_n)|^2\, ,
\end{equation}
as a quantity of interest to monitor the convergence of GPA \cref{eq:GPA:sec3}. 
Here $\hat{P}_n^M$ denotes the empirical distribution of the generative particles $(Y_n^{(i)})_{i=1}^M$.
Indeed,  $I_f^{\Gamma_L}(\hat{P}_n^M\|\hat{Q}^N)$ is  the total kinetic energy of the generative particles since $\nabla {\phi}_n^{L,*}(Y^{(i)}_n)$ is the velocity of the $i^{th}$ particle at time step $n$. The algorithm will stop when the total kinetic energy $I_f^{\Gamma_L}(\hat{P}_n^M\|\hat{Q}^N)\approx 0$. 

Overall, 
\cref{algo:gpa} estimates two natural quantities of interest: the Lipschitz regularized $f$-divergence $M^{-1}\sum_{i=1}^M\phi_n^{L,*}(Y_n^{(i)}; W) - N^{-1} \sum_{j=1}^N f^*(\phi_n^{L,*}(X^{(j)}; W) -\nu^*) + \nu^*$ and the Lipschitz regularized Fisher information \cref{eq:Fisher_on_particles}. 
These quantities 
are used to track the progress and  terminate the simulations.


\section{Generalization properties of GPA}\label{sec:generalization-overfit}
{ The transport/discriminator formulation in \cref{eq:scheme:1:sec3} is the core mechanism in GPA, facilitating sample generation by transporting particles through time-dependent vector fields obtained by iteratively solving \cref{eq:GPA:sec3} over time. Ensuring the diversity of generated samples and avoiding ``memorization" of the target data,  is a critical challenge in generative modeling, as discussed extensively in recent publications, for instance in the context of  diffusion models,
\cite{pidstrigach2022scorebased,
somepalli2023diffusion,somepalli2023understanding,gu2023memorization,li2024good,carlini2023extracting}, including empirical \cite{somepalli2023understanding} and theory-based mitigation strategies \cite{zhang2024wasserstein}.  In GPA  as well, there is the theoretical possibility,  based on the gradient flow dynamics and the dissipation estimate in \Cref{thm:dissipation}, that with a rich enough neural network to learn the discriminator, suitable learning rates, and and long enough runs, \Cref{algo:gpa} may reproduce the empirical distribution of the target data, especially when $M=N$. This phenomenon can be observed for the MNIST data set in \Cref{fig:mnist_rich_nn}}. 
To mitigate these challenges and ensure better generalization for the proposed GPA algorithms,  we explore three distinct strategies: 
\begin{enumerate}
\item {\it From training particles to generated particles.}
In this approach  we use $M$ training particles from an initial distribution $P_0$ and $N$ target particles to learn the time-dependent vector fields given by  \Cref{algo:gpa}.  This vector field is  constructed as a neural network  on the \textit{entire} space. Therefore,
we can transport (e.g. simultaneously) any additional number of particles sampled from $P_0$ using this, already learned, vector field. We refer to the latter type of particles as ``generated particles".
See \Cref{fig:mnist_different_method} and \Cref{fig:mnist:scalability} for practical demonstrations of such generated particles.

This approach which is  based on learning a time-dependent vector field aligns with other flow-based generative models such as 
score-based generative models (SGM) \cite{Song2021ScoreBasedGM}, 
and normalizing flows \cite{NEURALODE}. 
However the latter methods are more efficient in learning their time-dependent vector field by employing a corresponding space/time objective functional. We believe that a similar formulation can be built for GPA, by using the mean-field game functionals for Wasserstein gradient flows in \cite{zhang2023meanfield}. We plan to explore this space/time approach in a follow-up work.

\item {\it Imbalanced sample sizes.} In this strategy we choose  $M \gg N$ in \Cref{algo:gpa}. First, we empirically found strong evidence of overfitting and memorization in the $M = N$ case, i.e. 
    training particles eventually match  the target particles. 
    However,  in the setting of the imbalanced sample sizes $M \gg N$
    particles maintain their  sample  diversity. See \Cref{fig:mnist_overfitting}. 
    These different behaviors are captured and quantified by the two estimators  (divergence and kinetic energy) in \Cref{algo:gpa},  compare  the findings in parts (c, e) of \Cref{fig:mnist_overfitting}.

\item {\it GPA for data augmentation.} Lastly, we demonstrate that GPA can serve as a data augmentation tool to train other generative models 
particularly those requiring large sample sizes. For instance, the examples in \Cref{fig:swissroll} and \Cref{fig:mnist:augmentation} showcase the effectiveness of GPA-based data augmentation for GANs.
\end{enumerate}

Overall, GPA learns from target data and training particles,  a time-dependent vector field represented by Lipschitz neural networks  defined on the entire space. In this sense, GPA is expected to gain in \textit{extrapolation} properties since the learned vector field can be used to move arbitrary new particles towards the target data.

\section{Data Processing Inequality and latent space GPA
}
\label{sec:dpi:latentGPA}


Performance degradation is a common challenge for all generative models in high-dimensional settings, a problem that becomes more pronounced in regimes with low sample sizes. For GPA, the optimization of the discriminator within the neural network space exhibits superior scalability, particularly in regimes of  hundreds of dimensions, compared to optimization in RKHS which typically performs well in lower dimensions. 
However, similarly to other neural-based generative models,  GPA faces challenges in really high dimensional problems. 
To overcome this type of scalability constraints, we can take advantage of latent space formulations used in recent papers in generative flows, e.g.  \cite{vahdat2021score,
RombachDiffusionLatent,onken2021otflow}, to complement and scale-up score-based models, diffusion models  and normalizing flows. 
The key idea is simple and powerful as demonstrated in these earlier works: a pre-trained auto-encoder first projects the high-dimensional real space to a lower dimensional latent space and then a generative model is trained in the compressed latent space. Subsequently, the decoder of the auto-encoder allows to map  the data generated in the latent space back to the original high-dimensional space. 

In \Cref{thm:reconstructed_variable_converges_main}, we demonstrate that operating in the latent space can be understood in light of a suitable Data Processing Inequality (DPI) and we provide conditions which guarantee that the error induced by the transportation of a high-dimensional data distribution via combined encoding/decoding and particle transportation  in a  lower dimensional latent space  is  controlled by the error only in the (much more tractable) latent space.
More specifically, we consider the following mathematical setting: i) a probability $Q=Q^\mathcal{Y}$, defined on the original, high dimensional space $\mathcal{Y}$, typically   supported on some low dimensional set  $S \subset \mathcal{Y}=\mathbb{R}^d$; ii)  an  encoder map $\mathcal{E}:  \mathcal{Y} \rightarrow \mathcal{Z}$ where  $\mathcal{Z} \subset \mathbb{R}^{d'}$, $d'<d$ and a decoder map $\mathcal{D}: \mathcal{Z} \rightarrow \mathcal{Y}$  which are invertible in $S$, i.e.
$\mathcal{D} \circ \mathcal{E} (S) = S $.
Let $\mathcal{E}_\# Q^\mathcal{Y}$ denote the image of the measure $Q^\mathcal{Y}$ by the map $\mathcal{E}$, i.e. for 
$A\subset \mathcal{Z}$,  $\mathcal{E}_\# Q^\mathcal{Y}(A):= Q^\mathcal{Y}(\mathcal{E}^{-1}(A))$. Similarly we define $\mathcal{D}_\# P^\mathcal{Z}$ as the combination of the encoding/decoding and particle transportation $\mathcal{T}^n$ in a lower dimensional latent space where
$
P^\mathcal{Z}:=\mathcal{T}^n_{\#}\mathcal{E}_{\#}P_0\,$.
The fidelity of the  approximation $Q^{\mathcal{Y}}\approx \mathcal{D}_\# P^\mathcal{Z}$ of the target measure $Q^{\mathcal{Y}}$ in the original space $\mathcal{Y}$ will be then guaranteed by the  \textit{a posteriori} estimate  in \cref{thm:reconstructed_variable_converges_main}, interpreted  in the sense of numerical analysis, where  the approximation in the compressed latent space $\mathcal{Z}$  bounds the error in the original space $\mathcal{Y}$. Its proof 
is a consequence of  a new, tighter data processing inequality derived in \cite{birrell2022f}, see also \cref{thm: data processing inequality},  that involves both transformation of probabilities and discriminator  space $\Gamma$. 

\begin{theorem}[Autoencoder performance guarantees]
\label{thm:reconstructed_variable_converges_main}
For $Q^\mathcal{Y} \in \mathcal{P}(\mathcal{Y})$, suppose that there is a exact encoder/decoder with  encoder $\mathcal{E}: \mathbb{R}^d \rightarrow \mathbb{R}^{d'}$ and  decoder $\mathcal{D}: \mathbb{R}^{d'} \rightarrow \mathbb{R}^{d}$, where exact means  perfect reconstruction $\mathcal{D}_\# \mathcal{E}_\# Q^\mathcal{Y} = Q^\mathcal{Y}$. 
Furthermore, assume  the decoder is Lipschitz continuous with Lipschitz constant $a_{\mathcal{D}}$.
Then, for any  $P^\mathcal{Z} \in \mathcal{P}_1(\mathcal{Z})$ we have  
\begin{equation}
\label{eq:DPI:aposteriori}
    D_f^{\Gamma_L} ( \mathcal{D}_\# P^\mathcal{Z} \|  Q^\mathcal{Y} ) \leq D_f^{a_\mathcal{D}\Gamma_{L}} (P^{\mathcal{Z}} \| \mathcal{E}_\# Q^\mathcal{Y}).
\end{equation}
\end{theorem}

\begin{proof} \label{proof:latentgpa}
From the data processing inequality \cref{thm: data processing inequality} and using that the composition of Lipschitz functions with Lipschitz constants $L_1, L_2$ is $L_1 L_2$-Lipschitz, we have:
\begin{equation}
    D_f^{\Gamma_L} ( \mathcal{D}_\# P^\mathcal{Z} \| \mathcal{D}_\# \mathcal{E}_\# Q^\mathcal{Y} ) \leq D_f^{a_\mathcal{D}\Gamma_{L}} (P^{\mathcal{Z}} \| \mathcal{E}_\# Q^\mathcal{Y}).
\end{equation}
Since the encoder $\mathcal{E}$ and the decoder $\mathcal{D}$ perfectly reconstruct $Q^{\mathcal{Y}}$, namely $\mathcal{D}_\# \mathcal{E}_\# Q^\mathcal{Y} = Q^\mathcal{Y}$,  we obtain that
\begin{equation}
    D_f^{\Gamma_L} ( \mathcal{D}_\# P^\mathcal{Z} \|  Q^\mathcal{Y} ) \leq D_f^{a_\mathcal{D}\Gamma_{L}} (P^{\mathcal{Z}} \| \mathcal{E}_\# Q^\mathcal{Y}).
\end{equation}
Note also that, if $a_\mathcal{D} \leq 1$,
$
    D_f^{\Gamma_L} ( \mathcal{D}_\# P^\mathcal{Z} \| \mathcal{D}_\# \mathcal{E}_\# Q^\mathcal{Y} ) \leq D_f^{\Gamma_{L}} (P^{\mathcal{Z}} \| \mathcal{E}_\# Q^\mathcal{Y})\, .
$
\end{proof}

 We apply this result in \Cref{sec:example:batch:effects} where 
the merging (transporting) of high-dimensional gene
expression data sets with dimension exceeding 54K in performed in a latent space which is constructed via  Principal Component Analysis (PCA), i.e. a linear auto-encoder.

\begin{remark}[Autoencoder guarantees in generative modeling]
{\rm It is clear that \Cref{thm:reconstructed_variable_converges_main} is a result about autoencoders and it is independent of the choice of any specific transport/generation algorithm in the latent space. In this sense our conclusions from \Cref{thm:reconstructed_variable_converges_main} are   generally applicable to other  latent space methods for generative modeling, such as GANs. }
\end{remark}
%


\section{Lipschitz regularization and numerical stability}
\label{sec:example:cfllipschitzregularization}

In this section, we discuss the numerical stability of GPA induced by Lipschitz regularization. 
The Lipschitz bound $L$ on the discriminator space implies a pointwise bound
$
|\nabla {\phi}_n^{L,*}(Y_n^{(i)})| \le L \, .
$
 Hence the Lipschitz regularization imposes a speed limit $L$ on the particles, ensuring the stability of the algorithm for suitable choices of $L$, as we will discuss next.  

We first illustrate how Lipschitz regularization works in GPA \cref{algo:gpa} in a mixture of 2D Gaussians. 
We explore the influence of the Lipschitz regularization constant $L$ by monitoring the  Lipschitz regularized Fisher information \cref{eq:Fisher_on_particles} (i.e. kinetic energy of particles). In \cref{fig:mixtureofgaussian_ke} we track this quantity in time. We empirically observe that   a proper choice of $L$ enables the particles slow down and eventually stop near the target particles, using \cref{eq:Fisher_on_particles} as a convergence indicator. 
Time trajectories of particles are displayed in \cref{fig: Mixture_of_gaussians_2D - KL Lip GPA}. 
Individual curves in \cref{fig:mixtureofgaussian_ke} result from the Lipschitz regularized $(f_\text{KL}, {\Gamma_L})$-GPA 
with $L=1$, 10, 100, $\infty$. We  fix all other parameters including time step  $\Delta t$, focusing on  the  influence of the Lipschitz constant $L$. 
For $L=1, 10$, the kinetic energy decreases and particles eventually stop. However, without Lipschitz regularization, the particles keep (relatively) high speeds of propagation. 
\Cref{fig: Mixture_of_gaussians_2D - KL Lip GPA} verifies that in this case ($L=\infty$) the algorithm fails to converge.

\begin{figure}[h]
    \centering
    \captionsetup{format=plain}
    \includegraphics[width=0.4\textwidth]{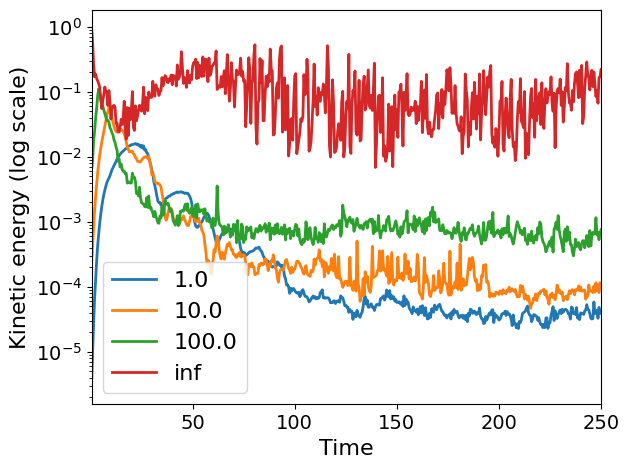}
    \caption{\textbf{(2D Mixture of Gaussians)} 
    Kinetic energy of particles \cref{eq:Fisher_on_particles}  for $(f_\text{KL}, {\Gamma_L})$-GPA with different $L$'s. \cref{thm:dissipation} suggests  that particles need to slow down and practically stop when they reach  the “vicinity”  of the target particles.
    }
    \label{fig:mixtureofgaussian_ke}
\end{figure}

\begin{figure}[h]
    \centering
    \includegraphics[width=\linewidth]{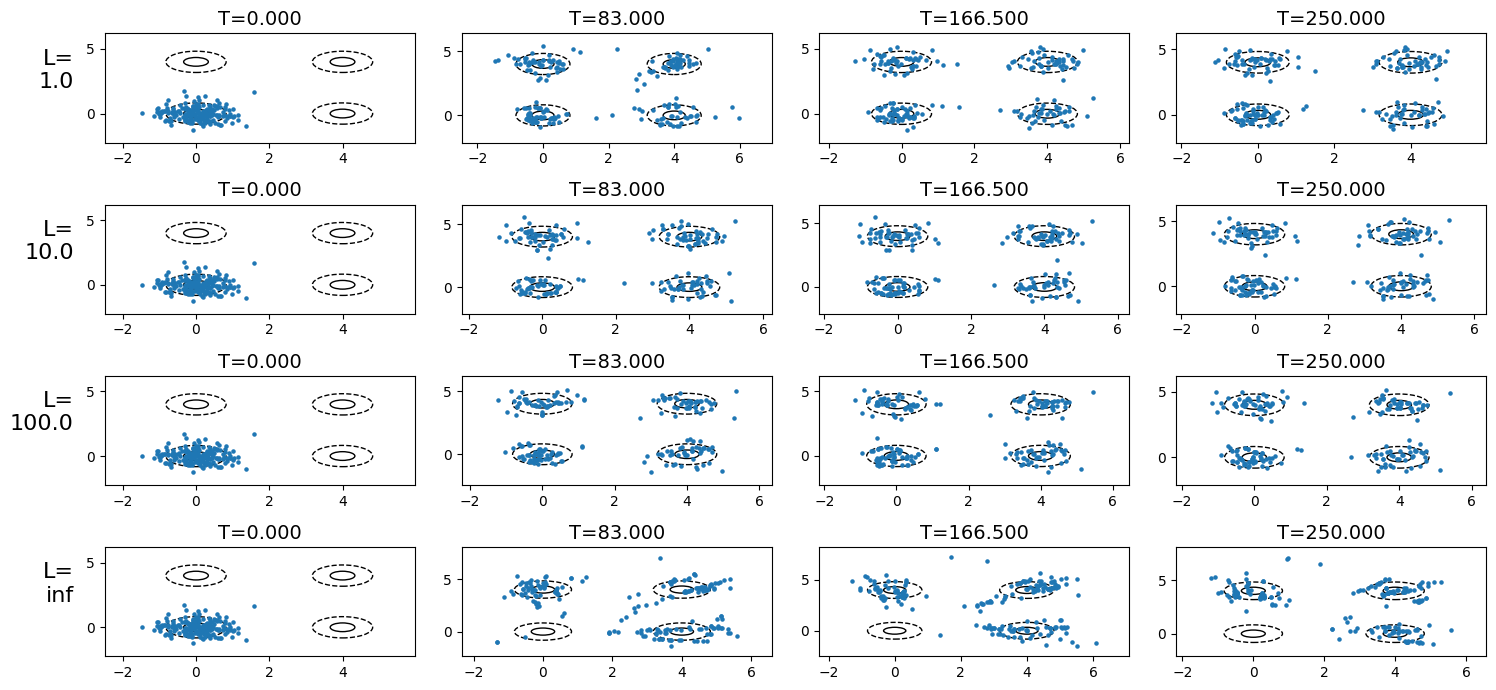}
    \caption{\textbf{(2D Mixture of Gaussians)} We empirically observe that Lipschitz constant $L$ controls the propagation speed of  $(f_\text{KL}, {\Gamma_L})$-GPA with different $L$'s. For $L < \infty$, the particles are propagated to the 4 wells. As $L$ gets larger, the algorithm becomes more unstable. 
    For $L = \infty$ (unregularized KL) GPA fails to capture the target. 
    }
    \label{fig: Mixture_of_gaussians_2D - KL Lip GPA}
\end{figure}

\paragraph{Numerical stability  of GPA}
\label{paragraph:example:cfl}

Based on  these empirical findings, we observe a close relationship between a finite propagation speed $L$ and numerical stability of the algorithm.
Indeed, from  a numerical analysis point of view,
\cref{eq:scheme:1:sec3} is a particle-based explicit scheme for the PDE \cref{eq:transport:variational:pde:sec2}. 
In this context, the Courant, Friedrichs, and Lewy (CFL) condition for stability of discrete schemes for transport PDEs such as the first equation in \cref{eq:transport:variational:pde:sec2}
becomes  
$\sup_x|\nabla \phi_t^{L,*}(x)|\frac{\Delta t}{\Delta x} \le 1$, 
 \cite{LeVeque}.
Clearly, the Lipschitz regularization $|\nabla \phi_t^{L,*}(x)|\le L$ enforces 
a CFL type condition with a learning rate $\Delta t$ proportional to the inverse of $L$. 
It remains an open question how to rigorously extend these CFL-based heuristics to particle-based algorithms, we also refer to some related questions in \cite{Carillo_CFL}.
However, in the context of the algorithm \cref{eq:GPA:sec3}, 
the  speed constraint $L$ on the particles  induces an implicit  spatial discretization grid $\Delta x$ where particles are transported for each $\Delta t$ by at most $\Delta x =  L \Delta t$.
Intuitively, this implicit spatio-temporal  discretization
suggests that $\sup_x|\nabla \phi_t^{L,*}(x)|\frac{\Delta t}{\Delta x} =
    \frac{\sup_x|\nabla \phi_t^{L,*}(x)|}{L} \le 1$.
    Hence  \cref{eq:GPA:sec3} or \cref{algo:gpa} are expected to  
satisfy  the same CFL condition for the transport PDE
in \cref{eq:transport:variational:pde:sec2}.
Based on these CFL heuristics for particles, here we keep the inversely proportional relation between $L$ and $\Delta t$  
as a criterion for tuning the learning rate $\Delta t$. 
Finally, these CFL-based bounds and the empirical findings in \Cref{fig: Mixture_of_gaussians_2D - KL Lip GPA} suggest that a time-dependent ``schedule" for $L$ could enhance the stability and convergence properties of GPAs, as the quantity $\sup_x|\nabla \phi_t^{L,*}(x)|$ could serve as (or inspire) an indicator of proximity to the target distribution. However, in this paper we do not explore further such time-adaptive strategies for $L$.


\section{Generative particle algorithms  for heavy-tailed data}
\label{sec:example:heavytails}

Lipschitz regularized gradient flows in \Cref{sec:L-reg:gradflow} and GPA in \Cref{sec:GPA} are built on a family of $f$-divergences as discussed in \Cref{introduction}. 
Here we study the  
choice   of $f_\text{KL}$ vs. $f_\alpha$ on GPA for  samples  from distributions with various tails, e.g. gaussian, stretched exponential, or polynomial. 
This   exploration rests on the intuition that  transporting a Gaussian  to a heavy-tailed distribution and vice-versa is a nontrivial task.  This is due to the fact that  a significant amount of  mass deep in the tail needs to be transported to and from a (light-tailed) Gaussian. 
Furthermore, for heavy tailed distributions, KL divergence may become infinity, and thus cannot be trained, while in the $f_\alpha$ divergence we have flexibility to accommodate heavy tails using the parameter $\alpha$. However, even with the use of an $f_\alpha$ divergence, transporting particles deep into the heavy tails takes a considerable amount of time due to the speed restriction $L$ of Lipschitz regularization, see \Cref{sec:example:cfllipschitzregularization}. Therefore, in our experiments, we are less focused on ``perfect" transportation and more on ``numerically stable" transportation of moderately heavy-tailed distributions.

Indeed,in our first experiment we observe the following. The choice of $f_\text{KL}$ for heavy-tailed data renders the function optimization step in \cref{eq:GPA:sec3} numerically unstable and eventually leads to the collapse of the algorithm.  On the other hand, the choice of $f_\alpha$ with $\alpha > 1$ makes the algorithm  stable. 
The different behaviors of $f_\text{KL}$ and $f_\alpha$ on heavy-tailed data is illustrated in \cref{fig:2D student-t}  and \cref{fig:2D student-t_additional}. 
\begin{figure}[h]
    \centering
    \begin{subfigure}{0.4\textwidth}
        \centering
        \includegraphics[width=\linewidth]{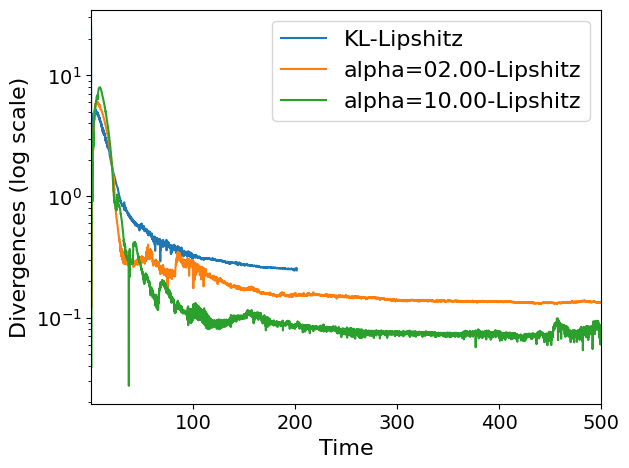}
        \caption{$(f, \Gamma_1)$-divergences}
    \end{subfigure}
    \begin{subfigure}{.59\textwidth}
        \centering
        \includegraphics[width=.9\linewidth]{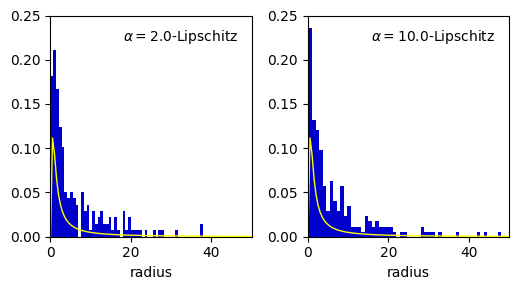}
        \caption{ 
        The radii of transported samples (blue), and the corresponding radial distribution function  (yellow).}
    \end{subfigure}  
\caption{\textbf{(Gaussian to Student-t with $\nu=0.5$ in 2D)} 
We consider 200 initial samples from $N((10,10), 0.5^2I)$, transported towards  200 target samples from $Student-t(\nu)$ with $\nu=0.5$ using $(f, \Gamma_1)$-GPA's for $f=f_\text{KL}$ and $f=f_\alpha$ with $\alpha=2, 10$. {\bf (a)} $(f, \Gamma_1)$-divergences are computed by the corresponding estimator in \cref{eq:GPA:sec3}. 
$(f_\text{KL}, \Gamma_1)$-GPA collapses at around $t=202$ as the function optimization step with $f_\text{KL}$ is numerically unstable on heavy-tailed data while $(f_\alpha, \Gamma_1)$-GPA with $\alpha=2, 10$ propagate particles stably during the entire simulation window. See \Cref{fig:2D student-t_additional} for details. However, GPA still appears to take a long time to transport particles deep into the  heavy tails due to the speed restriction of the Lipschitz regularization. Stability in  performance that lacks in accuracy is manifested 
in the relatively large size of the $\alpha$-divergences. 
{\bf (b)} We observed that $(f_\alpha, \Gamma_1)$-GPA with $\alpha=10$ transports particles further and deeper into   the tails than $(f_\alpha, \Gamma_1)$-GPA with $\alpha=2$.
}
\label{fig:2D student-t}
\end{figure}
\begin{table}[H]
    \centering
    \begin{tabular}{|c|c|c||c|c|}
        \hline
        Case &  GPA source $P_0$ & GPA target $Q$  & $D_\text{KL}^{\Gamma_1}$ & $D_\alpha^{\Gamma_1}$ with $\alpha=2$ \\ \hline
        1 &  $GGM(0.5)$ & $\mathcal{N}((10,10), 0.5^2 I)$ & $O(10^{-6})$ & $O(10^{-6})$ \\ \hline
        2 & $Student-t(3)$ & $\mathcal{N}((10,10), 0.5^2 I)$ & $O(10^{-4})$ & $O(10^{-4})$ \\ \hline
        3 &  $Student-t(1.5)$ & $\mathcal{N}((10,10), 0.5^2 I)$  & \textbf{diverged at $t=0$} & $O(10^{0})$  \\ \hline
        4 &  $Student-t(0.5)$ & $\mathcal{N}((10,10), 0.5^2 I)$ & \textbf{diverged at $t=0$} & $O(10^{7})$ \\ \hline \hline
        5 &  $\mathcal{N}((10,10), 0.5^2 I)$ & $GGM(0.5)$ &   $O(10^{-6})$ & $O(10^{-3})$\\ \hline
        6 &  $\mathcal{N}((10,10), 0.5^2 I)$ & $Student-t(3)$ & $O(10^{-6})$ & $O(10^{-4})$ \\ \hline
        7 & $\mathcal{N}((10,10), 0.5^2 I)$ & $Student-t(1.5)$ & $O(10^{-3})$ & $O(10^{-3})$  \\ \hline
        8 & $\mathcal{N}((10,10), 0.5^2 I)$ & $Student-t(0.5)$ & \textbf{diverged at $t=202$} & $O(10^{-1})$\\ \hline   
    \end{tabular}
\caption{Transportation of heavy-tails to Gaussian (cases 1-4) and Gaussian to heavy-tails (cases 5-8) by $(f, \Gamma_1)$-GPA with $f_\text{KL}$ and $f_\alpha$ with $\alpha=2$. When the algorithm collapses, the corresponding time is reported. In other cases,
the converged $D_f^{\Gamma_1}(P_T\| Q)$'s are reported. }
    \label{tab:heavy_tail_result_summary}
\end{table}
Next, we explore the performance of GPA for several   distributions with varying degrees of heavy-tailed structure. 
Initial distributions $P_0$ are chosen as heavy-tailed distributions in cases 1-4 in \cref{tab:heavy_tail_result_summary}, whereas target distribution $Q$ are chosen as heavy-tailed distributions in cases 5-8. We chose Generalized Gaussian distribution (Stretched exponential distribution, $GMM(\beta) \propto \exp(-|x|^\beta)$) with $\beta=0.5$ as a heavy-tailed distribution because it fails to be subexponential. But it
has finite moments of all orders. On the other hand,
Student-t distributions with degree of freedom $\nu$ ($Student-t(\nu)$) have  polynomial tails. Among them, $Student-t(3.0)$ has a finite second moment, $Student-t(1.5)$ has an infinite second moment but has a finite first moment, and $Student-t(0.5)$ has an infinite second moment but its first moment is undefined. 
In all cases in \cref{tab:heavy_tail_result_summary} we use  the Gaussian distribution $N((10,10), I)$ as either source or target.
\Cref{tab:heavy_tail_result_summary} displays the summary of the transportation of particles for different cases. 
Overall, with the exception of especially heavy-tailed distributions in cases 3 \& 4 (both with infinite second moments  and thus very heavy tails), KL and/or $\alpha$-divergences work reasonably well.
We also note that  $\alpha$-divergences in GANs for images can  provide superior performance to KL and related divergences, even in the abscence of heavy tails \cite{f-GAN, LS_GAN, birrell2022f, birrell2022structure}.

\section{Learning from scarce data}
\label{sec:example:scarceMNIST}

In this section, we empirically demonstrate that GPA can be an effective generative model when only  scarce target data is available. 
We analyze three types of problems: GPA for generating images in a high-dimensional space given scarce target data, GPA for data augmentation, and GPA for approximating a 
multi-scale distribution represented by scarce data. Experiments for the first two applications are conducted following the strategies outlined in \Cref{sec:generalization-overfit} to uphold the generalization properties of GPA.

\paragraph{GPA for image generation  given scarce target data}
Here we consider the example of MNIST image generation using GPA, given a target data set that is relatively sparse compared to the corresponding spatial dimensionality. Recall the entire MNIST data set has $60,000$ images. We demonstrate an example of 
generating images for MNIST in $\mathbb{R}^{784}$ from 200 target samples
in \cref{fig:mnist_different_method}. 
We showcase results from our first two strategies in \Cref{sec:generalization-overfit} to ensure the generalization property of GPA: (i) the imbalanced sample sizes $M \gg N$ (\Cref{fig:sub1}) and (ii) the  generated particles that are simulataneously transported with $M$ training particles (\Cref{fig:sub2}). {In addition, we highlight the efficiency of GPA in training time and target sample size by comparing GPA against WGAN \cite{Wasserstein:GAN} and SGM \cite{Song2021ScoreBasedGM} in \Cref{fig:mnist:diff:methods:2}, in a scarce data regime. 
On the other hand, for a demonstration of scalability of GPA in the number of data, we refer to \cref{fig:mnist:scalability}.

\begin{figure}[h]
     \centering
     \begin{subfigure}{.32\linewidth}
   \centering
   \includegraphics[width=1.\textwidth]{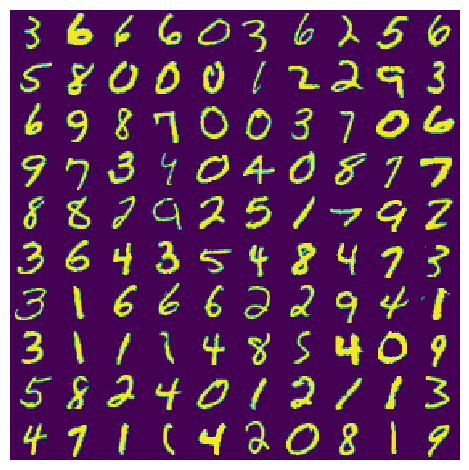}
   \caption{Fixed target samples with sample size $N=200$}
 \end{subfigure}%
 \begin{subfigure}{.32\linewidth}
   \centering
   \includegraphics[width=1.\textwidth]{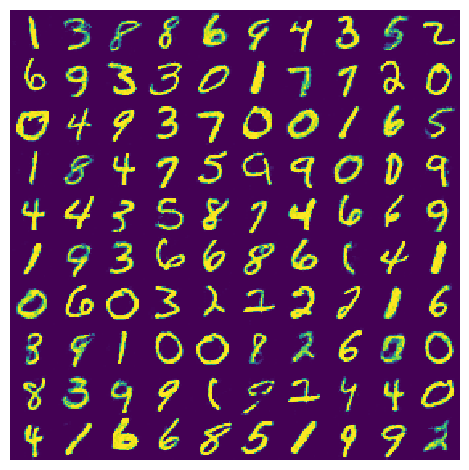}
   \caption{$M=600$ transported particles from $(f_\text{KL},\Gamma_5)$-GPA}
   \label{fig:sub1}
 \end{subfigure}%
 \begin{subfigure}{.32\linewidth}
   \centering
   \includegraphics[width=1.\textwidth]{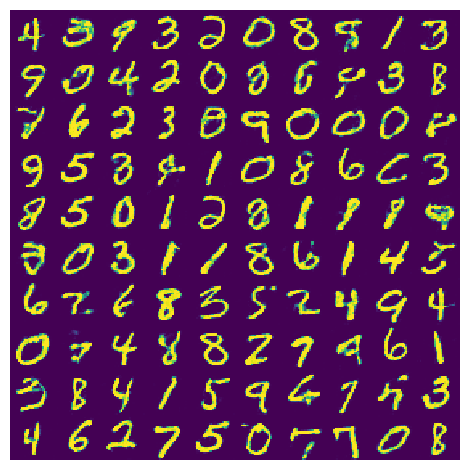}
   \caption{600 generated particles that are simultaneously transported from $(f_\text{KL},\Gamma_5)$-GPA}
   \label{fig:sub2}
 \end{subfigure}%
 \caption{\textbf{(MNIST) GPA for image generation given scarce target data. }
 {\bf (a)} A subset of the $N=200$ target samples. Results in (b-c) are generated by $(f_\text{KL}, \Gamma_5)$-GPA based on the first two strategies in \Cref{sec:generalization-overfit}.  
 We report GPA results with $L=5$, which was empirically found to generate samples stably and in a reasonable amount of time.
 {\bf (b)}  $M=600$ initial particles from $Unif([0,1]^{784})$ were transported toward the target in the setting of $M \gg N$, which promotes sample diversity. See \Cref{fig:mnist:si} for details. {\bf (c)} 
 A new set of 600 initial particles from $Unif([0,1]^{784})$ were transported through the previously learned vector fields. These transported samples are referred to as generated particles, as explained in \Cref{sec:generalization-overfit}. 
 Training time: 5000 time steps ($T=2500$) or 48 minutes in the setting \ref{subsec:appendix:nn:architecture:comp:resources}.} 
\label{fig:mnist_different_method}
\end{figure}

 \begin{figure}[h]
    \centering
    
    \begin{subfigure}{\textwidth}
        \centering
        \includegraphics[width=.9\textwidth]{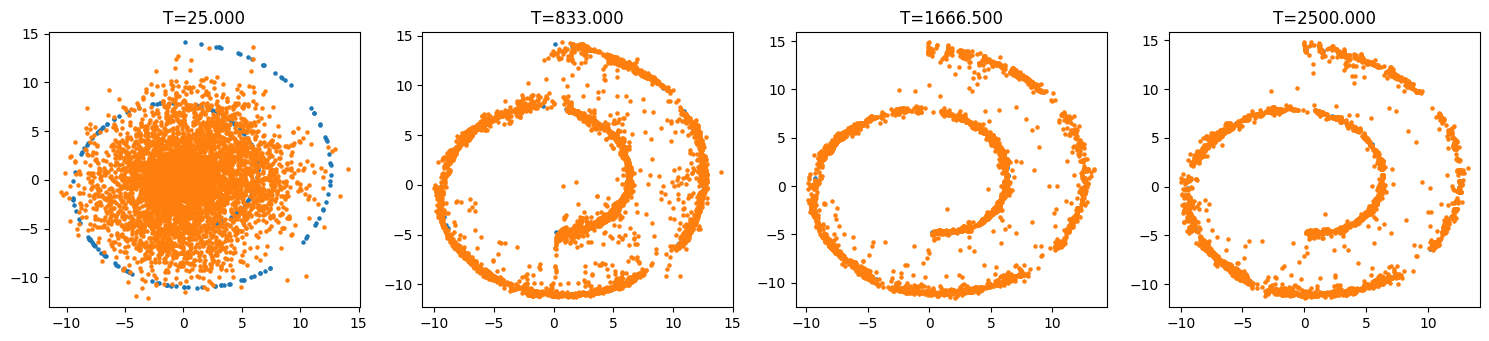}
    \caption{Trajectories of $(f_\text{KL}, \Gamma_1)$-GPA. 
    {\bf (blue)} $N=200$ target data, {\bf (orange)} $M=5000$ transported particles}
    \label{subfig:swissroll:gpa}
    \end{subfigure}
    \begin{subfigure}{0.33\textwidth}
        \centering
        \includegraphics[width=\textwidth]{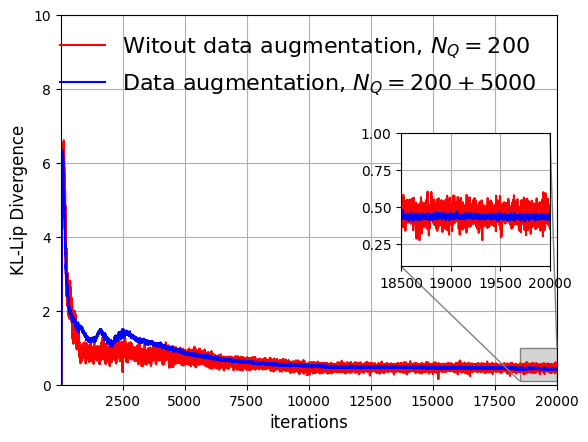}
    \caption{$(f_\text{KL}, \Gamma_1)$-divergences}
    \end{subfigure}
    \begin{subfigure}{0.65\textwidth}
        \centering
        \includegraphics[width=\textwidth]{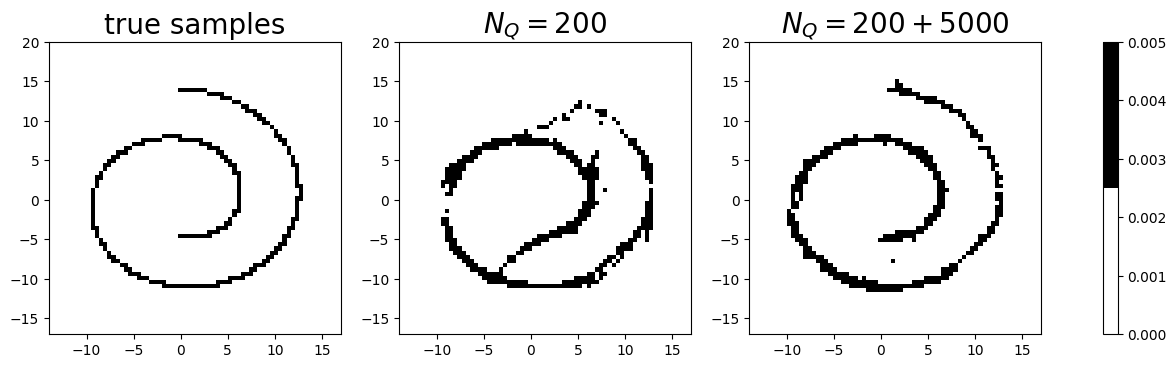}
    \caption{ Data manifolds from true distribution \textbf{(left)}, $(f_\text{KL}, \Gamma_1)$-GAN trained without data augmentation \textbf{(center)} and the same GAN trained with data augmentation  \textbf{(right)}.}
    \end{subfigure}

    \caption{{\bf(Swiss roll) 
    Data augmentation using  GPA.} {\bf (a)} Given  $N=200$ samples from the Swiss roll uniform distribution $Q$,  $M=5000$ additional samples are generated by  transporting initial samples from  $P_0=\mathcal{N}(0, 3^2I)$ using $(f_\text{KL}, \Gamma_1)$-GPA.  Imbalanced sample sizes $M\gg N$ strategy in \Cref{sec:generalization-overfit} is used to ensure  sample diversity. Particles at  $T=2500$ with $D_{f_\text{KL}}^{\Gamma_1}(P_T\| Q) \leq 1.07 * 10^{-4}$ are used as the augmented data. {\bf (b)} When $(f_\text{KL}, \Gamma_1)$-GAN is trained from 200 original samples {\bf (red)}, the loss (divergence) oscillates, see inset in (b). To improve the GAN, we train it with 200 original + 5000 augmented samples. By GPA-data augmentation, the GAN loss decreases stably, see inset in (b).  {\bf (c)} 
    GPA-augmented data significantly enhanced the learning of the manifold when using a GAN on the 5200 samples. 
   }
    \label{fig:swissroll}
\end{figure}

\paragraph{GPA for data augmentation} 
Here, we further verify the capabilities of GPA to learn from scarce target data  in low- and high-dimensional examples such as \Cref{fig:swissroll,fig:mnist:augmentation}. Specifically, GPA can serve as a data augmentation tool for GANs or other generative models, including variational autoencoders \cite{VAE}, autoencoders, and conditional generative models. These models often require a substantial amount of target data in order to enable effective learning of generators.
GPA provides augmented data needed for the proper training of the generative model with both sample diversity and quality, as depicted in \Cref{fig:swissroll} and \Cref{fig:mnist:augmentation}.
An additional advantage of GPA augmentation is that proximity between the augmented data and the original data can be
monitored and controlled by the GPA termination time $T$.
 Indeed, the $(f, \Gamma_L)$-divergence, one of the estimators of GPA in \Cref{algo:gpa}, ensures that the divergence between these datasets remains below the tolerance error $\epsilon_{TOL}$:
\begin{equation}
\label{eq:gpa:augmentation:error}
    D_{f_\text{KL}}^{\Gamma_1}(P_T\| Q) \le \epsilon_{TOL}.
\end{equation}
Other data augmentation techniques, such as small noise injection or transformations, do not inherently ensure the proximity to the target distribution, as  captured in \cref{eq:gpa:augmentation:error}.
 Here we present two examples for this purpose. First, we use a Swiss roll example in \Cref{fig:swissroll} to illustrate the procedure and features of GPA augmentation. 
Furthermore, in  \Cref{fig:mnist:augmentation}, 
we showcase a high-dimensional and consequently more intriguing example of data augmentation for the MNIST dataset. 
 This illustration demonstrates that a WGAN trained with GPA augmented data performs similarly to one trained with original, real  data of the same size.
In conclusion, we demonstrated how to employ GPA for data augmentation as another strategy  for acquiring the generalization properties discussed in  \Cref{sec:generalization-overfit}.

\paragraph{GPA for 
multi-scale distribution}
We consider  a target distribution with  a multi-scale (fractal) structure such as a Sierpinski carpet of level 4.  Namely, this uniform distribution is  constructed  from a fractal set  by keeping  the 4 largest scales and truncating all finer scales.  We refer to  \cref{subfig:sierpinski_target} where we consider  4096 target particles in $[0, 10] \times [0, 10]$. 
Each target particle is random-sampled only once in each dark pixel with size of $[0, 10/3^4]\times [0, 10/3^4]$.
We transport 4096 initial samples from $N(0, 3^2I)$ using $(f_\text{KL}, \Gamma_1)$-GPA. 
\Cref{subfig:sierpinski_gpa_ke,subfig:sierpinski_gpa_output} indicate that $(f_\text{KL}, \Gamma_1)$-GPA approximates the target distribution and stops in a reasonable time $T=1000$ with time steps $n=5000$. 
We also refer to  the related 3D result   in \cref{fig:3dsierpinskicarpet}, where   particles in 3D find a multi-scale structure in the 2D plane. 
On the other hand, training the generator for a multi-scale distribution with the given dataset size posed a significant challenge for both  Wasserstein GAN \cite{Wasserstein:GAN}, $(f_\text{KL}, \Gamma_1)$-GAN \cite{birrell2022f} and score-based generative models (SGM) \cite{Song2021ScoreBasedGM}, as evident in \cref{subfig:sierpinski_wgan2,subfig:sierpinski_fgammagan,subfig:sierpinski_sgm}. 

\begin{figure}[h]
    \centering
    \begin{subfigure}{0.32\textwidth}
   \centering
   \includegraphics[width=0.7\linewidth]{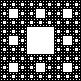}
   \caption{Target distribution }
   \label{subfig:sierpinski_target}
 \end{subfigure}
 \begin{subfigure}{.32\textwidth}
   \centering
   \includegraphics[width=.92\linewidth]{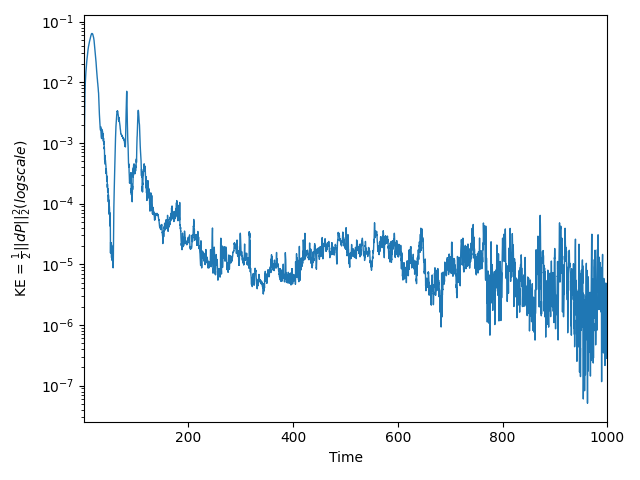}
   \caption{Kinetic energy of particles \cref{eq:Fisher_on_particles} for $(f_\text{KL}, \Gamma_1)$-GPA}
   \label{subfig:sierpinski_gpa_ke}
 \end{subfigure}
  \begin{subfigure}{.32\textwidth}
   \centering
   \includegraphics[width=\linewidth]{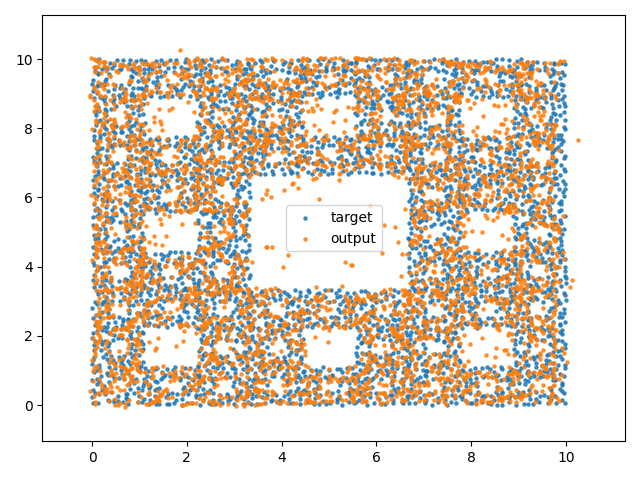}
   \caption{Output of $(f_\text{KL}, \Gamma_1)$-GPA}
   \label{subfig:sierpinski_gpa_output}
 \end{subfigure}
 \begin{subfigure}{.32\textwidth}
   \centering
   \includegraphics[width=\linewidth]{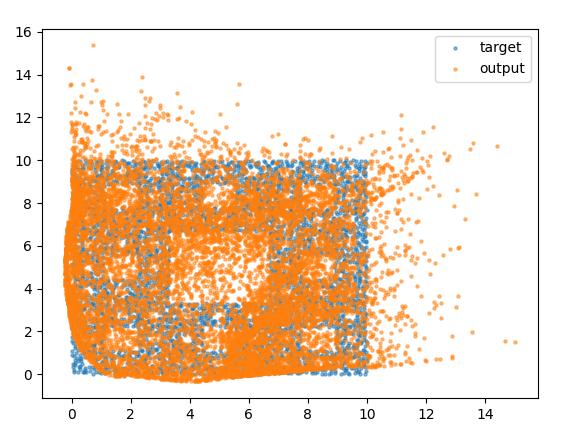}
   \caption{Output of WGAN \cite{Wasserstein:GAN}}
   \label{subfig:sierpinski_wgan2}
 \end{subfigure}
 \begin{subfigure}{0.32\textwidth}
   \centering
   \includegraphics[width=\linewidth]{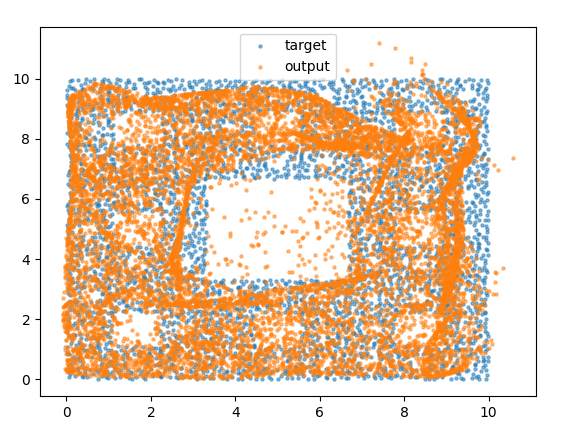}
   \caption{Output of $(f_\text{KL}, \Gamma_1)$-GAN \cite{birrell2022f} }
   \label{subfig:sierpinski_fgammagan}
 \end{subfigure}
 \begin{subfigure}{0.32\textwidth}
   \centering
   \includegraphics[width=\linewidth]{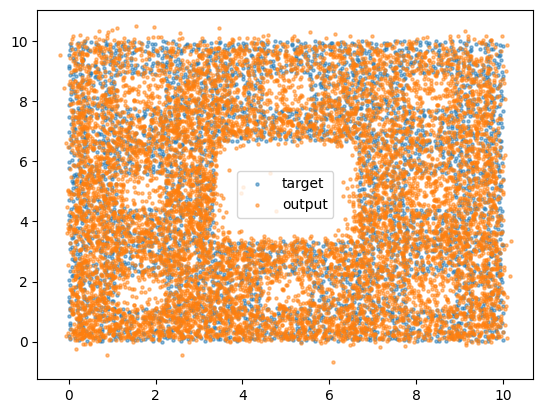}
   \caption{Output of SGM \cite{Song2021ScoreBasedGM} }
   \label{subfig:sierpinski_sgm}
 \end{subfigure}
    \caption{\textbf{(Sierpinski carpet of level 4) GPA for  multi-scale distributions.} 
    GPA demonstrates superior performance over two widely employed generative models in  approximating multi-scale distributions. 
    \textbf{(a)} The problem is to approximate a target distribution with four different scales using 4096  samples.
    \textbf{(b - c)} The $(f_\text{KL}, \Gamma_1)$-GPA successfully transports 4096 Gaussian samples to capture the three largest scales of the target distribution. 
    \textbf{(d - e)} 
   GANs exhibit notably inferior performance compared to GPA, even when sharing the same discriminator structure and loss function, as evidenced in Figure (e). See also \ref{sec:appendix:gpa_vs_gan}. 
 \textbf{(f)} SGM is unable to capture finer scales, even with prolonged training.}
    \label{fig: Sierpinski carpet}
\end{figure}

\section{Latent-space GPA for high-dimensional dataset integration}
\label{sec:example:batch:effects}
%




The integration of two or more datasets that essentially contain the same information, yet  whose statistical properties are different due to e.g., distributional shifts is crucial for the successful training and deployment of statistical and machine learning models \cite{Kirk2012-integrate_datasets,Hendler2014-Kirk2012-integrate_datasets,Samuelsen2019}. Taking bioinformatics as an example, datasets, even when they study the same disease, have been created from different labs around the globe resulting in statistical differences which are also known as batch effects \cite{TranHoaThiNhu2020Abob}. Furthermore, those datasets often have low sample size due to budget constraints or limited availability of patients (e.g., rare diseases).
GPA offers an elegant solution for dataset integration by transporting samples from one dataset to another. Unlike the standard generation process, where the source distribution typically needs to be simple and explicit (e.g., isotropic Gaussian), GPA imposes no assumptions on the source and target distributions. It can also produce stable and accurate results even with very small sample sizes, as demonstrated in \Cref{sec:example:scarceMNIST}.
%
However, applying GPA becomes challenging when the dimensionality of the data rests in the order of tens of thousands. Therefore, we first substantially reduce the dimensionality of the data before employing GPA. After the dimensionality reduction, we apply GPA in the latent space and, when necessary, reconstruct the transported data back to its original high-dimensional space.
%
\emph{This three-step approach efficiently transports samples from the source dataset to the target dataset.} Additionally, it is worth noting that the error resulting from the projection to a lower dimensional latent space is handled via  \Cref{thm:reconstructed_variable_converges_main}. This theorem states that when the target distribution is supported on a lower dimensional manifold, it is theoretically guaranteed through the new data processing inequality that the error in the original space can be bounded by the error occurred in the latent space.

\paragraph{Gene expression datasets}
We consider the integration of two gene expression datasets which are publicly available at \url{https://www.ncbi.nlm.nih.gov/geo/} with accession codes GSE76275 and GSE26639. These datasets have been measured using the GLP570 platform which creates samples with  $d=54,675$ dimensions. Each dataset consists of a low number of data while each individual sample corresponds to the gene expression levels of a patient. Moreover, each sample is labeled by a clinical indicator which informs if the patient was positive or negative to ER (estrogen receptor), see \cref{tab:samplesize}. 
\begin{table}[h]
    \centering
    \begin{tabular}{c|c|c||c}
         & Positive & Negative & Total \\ \hline
        GSE26639 (source) & 138 & 88 & 226 \\ \hline
        GSE76275 (target) & 49 & 216 & 265 \\ 
    \end{tabular}
    \caption{Sample sizes of the studied gene expression  datasets.}
    \label{tab:samplesize}
\end{table}
 The dataset with accession code GSE26639 was selected as the source dataset, while  GSE76275 was chosen as the target. In this example, we chose GSE76275 as the target due to its more distinguishable geometric structure compared to the source, as illustrated in \Cref{fig:source_target_gene_expression}.
 This choice is aimed at showcasing the transportation capabilities of GPA. However, in reality, the decision of selecting the source and target datasets depends on the user and the application context.
Despite measuring the same quantities, a direct concatenation of the two datasets will result in erroneous statistics as is evident in  \cref{fig:source_target_gene_expression} where a 2D visualization reveals that the two datasets are completely separated.

\paragraph{Dimensionality reduction using PCA} 
Applying GPA, along with most machine learning models that do not utilize transfer learning, in the original high-dimensional space is especially challenging when dealing with a low sample size regime.
Hence, we first perform dimensionality reduction constructing a latent space and subsequently perform GPA within the latent space. 
Specifically, we use invertible dimensionality reduction methods by deploying autoencoders suitable for the data. An autoencoder comprises of two functions: the encoder, denoted as $\mathcal{E}(\cdot)$, compresses information from a high-dimensional space to a lower-dimensional latent space, while the decoder, represented as $\mathcal{D}(\cdot)$, decompresses latent features back to the original space. 
Given that training a nonlinear autoencoder based on neural networks requires tens of thousands of samples, we choose PCA as a linear alternative, 
\cite{Bishop:book, pca_review,hastie01statisticallearning}.
Using PCA, we derive a $d'$-dimensional linear basis $\{\mathbf{v}_i\}_{i=1}^{d'}$
 from the entire set of samples in both the source and the target datasets.
Then each sample $\mathbf{x}$ is projected to a $d'$-dimensional space, defining the encoder as the corresponding projection: $\mathbf{z}= \mathcal{E}(\mathbf{x})= \text{Proj}_{\mathbf{v}_{1:d'}}(\mathbf{x})$.
 Subsequently,  the 
GPA \Cref{algo:gpa} will be applied on the latent samples $\mathbf{z}$. 
The decoder $\mathbf{x}=\mathcal{D}(\mathbf{z})$ is also defined by PCA using a reconstruction on the entire $d$-dimensional  space, e.g.  
\cite[Ch 12.1.2]{Bishop:book}. The decoder  is 1-Lipschitz continuous   
since $\|\mathcal{D}(\mathbf{z})-\mathcal{D}(\mathbf{z'})\|^2 = \|\sum_{i=1}^{d'}(z_i - z'_i) \mathbf{v}_i\|^2 = \sum_{i=1}^{d'} |z_i- z_i'|^2 \|\mathbf{v}_i\|^2    =   \|\mathbf{z}- \mathbf{z'}\|^2$. Here we used that    $\mathbf{v}_i$'s are orthonormal and that decoders $\mathcal{D}(\mathbf{z}), \mathcal{D}(\mathbf{z'})$ only differentiate  on the $d'$-dimensional  space in PCA \cite[Ch 12.1.2]{Bishop:book}.
Here we chose $d'=50$ to balance 
computational cost of \Cref{algo:gpa}
and 
error between reconstructed and original datasets, aiming for a practically applicable  approximation of an ideal encoder/decoder, see \Cref{fig:standardized_all_pca}. In this context, \Cref{thm:reconstructed_variable_converges_main}  guarantees that the projection error  remains controlled under encoding/decoding assuring that the performance of the transportation in the original space is dictated by the performance of the GPA in the latent space. 

\paragraph{Results on dataset integration} 
We integrate gene expression datasets by applying the latent-space GPA, transporting samples from the positive-labeled source distribution to the corresponding positive-labeled target distribution and similarly for the negative-labeled data. The respective transportation maps $\mathcal{T}^{n,+}$ and $\mathcal{T}^{n,-}$ are composed of $(f_\text{KL},\Gamma_1)$-GPA transport maps as defined in \cref{eq:scheme:1:sec3}, executed for $n=5000$ time steps (and $\Delta t=0.2$). Each of these separate transportation maps utilizes its own independent discriminator, each with its own unique parameters.
%
%
The visualization of the dataset integration in \cref{fig:gpa_integration} shows that both positive and negative distributions have been efficiently transported via latent-space GPA. 
As a comparison, we present a baseline data transformation for each class, denoted by $\mathcal{F}^+$ and $\mathcal{F}^-$, respectively, which performs mean and standard deviation (std) adjustment. 
As it is evident in \Cref{fig:simple_integration}, the baseline dataset integration only partially relocates the samples from the transformed distribution to the target distribution. The discrepancies are especially pronounced in the negative samples (see inset in \Cref{fig:simple_integration}).

\begin{figure}[h]
    \centering
    \begin{subfigure}{0.325\linewidth}
        \includegraphics[width=\textwidth]{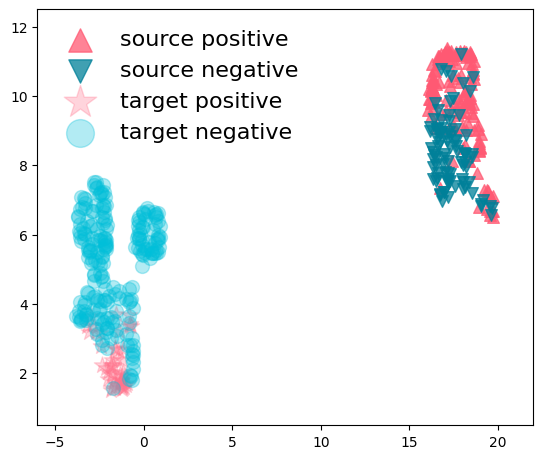}
    \caption{Two gene expression data sets without any transformation.}
    \label{fig:source_target_gene_expression}
    \end{subfigure}
    \begin{subfigure}{0.325\linewidth}
        \includegraphics[width=\textwidth]{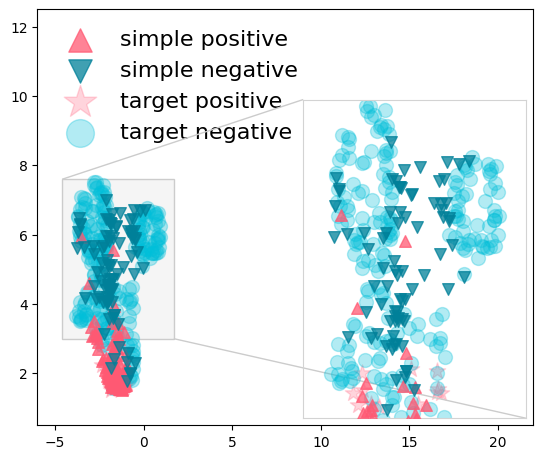}
    \caption{Dataset integration using mean and std adjustment.}
    \label{fig:simple_integration}
    \end{subfigure}
    \begin{subfigure}{0.325\linewidth}
        \includegraphics[width=\textwidth]{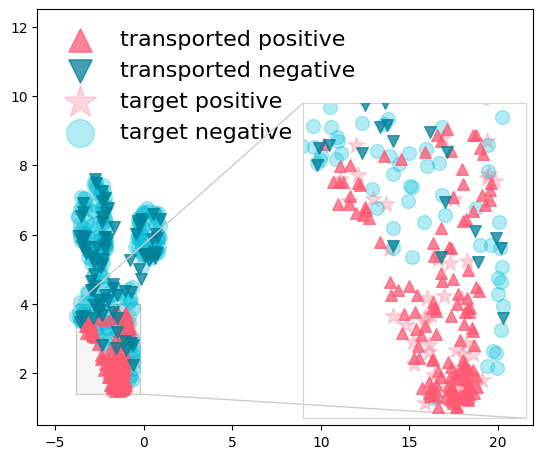}
    \caption{Dataset integration using latent $(f_\text{KL}, \Gamma_1)$-GPA with $d'=50$.}
    \label{fig:gpa_integration}
    \end{subfigure}
    
    \caption{\textbf{Gene expression dataset integration by GPA.} We integrate two high-dimensional gene expression datasets via GPA transportation. \textbf{(a)} A direct concatenation of the two datasets results in incorrect integration as visualized in the 2D plane using UMAP algorithm \cite{McInnes2018}. 
    \textbf{(b)} The baseline approach consists of a mean and std adjustment of each feature in the original space.
    In the inset, we notice that  transformed negative samples do not evenly cover the support of the negative target samples.  \textbf{(c)} The proposed latent GPA data transportation results in transported distributions close to the target ones.  
    }
    \label{fig:integration_gene_expression_data}
\end{figure}

We  quantify the distributional differences between the transported and target distributions via the 2-Wasserstein distance in \Cref{tab:sinkhorn}, which is a metric not used in latent GPA and can also be efficiently computed with the Sinkhorn algorithm. In summary, the 2-Wasserstein distance between datasets in the original space ($d = 54,675$) is reduced by two orders of magnitude (1.4726\% on positive datasets and 2.6104\% on negative datasets), while GPA is twice as effective compared to the baseline mean and standard deviation adjustment transformation (3.9526\% on positive datasets and 4.8718\% on negative datasets). 
Finally, we remark that there are other metrics that can be used to assess the quality of the latent GPA-based dataset integration. For instance, the merged dataset can be tested on subsequent tasks such as phenotype classification or feature selection and evaluate the relative improvement resulting from the integration. We reserve this type of evaluation for future research since it is beyond the scope of this paper. Conducting such an analysis would require dedicated experiments and comparisons specific to the selected subsequent task.

\section*{Appendix} 
Here we  provide \Cref{fig:mnist:augmentation} and \Cref{fig:mnist:diff:methods:2}, discussed earlier in \Cref{sec:example:scarceMNIST}.

\begin{figure}[h]
    \centering
    \begin{subfigure}{.32\textwidth}
      \centering  
      \includegraphics[width=\linewidth]{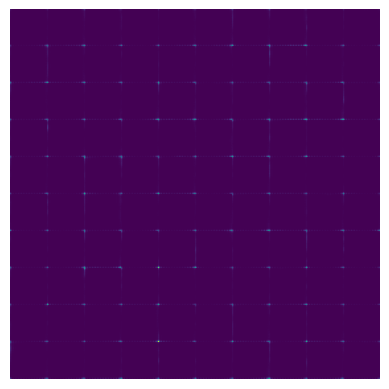}
        \caption{WGAN \cite{Wasserstein:GAN} trained with  200 original data for 3000 training epochs}
    \end{subfigure}
    \begin{subfigure}{.32\textwidth}
        \centering
        \includegraphics[width=\linewidth]{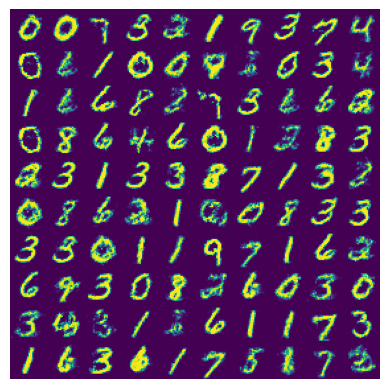}
        \caption{WGAN \cite{Wasserstein:GAN} trained with  original 1400 data for 500 training epochs}
    \end{subfigure}
   \begin{subfigure}{.32\textwidth}
      \centering  
      \includegraphics[width=\linewidth]{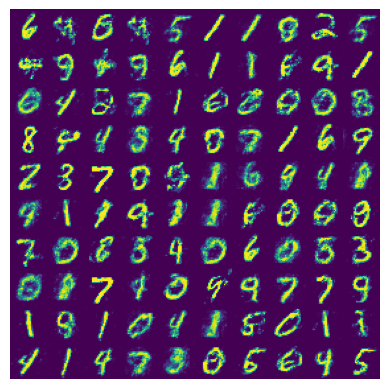}
        \caption{WGAN \cite{Wasserstein:GAN} trained with 200 original data and 1200 GPA-augmented data for 500 training epochs}
    \end{subfigure}
    
    \caption{ 
    {\bf (MNIST) Performance of data augmentation using GPA in a high-dimensional example.} {\bf (a)} WGAN was not able to learn from 200 original samples from the MNIST data base. {\bf (b)} WGAN trained with 1400 original data can now generate  samples but in a moderate quality. {\bf (c)} We obtained 600 GPA-transported data in \Cref{fig:sub1} and 600 generated  data in \Cref{fig:sub2} (see \Cref{sec:generalization-overfit}) from the 200 original   target samples  and used them for augmenting data to train a WGAN with a mixture of 1400 real, transported and generated  samples in total. Such a GAN generated samples    of  similar quality compared to the GAN trained with 1400 original samples in (b).
}
    \label{fig:mnist:augmentation}
\end{figure}

\begin{figure}[p]
    \centering
    
    \begin{subfigure}{\textwidth}
        \centering
        \includegraphics[width=.315\linewidth]{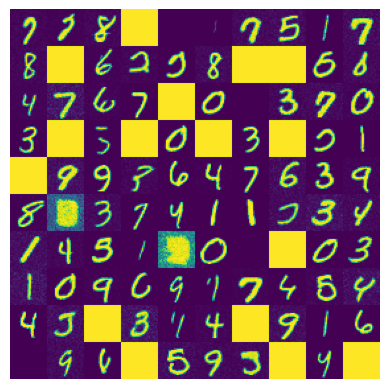}
        \includegraphics[width=.315\linewidth]{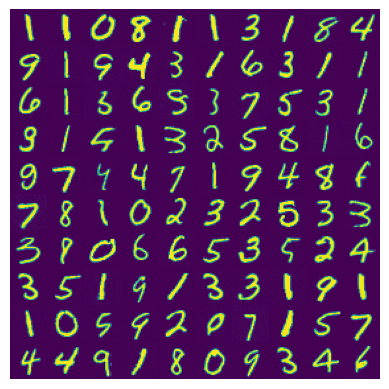}
        \caption{{\bf SGM \cite{Song2021ScoreBasedGM} needs more time.} SGM was able to generate samples from 200 target samples. However, the training was still ongoing for 30 minutes (7,500 training epochs) {\bf (left)}, and eventually overfitted (see related discussion in 
    \Cref{sec:generalization-overfit}.) running for  62 minutes (20,000 epochs) {\bf (right)}.}
    \end{subfigure}
    
    \begin{subfigure}{\textwidth}
        \centering
        \includegraphics[width=.315\linewidth]{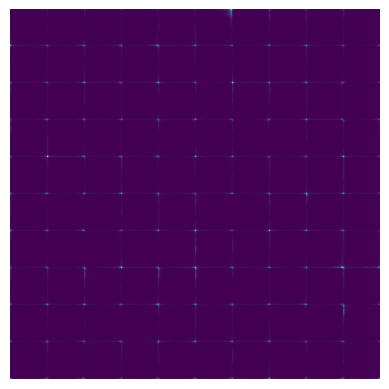}
        \includegraphics[width=.315\linewidth]{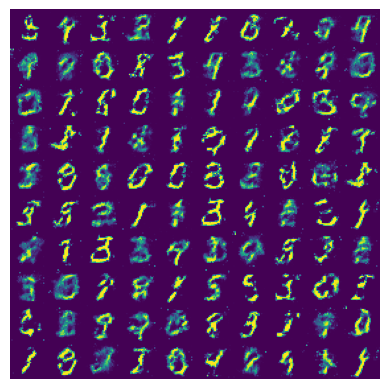}
        \caption{{\bf GAN \cite{Wasserstein:GAN} needs more data.} WGAN trained with 200 target samples did not generate samples while the same GAN trained with 1400 samples could. Its training time is also the slowest among the three models: 350 epochs {\bf (left)} and 
        70 training epochs {\bf (right)} were trained for 30 minutes. }
    \end{subfigure}

    \begin{subfigure}{\textwidth}
        \centering
        \includegraphics[width=.315\linewidth]{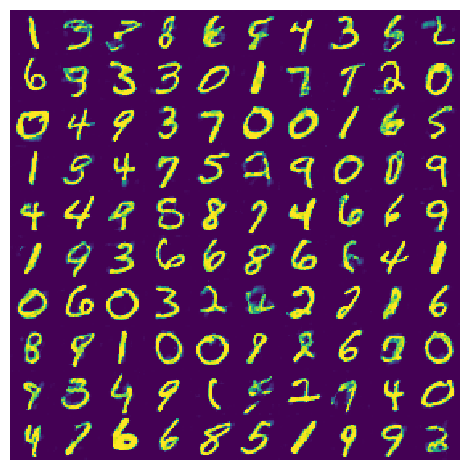}
        \includegraphics[width=.315\linewidth]{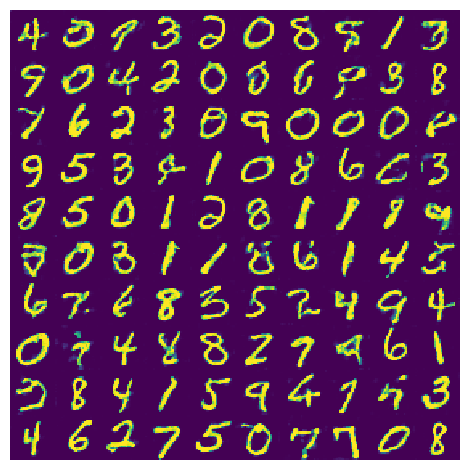}
        \caption{{\bf GPA is ``just right".}  $(f_\text{KL}, \Gamma_5)$-GPA generated samples from $N=200$ target samples in two different ways in \Cref{sec:generalization-overfit}: (i) transporting $M=600 \gg N$ samples {\bf (left)}, and (ii) generating additional 600 samples by transporting through the learned vector fields {\bf (right)}. Both settings in (i) and (ii) were able to produce samples. Lastly, 3160 training epochs were trained for 30 minutes.}
    \end{subfigure}

    \caption{  
    {\bf (MNIST) Comparison of image generation via GPA to SGM and GAN  models.}
    We demonstrate the efficiency of training GPA for image generation in terms of both training time and target sample sizes. The baseline setting restricts training time to 30 minutes and provides a fixed number of 200 target samples to each model. GPA learns to generate samples within this restricted setting, while SGM requires longer training time and WGAN requires more data to be trained.
     }
    \label{fig:mnist:diff:methods:2}
\end{figure}


\bibliographystyle{plain}
\bibliography{references}

\clearpage

\setcounter{page}{1}

\section*{Supplementary Materials} 




\appendix

\section{Background on Lipschitz-regularized divergences}
\label{sec:appendix:fdivergences}

In the paper \cite{dupuis2022formulation}, continuing with  \cite{birrell2022f} a new general class of divergences has been constructed which interpolate 
between $f$-divergences and integral probability metrics and inherit desirable properties from both.  We focus here one specific family which we  view as a Lipschitz regularization of the KL-divergence (or  $f$-divergences) or as an entropic  regularization of the 1-Wasserstein metric.  
We denote by $\mathcal{P}(\mathbb{R}^d)$ the space of all Borel probability measures on $\mathbb{R}^d$ 
by $\mathcal{P}_1(\mathbb{R}^d)= \left\{ P \in \mathcal{P}(\mathbb{R}^d) : \int |x| dP(x) < \infty\right\}$. We denote by $C_b(\mathbb{R}^d)$ the bounded continuous function and by  $\Gamma_L=\{f:\mathbb{R}^d\to \mathbb{R}: |f(x)-f(y)|\le L |x-y|\, \textrm{ for all } x,y\}$ the Lipschitz continous functions with  Lipschitz constant bounded by $L$ (note that $a\Gamma_L = \Gamma_{aL}$).

\paragraph{$f$-divergences} If $f:[0,\infty)\to\mathbb{R}$ is strictly convex and lower-semicontinuous with $f(1)=0$ the $f$-divergence of $P$ with respect to $Q$ is defined by  $D_f(P\|Q)= E_Q[f(\frac{dP}{dQ})]$ if $P\ll Q$ and set to be $+\infty$ otherwise. We have the variational representation (see e.g. \cite{birrell2022f} for a proof)
\begin{equation}\label{eq:VarfDiv}
  D_f(P\|Q)=  \sup_{\phi\in C_b(\mathbb{R}^d)}\left\{E_P[\phi]-\inf_{\nu \in \mathbb{R}}\left\{ \nu + E_Q[f^*(\phi -\nu)]\right\}  \right\}
\end{equation}
where $f^*(s)=\sup_{t\in \mathbb{R}}\left\{ st- f(t)\right\}$ is the Legendre-Fenchel transform of $f$.  We will use  the KL-divergence with  $f_{\mathrm{KL}}(x)=x\log x$ and the $\alpha$-divergence: $f_{\alpha}=\frac{x^{\alpha}-1}{\alpha(\alpha-1)}$ with  Legendre transforms $f^*_{\mathrm{KL}}(y)=e^{y-1}$ and $f^*_{\alpha} \propto y^{\frac{\alpha}{(\alpha-1)}}$. For KL the infimum over $\nu$ can be solved analytically and yields the Donsker-Varadhan with a $\log E_Q[e^\phi]$ term (see \cite{Birrell_IEEE_2022} for more on variational representations). 

\paragraph{Wasserstein metrics} The 1-Wasserstein metrics $W^{\Gamma_1}(P,Q)$ with transport cost $|x-y|$ is an integral probability metrics, see \cite{Wasserstein:GAN}. By keeping the Lipschitz constant as a regularization parameter we set 
\begin{equation}
    W^{\Gamma_L}(P, Q)=\sup_{\phi\in\Gamma_L}\left\{E_{P}[\phi]-E_Q[\phi]  \right\}
\end{equation}
and note that we have $W^{\Gamma_L}(P,Q)=LW^{\Gamma_1}(P,Q)$.

\paragraph{Lipschitz-regularized $f$-divergences} 
The Lipschitz regularized $f$-divergences are defined directly in terms their variational representations,  by replacing the optimization over bounded continuous functions in \eqref{eq:VarfDiv} by Lipschitz continuous functions in $\Gamma_L$.
\begin{equation}\label{def: f, gamma div 2}
D_{f}^{\Gamma_L}(P\|Q):=\sup_{\phi\in\Gamma_L}\left\{E_{P}[\phi]-\inf_{\nu \in \mathbb{R}}\left\{ \nu + E_{Q}[f^*(\phi -\nu)]\right\}\right\}.
\end{equation}
Some of the important properties of Lipschitz regularized $f$-divergences, which summarizes results from 
\cite{dupuis2022formulation,birrell2022f} are given in 
Theorem~\ref{thm:fgdivergence_properties:appendix}. It is assumed there that $f$ is super-linear (called admissible in \cite{birrell2022f}), that is $\lim_{s\to\infty}f(s)/s=+\infty$.  The case of $\alpha$-divergences for $\alpha<1$ is discussed in detail in \cite{birrell2022f}.

\begin{theorem}
\label{thm:fgdivergence_properties:appendix} Assume that $f$ is superlinear and strictly convex.  Then for $P,Q \in \mathcal{P}_1(\mathbb{R}^d)$ we have 

\begin{enumerate}
\item {\bf Divergence:} $D_{f}^{\Gamma_L}(P\|Q)$ is a divergence, i.e. $D_{f}^{\Gamma_L}(P\|Q)\ge 0$ and $D_{f}^{\Gamma_L}(P\|Q)= 0$ if and only if $P=Q$. Moreover the map $(P,Q)\to D_{f}^{\Gamma_L}(P\|Q)$  is convex and lower-semicontinuous. 

\item {\bf Infimal Convolution Formula:} We have 
\begin{equation}\label{def:inf con}
\displaystyle D_{f}^{\Gamma_L}(P\|Q)=\inf_{\gamma\in\mathcal{P}(\Omega)}\left\{W^{\Gamma_L}(P,\gamma) + D_f(\gamma\|Q) \right\}\,.
\end{equation}
In particular we have 
\begin{equation} 
0\leq D_{f}^{\Gamma_L}(P\|Q) \leq\min\left\{ D_f(P\|Q),W^{\Gamma_L}(P,Q)\right\}.
\end{equation}

\item
{\bf Interpolation and limiting behavior of  $D_{f}^{\Gamma_L}(P\|Q)$:}
\begin{equation}\label{property: limiting beh}
\lim_{L\to\infty}D_{f}^{\Gamma_L}(P\|Q)=D_{f}(P\|Q) \mathrm{\quad and\quad } \lim_{L\to0}\frac{1}{L}D_{f}^{\Gamma_L}(P\|Q)=W^{\Gamma_1}(P,Q)\, .
\end{equation}

\item
{\bf Optimizers:} There exists an optimizer $\phi^{L,*} \in \Gamma_L$, whic is unique up to a constant in $\textrm{supp}(P)\cup \textrm{supp}(Q)$.
The optimizer $\gamma^{L,*}$ in the infimal convolution formula exists and is unique and we have $ d\gamma^{L,*}\propto (f^*)'(\phi^{L,*})dQ$ (see \cite{birrell2022f} for details).  
For example for KL we get $d\gamma^{L,*} \propto e^{\phi^{L, *}}dQ$. 

\end{enumerate}

\end{theorem}

\noindent
For connections with Sinkhorn regularizations \cite{Cuturi:stoch_optim_OT:2016} we refer to \cite{birrell2022structure}.
Another useful result established in \cite{birrell2022f} is a new type of data processing inequality. For probability kernel $K(x,dy)$
we denote $K_\#P(dy)= \int K(x,dy)P(dx)$ and $Kf(x)=\int f(y)K(x,dy)$. We have
\begin{theorem}[Data Processing Inequality]\label{thm: data processing inequality} For probability kernel $K(x,dy)$ we have 
\begin{equation}
D_{f}^{\Gamma}( K_\# P\| K_\# Q) \le  D_{f}^{K(\Gamma)}( P\| Q)   
\end{equation}
\end{theorem}
Note that this is a stronger form than the usual data processing inequality since $K(\Gamma)$ maybe (much smaller) than $\Gamma$. This inequality will be used to construct and assess GPA  in latent space in  \Cref{sec:example:batch:effects}, see \cref{thm:reconstructed_variable_converges_main}.
The proof of \Cref{thm: data processing inequality} is similar to \Cref{thm:reconstructed_variable_converges_main}.
 
\section{Proofs for \Cref{sec:L-reg:gradflow}}

\subsection{Proof of \Cref{lemma}}\label{subsec:appendix:lemma}
\begin{lemma}
Let  $f$ be superlinear 
and strictly convex and  $P,Q \in \mathcal{P}_1(\mathbb{R}^d)$. For $y\notin {\rm{supp}}(P) \cup {\rm{supp}}(Q)$, we define 
\begin{equation}
\label{eq:phi_extensions:appendix}
    \phi^{L,*}(y)=\sup_{x\in {\rm{supp}}(Q)}\left\{\phi^{L,*}(x)+ L|x-y|\right\}\, . 
\end{equation}
Then $\phi^{L,*}$ is Lipschitz continuous on $\mathbb{R}^d$  with Lipschitz constant $L$ 
and $\phi^{L,*}=\sup\{h(x): h\in\Gamma_L,\, h(y)=\phi^{L,*}(y), {\textrm{for every }} y\in{\rm{supp}}(Q)\}$.
\end{lemma}
\begin{proof}
The fact that $\phi^{L,*}$ is Lipschitz continuous on $\mathbb{R}^d$ is straightforward by using the triangle inequality. Moreover, since $h\in \Gamma_L$, we have that $h(x)\leq h(y)+L\|x-y\|$. This implies that for $y\in{\rm{supp}}(Q)$ and $x\notin{\rm{supp}}(Q)$, $h(x)\leq\inf_{y\in{\rm{supp}}(Q)}\{ h(y)+L\|x-y\|\}=\inf_{y\in{\rm{supp}}(Q)}\{\phi^{L,*}(y)+L\|x-y\|\}=\phi^{L,*}(x)$. Since $\phi^{L,*}(y)\in\Gamma_L$, this concludes the proof.
\end{proof}

\subsection{Proof of \cref{thm:dissipation}}
\label{subsec:appendix:dissipation}
\begin{theorem}[Lipschitz-regularized dissipation]
Along a trajectory of a smooth solution $\{P_t\}_{t\geq 0}$ of \cref{eq:transport:variational:pde:sec2} with source probability distribution $P_0=P$ we have the following rate of decay identity: 
\begin{equation}\label{dissipation_appendix}
    \frac{d}{dt}D_{f}^{\Gamma_L}(P_t\|Q)=-I_f^{\Gamma_L}(P_t\|Q) \le 0
  \end{equation}
  where we define the Lipschitz-regularized Fisher Information as 
  \begin{equation}\label{eq:fisher_info_appendix}
   I_f^{\Gamma_L}(P_t\|Q)= E_{P_t}\left[|\nabla \phi^{L,*}|^2 \right] \,.
    \end{equation}
Consequently, for any $T\geq 0$, we have
$D_{f}^{\Gamma_L}(P_T\|Q) = D_{f}^{\Gamma_L}(P\|Q)-\int_{0}^{T} I_f^{\Gamma_L}(P_s\|Q)ds\,$. 
\end{theorem}

\begin{proof} We obtain \cref{dissipation_appendix} by  the next  calculation, assuming sufficient smoothness. We use the divergence theorem together with vanishing  boundary conditions, as well as \cref{eq:gradflow:variational:discriminator} and \cref{eq:transport:variational:pde}.
\begin{eqnarray}
\frac{d}{dt}D_{f}^{\Gamma_L}(P_t\|Q)&=&\left\langle\frac{\delta  D_{f}^{\Gamma_L}(P\|Q)}{\delta P}(P),\frac{\partial P_t}{\partial t}\right\rangle=\left\langle \phi_t^{L,*},{\rm div} \left(P_t\nabla\phi_t^{L,*}\right)\right\rangle\nonumber\\
&=&-\int|\nabla\phi_t^{L,*}|^2dP_t=-E_{P_t}\left[|\nabla \phi_t^{L,*}|^2 \right]\, .
\end{eqnarray}

\end{proof}

\section{Computational details}
\label{sec:appendix:experimentalsetting}

\subsection{Neural network architectures and computational resources}
\label{subsec:appendix:nn:architecture:comp:resources}

\paragraph{Neural network architectures} Discriminators $\phi:\mathbb{R}^d \rightarrow \mathbb{R}$'s are implemented using neural networks. We implemented FNN discriminators for general $\mathbb{R}^d$ problems and CNN discriminator especially for 2D image generation problems. For both networks, we impose the Lipschitz constraint on $\phi$ by spectral normalization (SN), where the weight matrix in each layer of the $D$ layers in total has spectral norm $\|W^l\|_2=L^{1/D}$.
See \cref{table:nn architecture} for details. Exact numbers of parameters differ in each example and can be found in the code repository \url{https://github.com/HyeminGu/Lipschitz_regularized_generative_particles_algorithm} v0.2.0. 

\begin{table}[ht]
\begin{minipage}{.49\linewidth}
\medskip
\centering
\begin{tabular}{c}

\hline
FNN Discriminator \\
\hline
  $W^1 \in \mathbb{R}^{d \times \ell_1}$ with SN, $b^1 \in \mathbb{R}^{\ell_1}$ \\
  ReLU\\
  \hline
  $W^2 \in \mathbb{R}^{\ell_1 \times \ell_2}$ with SN, $b^2 \in \mathbb{R}^{\ell_2}$ \\
  ReLU\\
  \hline
  $W^3 \in \mathbb{R}^{\ell_2 \times \ell_3}$ with SN, $b^3 \in \mathbb{R}^{\ell_3}$ \\
  ReLU\\
  \hline
  $W^4 \in \mathbb{R}^{\ell_3 \times 1}$ with SN, $b^4 \in \mathbb{R}$ \\
  Linear\\
\hline
\end{tabular}
\subcaption{General problems with dimension $d$}
\label{table:fnn}
\end{minipage}
\begin{minipage}{.49\linewidth}
\centering
\begin{tabular}{c}

\hline
CNN Discriminator \\
\hline
  $5 \times 5$ Conv SN, $2 \times 2$ stride  ($1 \rightarrow ch_1$) \\
  leaky ReLU\\
  Dropout, rate 0.3 \\
  \hline
  $5 \times 5$ Conv SN, $2 \times 2$ stride ($ch_1 \rightarrow ch_2$)  \\
  leaky ReLU\\
  Dropout, rate 0.3 \\
  \hline
  $5 \times 5$ Conv SN, $2 \times 2$ stride ($ch_2 \rightarrow ch_3$)  \\
  leaky ReLU\\
  Dropout, rate 0.3 \\
  \hline
  Flatten with dimension $\ell_3$ \\
  \hline
  $W^4 \in \mathbb{R}^{\ell_3 \times d}$ with SN, $b^4 \in \mathbb{R}^{d}$ \\
  ReLU\\
  \hline
  $W^5 \in \mathbb{R}^{d \times 1}$ with SN, $b^5 \in \mathbb{R}$ \\
  Linear\\
\hline
\end{tabular}\subcaption{2D image data (MNIST)}
\medskip
\label{table:cnn}
\end{minipage}

\caption{Neural network architectures of the discriminator $\phi: \mathbb{R}^d \rightarrow \mathbb{R}$}
\label{table:nn architecture}
\end{table}

\paragraph{Computational resources} 
MNIST image generation example is computed in the environment: \texttt{tensorflow-gpu=2.7.0} with \texttt{Intel(R) Xeon(R) CPU E5-2620 v3 @ 2.40GHz 4 cores} and \texttt{Tesla M40 24GB}. Other examples are computed in the environment: \texttt{Apple M2 8 cores} and \texttt{Apple M2 24 GB - Metal 3}.

\subsection{Additional features to improve the accuracy}
\label{subsec:settings:improve:accuracy}
\paragraph{Higher order explicit ODE solvers}

Besides the  forward Euler scheme considered in \cref{eq:scheme:1:sec3} and \cref{eq:GPA:sec3}, we can also take advantage of higher order schemes for differential equations such as 
Heun's method 
\begin{align}\label{eq: heun}
\begin{split}
    \tilde{y}_{t+1} & = y_{t} - \Delta t \nabla \phi_t (y_t) \\
    y_{t+1} & = y_{t} - \frac{\Delta t}{2} (\phi_t (y_t) + \phi_{t+1} (\tilde{y}_{t+1}))
\end{split}
\end{align}
and RK4. As we demonstrate in an example in \cref{fig:12D - 2D gaussian mixtures}, they can substantially  improve the accuracy of solution. 
In this example  GPA  learns a 2D Mixture of Gaussians embedded in 12D. We consider 600 particles from the 12D Gaussian ball $P_0 = N(8* \mathbf{1}_{12}, 0.5^2 I_{12})$, which are   transported via GPA to the target distribution.  In this example, forward Euler produces an oscillatory pattern in the orthogonal 10D subspace while \cref{eq: heun} produces a convergent approximation. 

\paragraph{Smooth activation functions} 
 Smoother discriminators $\phi_n^{L,*}$ allow us to take  smaller time step sizes $\Delta t$ in \cref{eq:GPA:sec3} so that the algorithm can slow down and eventually stop, avoiding oscillations around the target.  We build smoother discriminators  by replacing the standard $ReLU$ activation function in NNs by a smoother one, namely  $ReLU_s^\epsilon \in C^3$ with $0 \leq (ReLU_s^\epsilon)'(x)\leq 1$ given by \cref{eq:RELU_s}. This activation function is compatible with spectral normalization technique for imposing Lipschitz continuity to a NN and  is given by
 
\begin{minipage}{.55\linewidth}
    \begin{align}\label{eq:RELU_s}
\begin{split}
    & ReLU_s^\epsilon(x), ~~ \epsilon=2^{-n} \\
    = & \begin{cases}
    0,  & x \leq 0 \\
    \frac{x^2}{4 \epsilon} + \frac{\epsilon}{2 \pi^2} (\cos{(\frac{\pi}{\epsilon}x)-1)}, & 0 < x < 2 \epsilon \\
    x- \epsilon, & x \geq 2 \epsilon.
    \end{cases}
\end{split}
\end{align} 
\end{minipage}
\begin{minipage}{.4\linewidth}
    \centering
    \includegraphics[width=.85\textwidth]{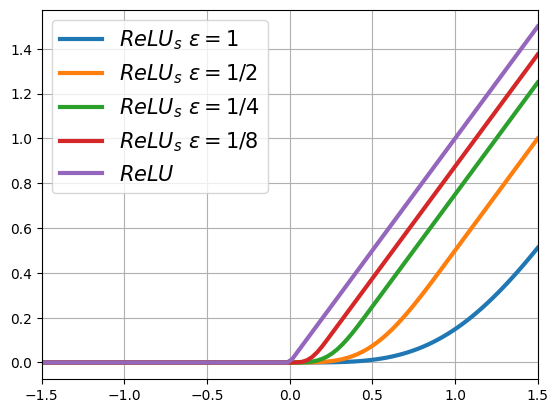}
\end{minipage}

\begin{figure}[h]
     \centering
     \begin{subfigure}{.45\linewidth}
   \centering
   \includegraphics[width=.9\textwidth]{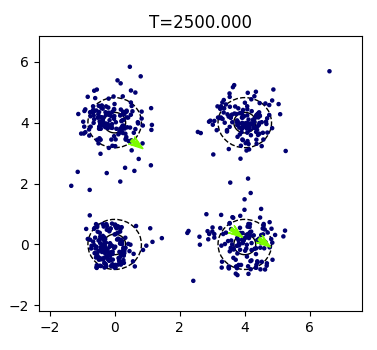}
   \caption{Forward Euler, $\Delta t=0.25$,  Last snapshot in the 2D subspace}
 \end{subfigure}%
 \begin{subfigure}{.5\linewidth}
   \centering
   \includegraphics[width=0.9\textwidth]{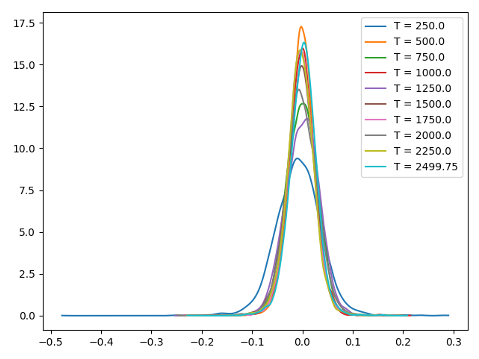}
   \caption{Forward Euler, $\Delta t=0.25$, Evaluation at orthogonal axes}
 \end{subfigure}
  
  \begin{subfigure}{.45\linewidth}
   \centering
   \includegraphics[width=.9\textwidth]{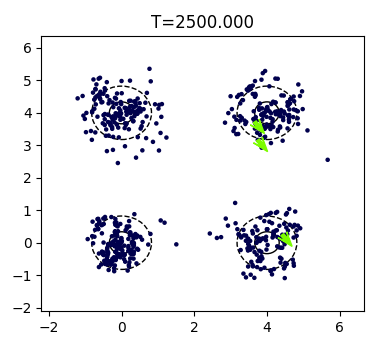}
   \caption{Heun, $\Delta t=0.25$, Last snapshot  in the 2D subspace}
 \end{subfigure}%
 \begin{subfigure}{.5\linewidth}
   \centering
   \includegraphics[width=.9\textwidth]{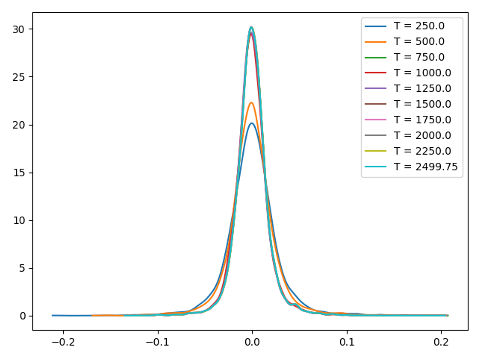}
   \caption{Heun, $\Delta t=0.25$, Evaluation at orthogonal axes}
 \end{subfigure}
 \caption{\textbf{(2D Mixture of Gaussians embedded in 12D) Forward Euler and Heun  for $(f_\text{KL},\Gamma_1)$-GPA}. Both Forward Euler and Heun were able to capture the 4 wells in the 2D subspace but forward Euler shows oscillatory behavior (in time) in the orthogonal subspace while Heun shows convergent in the orthogonal subspace. 
 }
 \label{fig:12D - 2D gaussian mixtures}
 \end{figure} 

The smooth ReLU in \Cref{eq:RELU_s} has a superior compatibility with spectral normalization technique compared to other candidate functions such as Softplus, Exponential Linear Unit (ELU), Scaled Exponential Linear Unit (SELU) and Gaussian Error Linear Unit (GELU) since outputs of hidden layers are inclined to be concentrated near 0 after the weight normalization.  
Therefore, the threshold to discriminate the outputs should be assigned as 0 which can be attained by putting an activation function which passes the origin and has distinguishable gradients on the left $x<0$ and the right $x>0$ near 0. 

\section{GPA vs. GAN}
\label{sec:appendix:gpa_vs_gan}

GPA generates particles by iteratively  solving \cref{eq:GPA:sec3}.  The velocities of the  particles   are computed  by the evaluation of the gradient of   the discriminator  $\phi_n^{L,*}$, and updated at each time step $n$. 
This discriminator evaluation feature is shared with  GANs  \cite{goodfellow2014generative,Wasserstein:GAN,birrell2022f}. However in GANs the particle generation step is different and involves  also  learning a generator $g_{\theta} : \mathbb{R}^{d'} \rightarrow \mathbb{R}^{d}$ parametrized in turn by a second  NN with its own parameters $\theta$. 
For each time step $n$, GANs solve two optimization problems on $\theta$ and $\phi$. For instance, an 
$(f, \Gamma_L)$-based GAN, \cite{birrell2022f},  is the minmax problem
\begin{equation}
\begin{aligned}
\label{eq:gan_generator}
    & \inf_{\theta} \sup_{\phi} H_f[\phi; g_\theta (Z), X]\, , \quad \text{where the objective function is} \\
    & H_f[\phi; g_\theta (Z), X] = \frac{\sum_{i=1}^{M}\phi(g_{\theta} (Z^{(i)}))}{M}- \inf_{\nu \in \mathbb{R}}\left\{ \nu + \frac{\sum_{i=1}^{N}f^*(\phi(X^{(i)})-\nu)}{N}\right\}\, .
  \end{aligned}
\end{equation}
Here $Z^{(i)}$ denote  random data  usually from the standard Gaussian in $\mathbb{R}^{d'}$ and $X^{(i)}$ correspond to the  given training data set. Different GANs \cite{goodfellow2014generative,Wasserstein:GAN,f-GAN,GP-WGAN, LS_GAN} have their own objective functionals  $H_f[\phi; g_\theta (Z), X]$, however $(f, \Gamma_L)$-based GANs provide a common, mathematically unifying framework \cite{birrell2022f}.
Once a GAN is trained, new samples can be reproduced instantly by evaluating the generator $g_{\theta_{\text{final}}}^*$ on random Gaussian samples $Z$. 
 GANs are  discriminator-generator models, while GPA are a discriminator-transport model
where the generator is replaced  by the transport mechanism, and  does not need to be learned, see \cref{algo:gpa}.
Since GPA does not  learn a generator,  instant generation is not allowed as in GAN. But GPA can excel in some tasks that GANs fail, see for example  \Cref{sec:example:scarceMNIST}.

\section{Supplementary experiments}
Here is the list of supplementary experiments/results to support the main text. 
\begin{itemize}
    \item \Cref{fig:2D student-t_additional}: (Gaussian to Student-t with $\nu=0.5$ in 2D) Snapshots and estimators of $(f, \Gamma_1)$-GPA introduced in \Cref{fig:2D student-t}
    \item \Cref{fig:mnist_overfitting}: (MNIST) Sample diversity for GPA obtained by $M \gg N$. See \Cref{sec:generalization-overfit}
    \item \Cref{fig:mnist_rich_nn}: (MNIST) A case study on the impact of complexity for neural network architecture. See \Cref{sec:generalization-overfit}
    \item \Cref{fig:mnist:si}: (MNIST) Sample diversity  of transported data in \Cref{fig:sub1}
    \item \Cref{fig:mnist:scalability}: (MNIST) The impact of increased sample sizes compared to \Cref{fig:mnist_different_method}
    \item \Cref{fig:standardized_all_pca}: (Gene expression data integration with GPA) Dimension reduction with PCA and choice of latent dimensions in \Cref{sec:example:batch:effects}
    \item \Cref{tab:sinkhorn}: (Gene expression data integration with GPA) Quantitative results of data integration in \Cref{fig:integration_gene_expression_data}
\end{itemize}

\begin{figure}[h]
    \centering
    
    \begin{subfigure}{\textwidth}
        \centering
        \includegraphics[width=\linewidth]{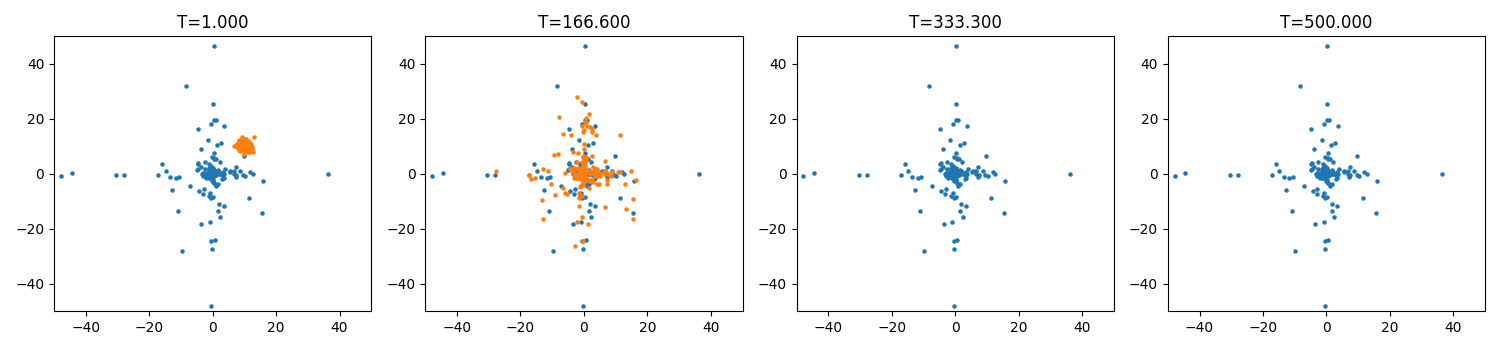}
        \caption{$f_\text{KL}$, $L=1$ collapses when the particle distribution become heavy-tailed}
        \label{subfig:heavytailed_kl}
    \end{subfigure}
    
    \begin{subfigure}{0.35\textwidth}
        \centering
        \includegraphics[width=\linewidth]{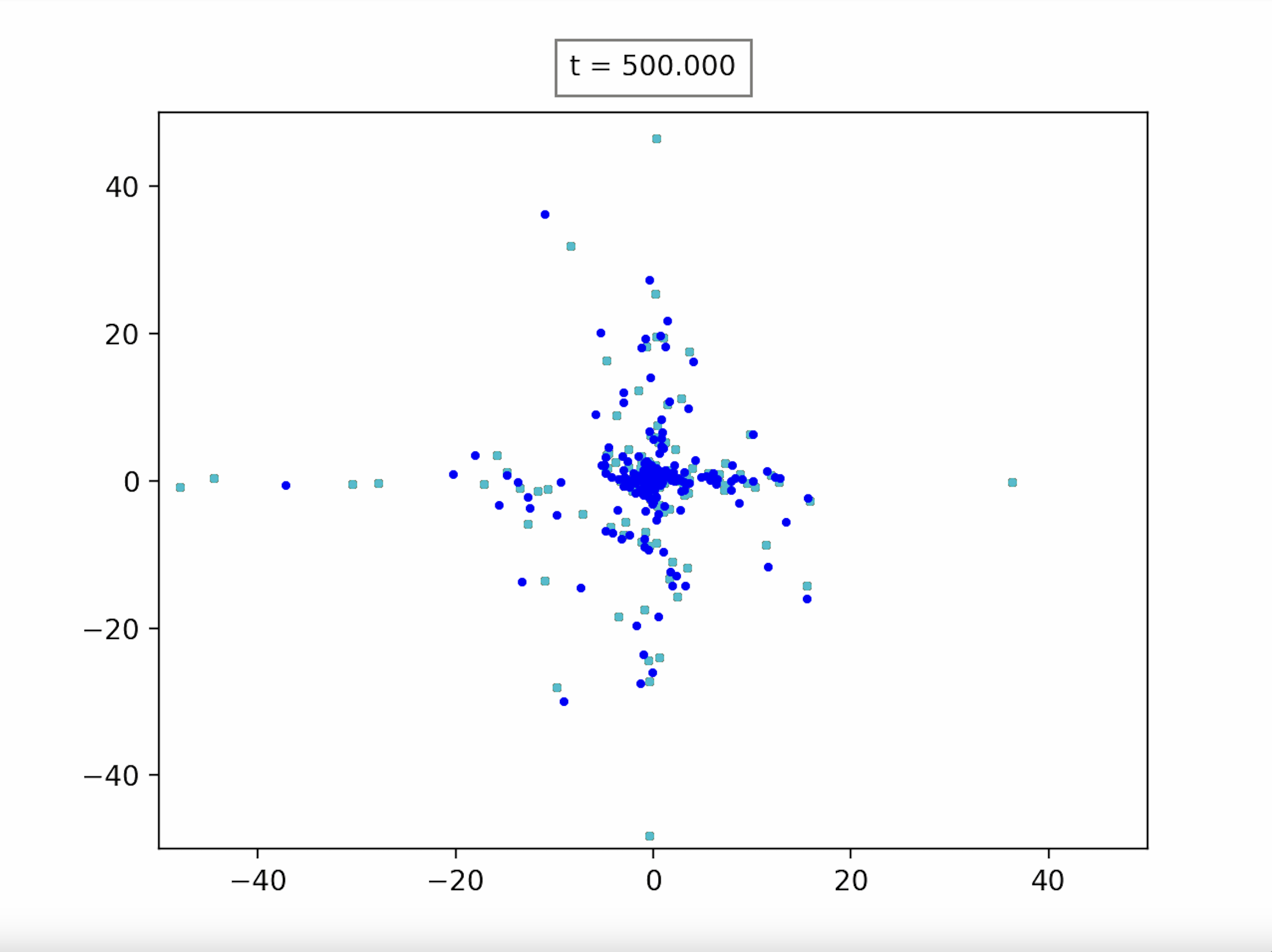}
        \caption{$f_\alpha$ with $\alpha=2$, $L=1$ at $t=500$}
    \end{subfigure}
    \begin{subfigure}{0.35\textwidth}
        \centering
        \includegraphics[width=\linewidth]{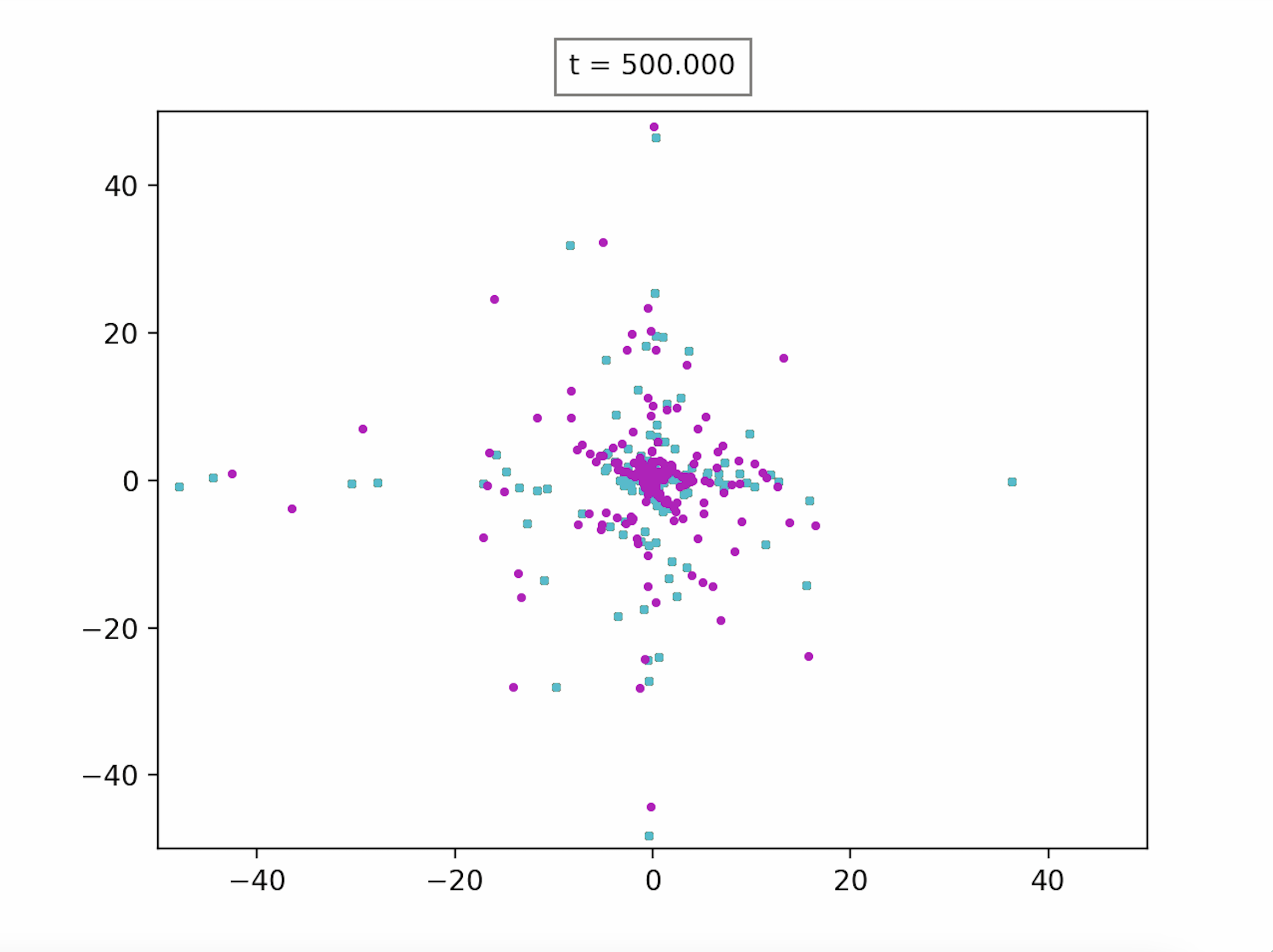}
        \caption{$f_\alpha$ with $\alpha=10$, $L=1$ at $t=500$}
    \end{subfigure}
    \begin{subfigure}{0.285\textwidth}
        \centering
        \includegraphics[width=\linewidth]{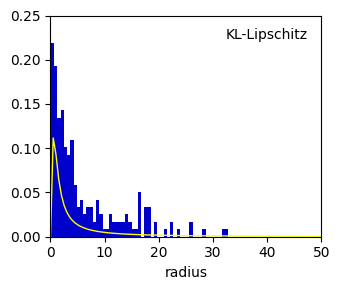}
        \caption{ Distribution of the radii of transported samples for $f_\text{KL}$, $L=1$ before the algorithm collapses}
    \end{subfigure}

 \begin{subfigure}{0.49\textwidth}
        \centering
        \includegraphics[width=\linewidth]{plots_final/different_fdivergences.png}
        \caption{Lipschtz regularized $f$-divergences}
    \end{subfigure}
    \begin{subfigure}{0.49\textwidth}
        \centering
        \includegraphics[width=\linewidth]{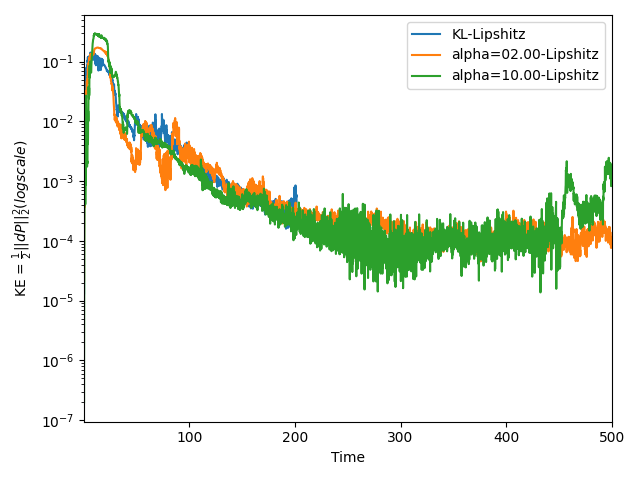}
        \caption{Kinetic energy of particles}
    \end{subfigure}
\caption{\textbf{(Gaussian to Student-t with $\nu=0.5$ in 2D) Snapshots and estimators of $(f, \Gamma_1)$-GPA introduced in \Cref{fig:2D student-t}.} 
{\bf (a)} Snapshots of $(f_\text{KL}, \Gamma_1)$-GPA at time points $t=1.0, 166.6, 333.0, 500.0$. The GPA using a KL divergence collapses at around $t=202$ as the function optimization step with $f_\text{KL}$ is numerically unstable on heavy-tailed data.
{\bf (b - c)} Snapshots of $(f_\alpha, \Gamma_1)$-GPA with $\alpha=2, 10$ at time point $t=500.0$. 
{\bf (d)} The radii of propagated particles from $(f_\text{KL}, \Gamma_1)$-GPA at $t=202$ compared to \Cref{fig:2D student-t} (b).  {\bf (e - f)} After $t > 300$, transportation speeds remain strictly positive but low, indicating that GPA continues to require significant time for transporting particles into the heavy tails.
}
\label{fig:2D student-t_additional}
\end{figure}

\begin{figure}[h]
    \begin{subfigure}{.6\linewidth}
    \includegraphics[width=\linewidth]{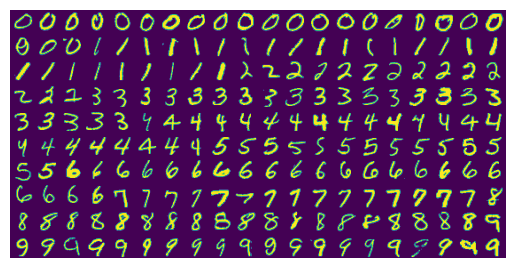}
    \caption{$N=200$ target samples from the true MNIST data set}
    \label{fig:mnist_target}
    \end{subfigure}
    
    \begin{subfigure}{.6\linewidth}
        \centering
    \includegraphics[width=\linewidth]{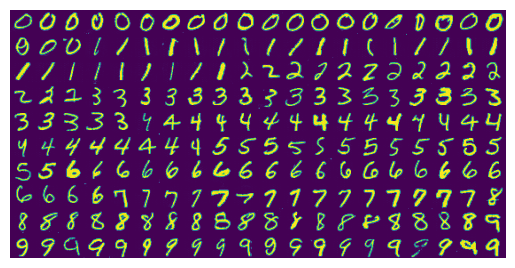}
    \caption{$M =200 (=N)$ transported samples from $(f_\text{KL}, \Gamma_1)$-GPA }
    \label{fig:mnist_M=N}
    \end{subfigure}
    \begin{subfigure}{.39\linewidth}
        \centering
        \includegraphics[width=\linewidth]{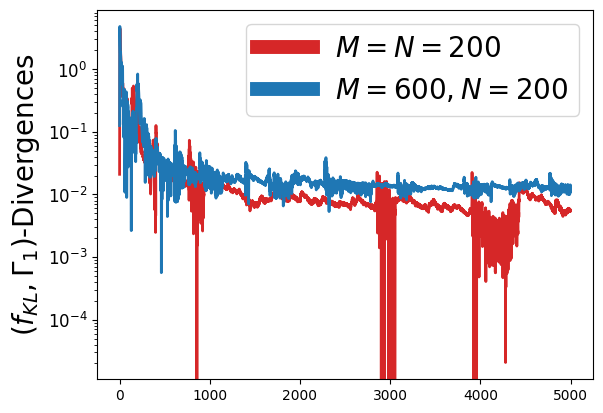}
    \caption{ $(f_\text{KL}, \Gamma_1)$-divergences}
    \label{fig:mnist_divergence}
    \end{subfigure}
    
    \begin{subfigure}{.6\linewidth}
        \centering
        \includegraphics[width=\linewidth]{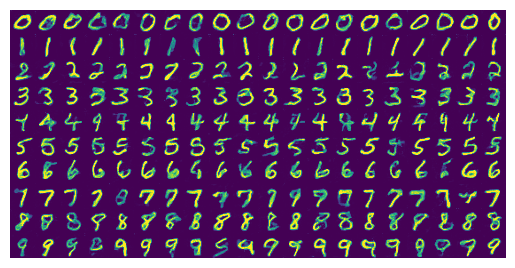}
    \caption{$M =600 (\gg N)$ transported samples from $(f_\text{KL}, \Gamma_1)$-GPA}
    \label{fig:mnist_M>>N}
    \end{subfigure}
   \begin{subfigure}{.39\linewidth}
        \centering
        \includegraphics[width=\linewidth]{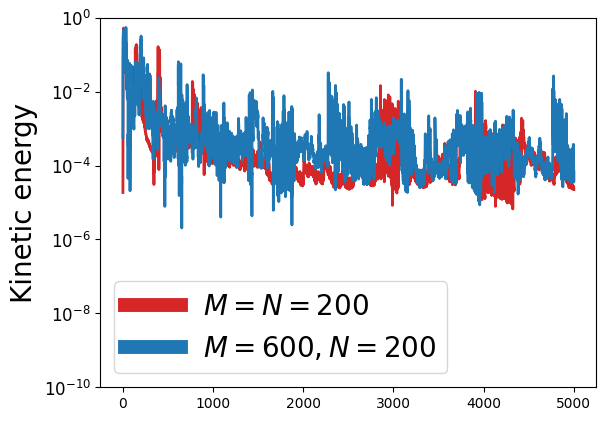}
    \caption{ Kinetic energy of particles}
    \label{fig:mnist_trajectory}
    \end{subfigure}
    \caption{{\textbf{(MNIST) Sample diversity for GPA obtained by $M \gg N$. See \Cref{sec:generalization-overfit}.} We present two  experiments which are conducted in  imbalanced sample sizes $M \gg N$ and equal sample size $M=N$ in \cref{algo:gpa}. {\bf (a)} $N=200$ target samples are fixed and $(f_\text{KL},\Gamma_1)$-GPA transported different numbers ($=M$) of particles toward the target. Generated samples in (b, d) can be compared one-by-one with the target. {\bf (b)} We note that  $(f_\text{KL},\Gamma_1)$-GPA is such an efficient and accurate transportation method  that when $M=N$,  it would typically transport the  source particles almost exactly on the target particles.  {\bf (d)} Generated samples when $M \gg N$ show diversity in shape and shades. {\bf (c, e)} The two estimators ($(f,\Gamma_L)$-divergence and kinetic energy of particles) can be used as diagnostics for the over-fitting behavior.}
    }

    \label{fig:mnist_overfitting}
\end{figure}

\begin{figure}[h]
    
    \begin{subfigure}{.49\linewidth}
        \centering
        \includegraphics[width=.8\linewidth]{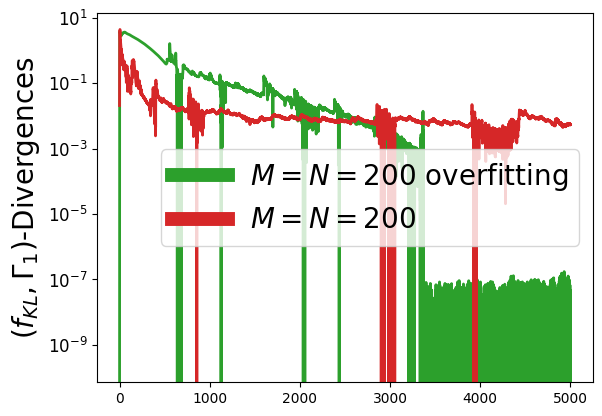}
    \caption{$(f_\text{KL}, \Gamma_1)$-divergences}
    \label{fig:mnist_trajectory_expon}
    \end{subfigure}
   \begin{subfigure}{.49\linewidth}
        \centering
        \includegraphics[width=.8\linewidth]{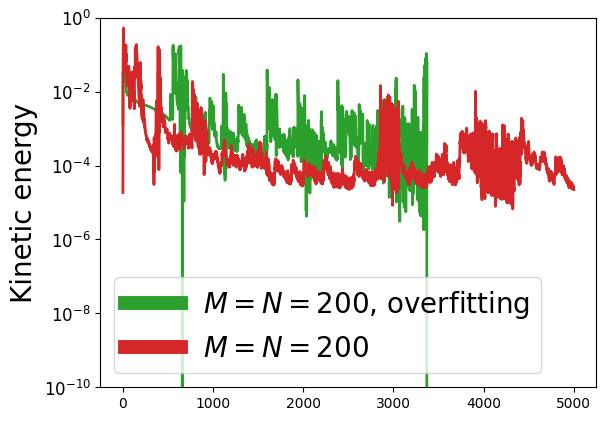}
    \caption{ Kinetic energy of particles}
    \end{subfigure}
    \caption{{\textbf{(MNIST) A case study on the impact of complexity for neural network architecture. See \Cref{sec:generalization-overfit}.} We present additional experiments on generating MNIST digits (as in the $M=N=200$ case in \Cref{fig:mnist_overfitting}.) where now the discriminator is parameterized by a more complex neural network architecture in \Cref{table:cnn} with  $ch_1=128$, $ch_2=256$, $ch_3=512$ numbers of filters on three hidden convolutional layers and is trained with increased learning rate $\delta=0.001$ in \Cref{algo:gpa} compared to our standard discriminator in \Cref{fig:mnist_overfitting}.
    Our standard discriminator takes  $ch_1=128$, $ch_2=128$, $ch_3=128$ and $\delta=0.0005$ respectively. 
    With the more complex neural network architecture and increased learning rate, 
    $(f_\text{KL}, \Gamma_1)$-GPA also overfits, 
    similarly to \Cref{fig:mnist_M=N}. Notably, with the  more complex NN setting, the overfitting/memorization is quite dramatic:  the $(f_\text{KL}, \Gamma_1)$-divergence decays exponentially fast, see the linear scaling in the green curve in \Cref{fig:mnist_trajectory_expon},  and eventually converges to (essentially) 0 which makes the particles literally stop, very much like the theoretical properties 
    of the gradient flow dynamics and the dissipation estimate in \Cref{thm:dissipation}.
    } }

    \label{fig:mnist_rich_nn}
\end{figure}

\begin{figure}
    \centering
    \begin{subfigure}{.64\textwidth}
        \centering
        \includegraphics[width=\linewidth]{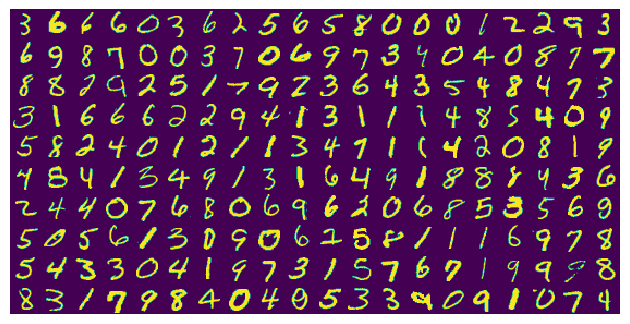}
        \caption{Entire target dataset with 200 samples}
    \end{subfigure}
    \begin{subfigure}{.32\linewidth}
   \centering
   \includegraphics[width=\textwidth]{plots_final/KL-Lipschitz_5.0000_uncond_0200_0600_00_trpt_n_gnrt-5000epochs-tiled_image.png}
   \caption{$M=600$ transported particles from $(f_\text{KL},\Gamma_5)$-GPA}
 \end{subfigure}%
 
    \begin{subfigure}{.49\textwidth}
      \centering  
      \includegraphics[width=.9\linewidth]{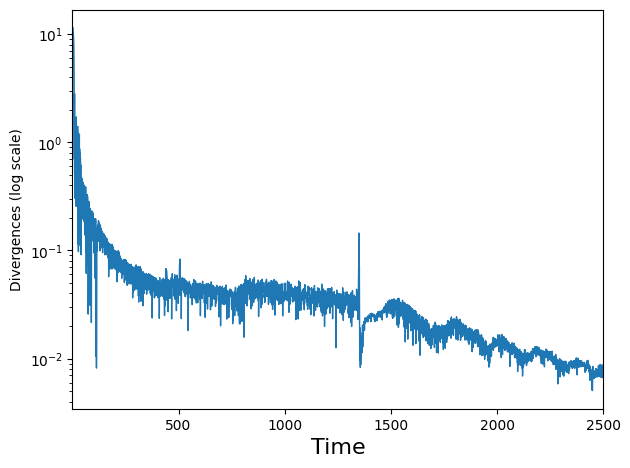}
        \caption{$(f_\text{KL}, \Gamma_5)$-divergence}
    \end{subfigure}
    \begin{subfigure}{.49\textwidth}
        \centering
        \includegraphics[width=.9\linewidth]{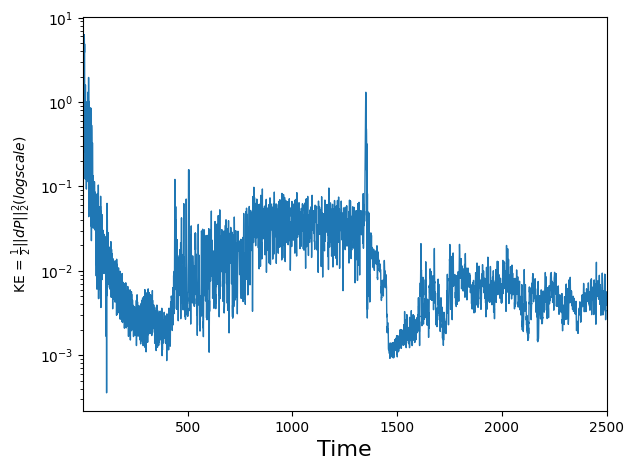}
        \caption{Kinetic energy of particles}
    \end{subfigure}
    \caption{ {\bf (MNIST) Sample diversity  of transported data in \Cref{fig:sub1}. }  For completeness in our study and presentation, we provide the entire target dataset to allow for a   one-by-one comparison.  In addition, we note that the divergence in (c) lies at the level of 1e-2, which is an evidence that the generated samples maintain diversity comparable to the original samples.  }
    \label{fig:mnist:si}
\end{figure}

\begin{figure}
    \centering
    \begin{subfigure}{.66\textwidth}
    \centering
\includegraphics[width=\linewidth]{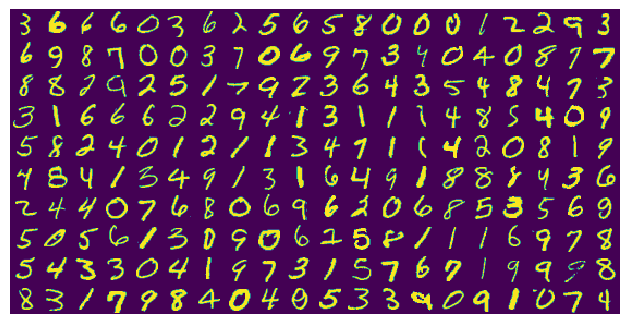}
        \caption{$N=1000$ Target data}
    \end{subfigure}
\begin{subfigure}{.33\textwidth}
    \centering
\includegraphics[width=\linewidth]{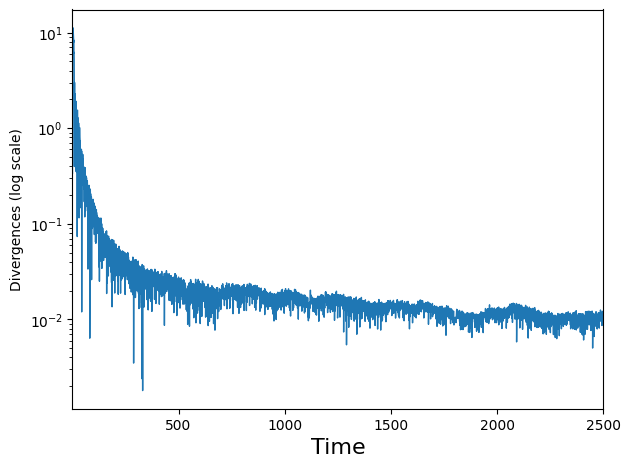}
        \caption{$(f_\text{KL},\Gamma_5)$-divergence}
    \end{subfigure}
    
    \begin{subfigure}{.32\textwidth}
    \centering
\includegraphics[width=\linewidth]{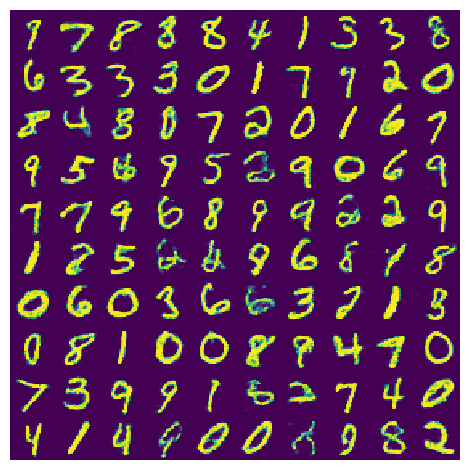}
        \caption{$M=3000$ transported particles from $(f_\text{KL},\Gamma_5)$-GPA}
    \end{subfigure}
    \begin{subfigure}{.32\textwidth}
    \centering
\includegraphics[width=\linewidth]{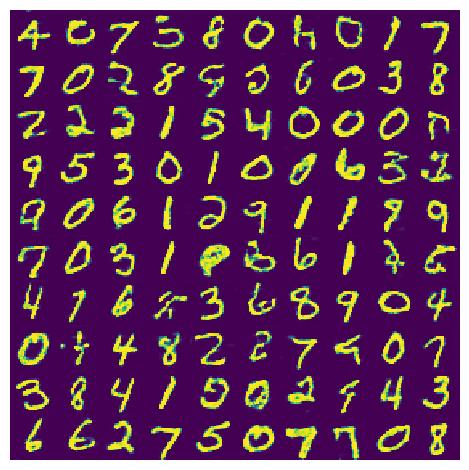}
        \caption{3000 generated particles that are simultaneously transported from $(f_\text{KL},\Gamma_5)$-GPA}
    \end{subfigure}
     \begin{subfigure}{.33\textwidth}
    \centering
\includegraphics[width=\linewidth]{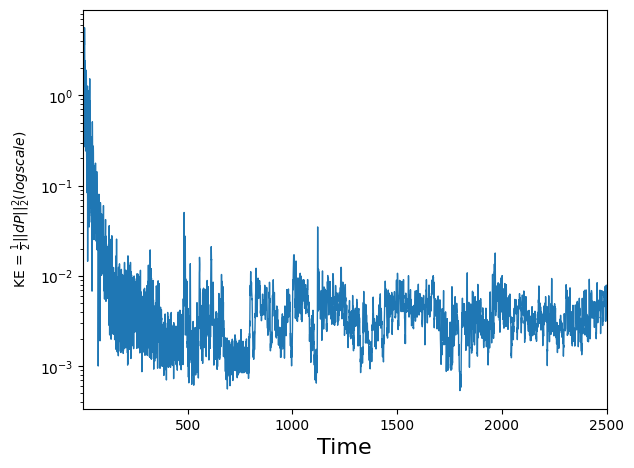}
        \caption{Kinetic energy of particles}
    \end{subfigure}

    \caption{\textbf{(MNIST) Impact of increased sample sizes compared to \Cref{fig:mnist_different_method}.} 
    We changed our data settings to $N = 1000$, $M = 3000$, while keeping other settings constant. The divergence in (b) remains at the level of $1 \times 10^{-2}$, ensuring that the generated samples maintain diversity comparable to the original samples. The computation time increased to 180 minutes, 3.77 times longer than the previous example in \Cref{fig:mnist_different_method}, but the resulting sample quality is similar or better.
    }
    \label{fig:mnist:scalability}
\end{figure}

\begin{figure}[h]
    \centering
    \captionsetup{format=plain}
\includegraphics[width=0.45\textwidth]{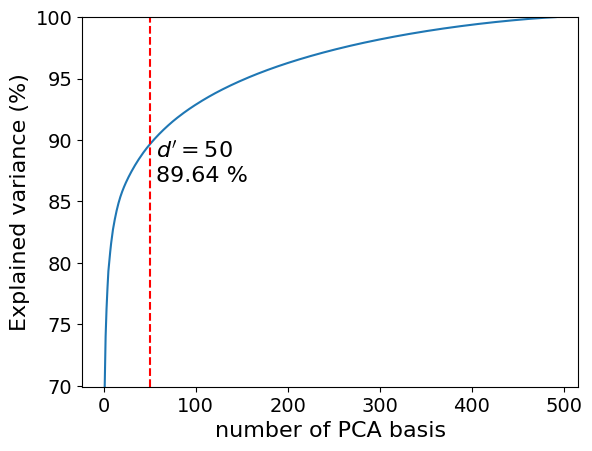}
    \caption{{\bf (Gene expression data integration with GPA) Dimension reduction with PCA and choice of latent dimensions in \Cref{sec:example:batch:effects}.}
    PCA computes an orthonormal basis via the singular value decomposition on the data covariance matrix. Each eigenvalue $\lambda_i$ is interpreted as the variance from the corresponding eigenvector $\mathbf{v_i}\in\mathbb R^d$. The latent features are obtained by projecting the data into the $d'$ eigenvectors with the highest eigenvalue values. We determine the dimension of the latent space, denoted by $d'$, according to the \textit{explained variance ratio} which is defined as $\sum_{i=1}^{d'} \lambda_i / \sum_{i=1}^d \lambda_i$.  \cref{fig:standardized_all_pca} shows the explained variance ratio as a function of $d'$. We choose to keep 89.64\%  of the explained variance which resulted in $d'=50$ dimensional latent space in  \Cref{fig:integration_gene_expression_data}. }
    \label{fig:standardized_all_pca}
\end{figure}

\clearpage

\begin{table}[h]
\centering
\begin{tabular}{l|m{6.5cm}|c|c}
 Metric & \centering Explanation & Positive  & Negative \\ 
\hline\hline
$W_2(\mathcal{E}_{\#}P_0, \mathcal{E}_{\#}Q)$ & Distance between source and target in latent space & 8.4320e+4 & 8.3507e+4 \\ \hline
$W_2(\mathcal{T}^n_{\#}\mathcal{E}_{\#}P_0,  \mathcal{E}_{\#}Q)$  & Distance between transported and target in latent space & 3.9939e+2 & 2.0675e+3 \\ \hline\hline
$W_2(\mathcal{D}_{\#} \mathcal{E}_{\#}Q, Q)$ & Reconstruction error of the target distribution (R) & 2.9768e+3 & 3.3805e+3 \\ \hline\hline
$W_2(P_0, Q)$  & Distance between source and target in original space (D) & 2.3883e+5 & 2.3845e+5 \\ \hline
$W_2(\mathcal{F}_{\#}P_0, Q)$ & Distance between mean \& std adjusted source and target in original space (S) & 9.4402e+3  & 1.1617e+4 \\ \hline
$W_2(\mathcal{D}_{\#} \mathcal{T}^n_{\#}\mathcal{E}_{\#}P_0, Q)$ & Distance between transported and target in original space (G)  & {3.5171e+3}  & {6.2247e+3} \\ \hline\hline
(R)/(D) &  The least relative distance that can be attained by latent GPA & 1.2464e-2  & 1.4176e-2 \\ \hline
(S)/(D) &  Relative distance attained by mean \& std adjusted source & {3.9526e-2}  & {4.8718e-2} \\ \hline
(G)/(D) &  Relative distance attained by latent GPA & \textbf{1.4726e-2}  & \textbf{2.6104e-2} \\ 
\hline
\end{tabular}

\caption{{\bf (Gene expression data integration with GPA) Quantitative results of data integration in \Cref{fig:integration_gene_expression_data}. } We compute the 2-Wasserstein distance between datasets in both the latent space ($d'=50$) and in the original space ($d=54,675$).
2-Wasserstein distance is approximated by Sinkhorn divergence \cite{genevay2017learning,geomloss}. For each class, original source $P_0$ and target $Q$ are determined by finite number of samples in \cref{tab:samplesize}. 
The reconstruction error due to PCA dimensionality reduction is below 1.5\% as quantified by the ratio (R)/(D). 
$\mathcal{T}^n$ is the composition of $(f_\text{KL},\Gamma_1)$-GPA transport maps defined in \cref{eq:scheme:1:sec3} for $n$ time steps while $\mathcal{F}$ denotes the baseline transformation which adjusts the mean \& std of the source to the mean \& std of the target distribution.
Dataset integration via latent GPA is approximately twice as effective compared to dataset integration via the baseline data  transformation, as can be readily observed in their respective ratios (G)/(D) and (S)/(D). Furthermore, we observe that the error as measured by the 2-Wasserstein distance in the latent space is higher for the negative class (row 2). This can be partially attributed to the smaller sample size of the source compared to the target. 
Interestingly, the relative ordering in the distance is also observed in the original space (row 6).
}
 \label{tab:sinkhorn}
\end{table}

\end{document}